\documentclass[journal,twoside]{IEEEtran}

\usepackage{amsmath,amsfonts}
\usepackage{algorithmic}
\usepackage{algorithm}
\usepackage{array}
\usepackage[caption=false,font=normalsize,labelfont=sf,textfont=sf]{subfig}
\usepackage{textcomp}
\usepackage{stfloats}
\usepackage{url}
\usepackage{verbatim}
\usepackage{graphicx}
\usepackage{cite}
\usepackage{orcidlink}
\usepackage{array}
\def\mystrut{\rule[-2ex]{0ex}{5ex}} 
\usepackage{hyperref}
\usepackage{amssymb}
\usepackage{xfrac}
\usepackage{dsfont}
\usepackage{subfig}
\usepackage[export]{adjustbox}
\usepackage{color, xspace} 
\usepackage[format=plain]{caption}

\newcommand{\blue}[1]{#1}

\definecolor{wine_red}{RGB}{228,48,64}
\definecolor{DSgray}{cmyk}{0,1,0,0}

\def \LBstatic {\mbox{OFUL}\xspace}
\def \LBweight {\mbox{D-LinUCB}\xspace}
\def \LBrestart {\mbox{RestartUCB}\xspace}
\def \LBwindow {\mbox{SW-UCB}\xspace}
\def \GLBstatic {\mbox{GLM-UCB}\xspace}
\def \GLBweight {\mbox{BVD-GLM-UCB}\xspace}

\def \SCBstatic {\mbox{LogUCB1}\xspace}
\def \LBweightours {\mbox{LB-WeightUCB}\xspace}
\def \GLBweightours {\mbox{GLB-WeightUCB}\xspace}
\def \GLBrestart {\mbox{GLB-RestartUCB}\xspace}
\def \SCBrestart {\mbox{SCB-RestartUCB}\xspace}
\def \SCBweightours {\mbox{SCB-WeightUCB}\xspace}
\def \SCBweightourspw {\mbox{SCB-PW-WeightUCB}\xspace}
\def \MASTER {\mbox{MASTER}\xspace}
\def \LBMASTER {\mbox{MASTER+\LBstatic}\xspace}
\def \GLBMASTER {\mbox{MASTER+\GLBstatic}\xspace}
\def \SCBMASTER {\mbox{MASTER+\SCBstatic}\xspace}
\def \LMDPweightours  {\mbox{WeightUCRL}\xspace}
\def \MNLweightours  {\mbox{MNL-WeightUCRL}\xspace}
\def \pathlength {path length\xspace}

\renewcommand{\tilde}{\widetilde}
\renewcommand{\hat}{\widehat}
\def \A {\mathcal{A}}
\def \C {\mathcal{C}}
\def \F {\mathcal{F}}
\def \Jb {\mathbb{J}}
\def \N {\mathcal{N}}
\def \M {\mathcal{M}}
\def \O {\mathcal{O}}
\def \Ot {\tilde{\O}}
\def \P {\mathbb{P}}
\def \Pt {\tilde{\P}}
\def \E {\mathbb{E}}
\def \R {\mathbb{R}}
\def \S {\mathcal{S}}
\def \T {\top}
\def \W {\mathcal{W}}
\def \X {\mathcal{X}}
\def \Xt {\tilde{X}}

\def \x {\mathbf{x}}

\def \DReg {\textsc{D-Reg}}
\def \define {\triangleq}
\def \thetah {\hat{\theta}}
\def \th {\hat{\theta}}
\def \thetat {\tilde{\theta}}
\def \thetab {\bar{\theta}}
\def \w {\mathbf{w}}
\def \wh {\hat{\w}}
\def \wb {\bar{\w}}
\def \wt {\tilde{\w}}
\def \lamt {{\lambda_\theta}}
\def \lamw {{\lambda_\w}}
\def \etat {\tilde{\eta}}
\def \Vt {\tilde{V}}
\def \diff {\mathrm{d}}
\def \dmu {\mu^\prime}
\def \ddmu {\mu^{\prime\prime}}
\def \Ht {\tilde{H}}

\def \Var {\textnormal{Var}}
\def \betat {\tilde{\beta}}
\def \betab {\bar{\beta}}
\def \betabr {\breve{\beta}}
\def \bias {\mathtt{bias~part}\xspace}
\def \variance {\mathtt{variance~part}\xspace}
\def \trace {\operatorname{trace}}

\def \psib {\bar{\psi}}
\def \Sigmab {\bar{\Sigma}}
\def \betab {\bar{\beta}}
\def \Qb {\bar{Q}}
\def \Vb {\bar{V}}
\def \gammabob {\tilde{\gamma}^{*}}

\let\norm\undefined 
\newcommand\norm[1]{\left\| #1 \right\|} 
\newcommand\abs[1]{\left| #1 \right|} 
\newcommand\inner[2]{\left\langle #1, #2 \right\rangle} 
\newcommand\commoninner[2]{\langle #1, #2 \rangle} 
\newcommand\sbr[1]{\left( #1 \right)} 
\newcommand\mbr[1]{\left[ #1 \right]} 
\newcommand\bbr[1]{\left\{ #1 \right\}}
\DeclareMathOperator*{\argmax}{arg\,max}
\DeclareMathOperator*{\argmin}{arg\,min}
\newcommand\given[1][]{\:#1\vert\:}
\newcommand\givenn[1][]{\:#1\middle\vert\:}
\def \parag {\vspace{3mm}\noindent\textbf}
\newcommand\term[1]{\textsc{term}~(\textsc{#1})}
\DeclareMathOperator{\indicator}{\mathds{1}}
\newcommand {\NA}{---}

\usepackage{amsthm}

\newtheorem{myThm}{Theorem}

\newtheorem{myLemma}{Lemma}

\newtheorem{myDef}{Definition}

\theoremstyle{definition}
\newtheorem{myAssum}{Assumption}

\newtheorem{myRemark}{Remark}

\usepackage{prettyref}
\newcommand{\pref}[1]{\prettyref{#1}}
\allowdisplaybreaks
\begin{document}

\title{Revisiting Weighted Strategy for Non-stationary Parametric Bandits and MDPs}

\author{Jing Wang\raisebox{0.5ex}{\orcidlink{0009-0001-2798-0884}}, Peng Zhao\raisebox{0.5ex}{\orcidlink{0000-0001-7925-8255}},~\IEEEmembership{Member,~IEEE,} and Zhi-Hua Zhou\raisebox{0.5ex}{\orcidlink{0000-0003-0746-1494}},~\IEEEmembership{Fellow,~IEEE}
        \thanks{J. Wang, P. Zhao and Z.-H. Zhou are with National Key Laboratory for Novel Software Technology and the School of Artificial Intelligence, Nanjing University, Nanjing 210023, China. (e-mail: \{wangjing,zhaop,zhouzh\}@lamda.nju.edu.cn). This paper was presented in part at the Proceedings of the 26th International Conference on Artificial Intelligence and Statistics (AISTATS), 2023.
        
        Manuscript received December 28, 2024; revised August 6, 2025. (Corresponding author: Peng Zhao)}
        }

\markboth{IEEE TRANSACTIONS ON INFORMATION THEORY,~Vol.~1, No.~1, December~2025}
{WANG \MakeLowercase{et al.}: Revisiting Weighted Strategy for Non-stationary Parametric Bandits and MDPs}

\IEEEpubid{0000--0000~\copyright~2025 IEEE}

\maketitle
\newif\iffullversion
\fullversiontrue   

\begin{abstract}
        Non-stationary parametric bandits have attracted much attention recently. There are three principled ways to deal with non-stationarity, including sliding-window, weighted, and restart strategies. As many non-stationary environments exhibit gradual drifting patterns, the weighted strategy is commonly adopted in real-world applications. However, previous theoretical studies show that its analysis is more involved and the algorithms are either computationally less efficient or statistically suboptimal. This paper revisits the weighted strategy for non-stationary parametric bandits. In linear bandits (LB), we discover that this undesirable feature is due to an inadequate regret analysis, which results in an overly complex algorithm design. We propose a \emph{refined analysis framework}, which simplifies the derivation and, importantly, produces a simpler weight-based algorithm that is as efficient as window/restart-based algorithms while retaining the same regret as previous studies. Furthermore, our new framework can be used to improve regret bounds of other parametric bandits, including Generalized Linear Bandits (GLB) and Self-Concordant Bandits (SCB). For example, we develop a simple weighted GLB algorithm with an $\Ot(k_\mu^{\sfrac{5}{4}} c_\mu^{-\sfrac{3}{4}} d^{\sfrac{3}{4}} P_T^{\sfrac{1}{4}}T^{\sfrac{3}{4}})$ regret, improving the $\Ot(k_\mu^{2} c_\mu^{-1}d^{\sfrac{9}{10}} P_T^{\sfrac{1}{5}}T^{\sfrac{4}{5}})$ bound in prior work, where $k_\mu$ and $c_\mu$ characterize the reward model's nonlinearity, $P_T$ measures the non-stationarity, $d$ and $T$ denote the dimension and time horizon. Moreover, we extend our framework to non-stationary Markov Decision Processes (MDPs) with function approximation, focusing on Linear Mixture MDP and Multinomial Logit (MNL) Mixture MDP. For both classes, we propose algorithms based on the weighted strategy and establish dynamic regret guarantees using our analysis framework.
\end{abstract}


\section{Introduction}
\label{sec:introduction}

\IEEEPARstart{N}{on-stationary} parametric bandits model the sequential decision-making problems where the reward distributions of each arm are structured with an unknown \emph{time-varying} parameter, which have been extensively studied in recent years~\cite{AISTATS'19:window-LB, NIPS'19:weighted-LB, AISTATS'20:restart,arXiv'20:NS-GLB,UAI'20:kim20a,arXiv'21:faury-driftingGLB, AISTATS'21:SCB-forgetting,COLT'21:black-box, AISTATS'22:weighted-GPB,arXiv'22:VanRoy,AISTATS'23:revisiting-weighted} due to their significance in many real-world non-stationary online applications such as recommendation systems~\cite{MLJ'21:TSnonstationary,AISTATS'21:change-preference}. This line of work also has a tight connection with the theoretical foundation of Reinforcement Learning (RL), particularly in the context of episodic Markov Decision Processes (MDPs) with function approximation~\cite{COLT'20:ChiJin,arXiv'20:Touati, AAAI'23:MNL_MDP,NIPS'24:MNL_MDP_efficient}. In these settings, parametric bandits are frequently employed to model both reward and transition dynamics across episodes. Moreover, when these underlying dynamics exhibit \emph{non-stationary} behavior across different episodes, non-stationary parametric bandit techniques naturally extend to capture the non-stationary dynamics of rewards and transitions~\cite{arxiv'22:ns_lin_mdp,NIPS'24:ALM_MDP_dynamic_regret}. 

Linear Bandits (LB) is a fundamental instance of parametric bandits, where the expected reward for pulling a certain arm at time $t$ is the inner product between the arm's feature vector $X_t$ and an unknown parameter $\theta_t$, namely, $\E[r_t \given X_t] = X_t^\T \theta_t$. Moreover, Generalized Linear Bandits (GLB) is introduced as a generalization of LB to model a broader range of reward functions (e.g. binary rewards), where the expected reward obeys a generalized linear model as $\E[r_t \given X_t] = \mu(X_t^\T \theta_t)$ with $\mu(\cdot)$ being an inverse link function. Furthermore, LB and GLB have fundamental applications in Markov Decision Processes (MDPs) with function approximation. As a representative instance, the Linear Mixture MDP adopts LB to model both reward functions and transition dynamics. Building on this, the Multinomial Logit (MNL) Mixture MDP was introduced to address the limitation of linear functions to model probabilities. By employing the MNL bandit (a special case of GLB), it effectively models transition probabilities and ensures valid distributions. Notably, the non-stationary models allow the parameter $\theta_t$ in the above models to be time-varying; therefore, we use dynamic regret~\cite{NIPS'14:besbes, JMLR'24:sword++} to evaluate the algorithm's performance. There are two typical non-stationarity measures to quantify the intensity of parameter changes: (i) in gradually drifting cases, \pathlength $P_T = \sum_{t=2}^{T} \norm{\theta_{t-1} - \theta_{t}}_2$ is used to measure the cumulative variations of the underlying parameters; and (ii) in piecewise-stationary cases, $\Gamma_T$ denotes the number of parameter changes in $T$ rounds. 

\IEEEpubidadjcol

To deal with non-stationarity, there are three principled ways: sliding-window, weighted, and restart strategies. For the sliding-window strategy, the learner maintains a time window that contains the most recent observed data to discard the outdated data. For the weighted strategy, the learner assigns more weight to the most recent data and less to older data, gradually forgetting the outdated data. For the restart strategy, the learner restarts the algorithm according to a certain period to discard the outdated data. The currently best-known result for non-stationary (generalized) linear bandits and episodic MDPs with linear function approximation is by~\cite{COLT'21:black-box}, who developed a minimax optimal algorithm consisting of a non-stationarity detector and a base algorithm that performs well in near-stationary environments. Whenever the detector examines that the non-stationarity exceeds a certain limit, the algorithm will \emph{restart} itself to handle the non-stationarity. In this sense, their algorithm can be regarded as an \emph{adaptive restart-based algorithm}. Building on the \LBrestart algorithm~\cite{AISTATS'20:restart} and a carefully designed non-stationarity detector with multi-scale explorations, their algorithm can achieve an $\Ot(\min\{\sqrt{\Gamma_T T}, P_T^{\sfrac{1}{3}} T^{\sfrac{2}{3}}\})$ optimal dynamic regret for both LB and GLB and $\Ot(\min\{\sqrt{\Gamma_T T}, \Delta^{\sfrac{1}{3}} T^{\sfrac{2}{3}}\})$ dynamic regret for Episodic MDPs where $\Delta$ represents the total \pathlength which includes the cumulative parameters variations of both reward function and transition function.

In real-world scenarios, the distributional change of environments often exhibits gradually drifting patterns~\cite{NSR'22:OpenML,COLT'10:concept-drift,COLT'13:Chiang,ACM'14:survey-concept-drift, TIT'24:smooth-change, FCSC'25:robust-domain-adaptation}, in such cases, a soft weighted strategy can be (empirically) more advantageous than a hard restart strategy to deal with the non-stationarity, as can be observed in bandits learning~\cite{NIPS'19:weighted-LB,AISTATS'20:restart,AISTATS'22:weighted-GPB}, classification with concept drift~\cite{SADM'12:adaptive-forgetting,TKDE'21:DFOP}, and adaptive system identification~\cite{guo1993performance, SP'17:FFRLS}. As a result, it will be highly attractive to design an \emph{adaptive weight-based algorithm} for non-stationary parametric bandits, which imposes weights to discount the importance of past data, and the weights are set adaptively according to environments. Towards this end, we examine existing methods for non-stationary parametric bandits based on the weighted strategy, and (surprisingly) find that current results exhibit \emph{unnatural} gaps compared to the other strategies, such as restart-based algorithms, as well as \emph{unnatural} regret analysis transitions from GLB to LB.

\begin{table*}
    \centering
    \caption{\small{Comparisons of our dynamic regret bounds to the previous best-known results for weight-based algorithms, under different non-stationary bandit and MDP settings. Below, $k_\mu/c_\mu$ characterize the non-linearity in GLB/SCB (reducing to $1$ for LB) and $\kappa$ denotes the non-linearity in MNL Mixture MDP; $d$ is the dimension, $H$ is the length of an episode in the MDP setting, \pathlength $P_T$ and the change number $\Gamma_T$ are non-stationarity measures for drifting and piecewise-stationary cases, respectively, and total \pathlength $\Delta$ measures the non-stationarity in the drifting MDP case.} }
    \label{table:results}
\renewcommand*{\arraystretch}{1.25}
    {\normalsize
    \begin{tabular}{c|r|r}
        \hline

        \hline
        \textbf{Settings}   &\multicolumn{1}{c|}{\textbf{Previous Work}} &\multicolumn{1}{c}{\textbf{Our Results}}\\\hline 
        \mystrut Drifting LB          &$\Ot\big(d^{\sfrac{7}{8}} P_T^{\sfrac{1}{4}} T^{\sfrac{3}{4}}\big)$~\cite{NIPS'19:weighted-LB}  &$\Ot\big({d^{\sfrac{3}{4}} P_T^{\sfrac{1}{4}} T^{\sfrac{3}{4}} }\big)$~[\pref{thm:LB-regret}] \\\hline
        \mystrut Drifting GLB              &$\Ot\Big(\frac{k_\mu^2}{c_\mu}d^{\sfrac{9}{10}} P_T^{\sfrac{1}{5}}T^{\sfrac{4}{5}}\Big)$~\cite{arXiv'21:faury-driftingGLB} &$\Ot\Big(\frac{k_\mu^{\sfrac{5}{4}}}{c_\mu^{\sfrac{3}{4}}}d^{\sfrac{3}{4}} P_T^{\sfrac{1}{4}}T^{\sfrac{3}{4}}\Big)$~[\pref{thm:GLB-regret}]\\\hline      
        \mystrut Drifting SCB              &$\Ot\Big(\frac{k_\mu^2}{c_\mu}d^{\sfrac{9}{10}} P_T^{\sfrac{1}{5}}T^{\sfrac{4}{5}}\Big)$~\cite{arXiv'21:faury-driftingGLB} &$\Ot\Big(\frac{k_\mu^{\sfrac{5}{4}}}{c_\mu^{\sfrac{1}{2}}}d^{\sfrac{3}{4}} P_T^{\sfrac{1}{4}}T^{\sfrac{3}{4}}\Big)$~[\pref{thm:SCB-regret}]\\\hline        
        \mystrut Piecewise Stationary SCB              &$\Ot\Big(\frac{1}{c_\mu^{\sfrac{1}{3}}}d^{\sfrac{2}{3}}\Gamma_T^{\sfrac{1}{3}}T^{\sfrac{2}{3}}\Big)$~\cite{AISTATS'21:SCB-forgetting} &$\Ot\big(d^{\sfrac{2}{3}}\Gamma_T^{\sfrac{1}{3}}T^{\sfrac{2}{3}}\big)$~[\pref{thm:SCB-PW-regret}]\\\hline      
        \mystrut Drifting Linear Mixture MDP       &$\Ot\Big(Hd\Delta^{\sfrac{1}{4}}T^{\sfrac{3}{4}}\Big)$~\cite{arxiv'22:ns_lin_mdp} &$\Ot\Big(Hd\Delta^{\sfrac{1}{4}}T^{\sfrac{3}{4}}\Big)$~[\pref{thm:LMDP-regret-bound}]\\\hline        
        \mystrut Drifting MNL Mixture MDP              &\multicolumn{1}{c|} \NA &$\Ot\big(\kappa^{-1}Hd\Delta^{\sfrac{1}{4}}T^{\sfrac{3}{4}}\big)$~[\pref{thm:MNL-regret-bound}]\\
        \hline

        \hline
    \end{tabular}
    }
\end{table*}

Those unnatural phenomena motivate us to revisit the algorithm design and regret analysis of the weighted strategy for non-stationary parametric bandits~\cite{NIPS'19:weighted-LB,AISTATS'21:SCB-forgetting,arXiv'21:faury-driftingGLB}. Indeed, the key ingredient is the \emph{estimation error} analysis for the weight-based estimator, which is usually decomposed into two parts --- one is the \emph{bias} part due to the parameter drift, and the other is the \emph{variance} part due to the stochastic noise. Generally, the bias part is controlled by non-stationary strategies, and the variance part is handled by carefully designed concentration.~\cite{NIPS'19:weighted-LB} provided the first analysis of a weight-based algorithm for linear bandits (LB). In their bias analysis, they introduced a virtual window size in the analysis to control the bias in order to mimic the analysis of a sliding-window strategy~\cite{AISTATS'19:window-LB}. For the variance analysis,~\cite{NIPS'19:weighted-LB} developed a weighted version of the self-normalized concentration inequality, which required a specially designed local norm. This introduced additional analytical complexity, since the previously studied sliding-window~\cite{AISTATS'19:window-LB} and restart strategy~\cite{AISTATS'20:restart} could directly apply the standard self-normalized concentration inequality~\cite{NIPS'11:AY-linear-bandits}. Consequently, they have to use \mbox{\emph{different}} local norms to control bias and variance parts, resulting in unexpected inefficiencies of algorithm design and complications of analysis. For LB, this leads to an algorithm \LBweight~\cite{NIPS'19:weighted-LB} requiring the maintenance of an extra covariance matrix as the local norm for the weighted version self-normalized concentration, which is less efficient than the window and restart-based algorithms~\cite{AISTATS'19:window-LB, AISTATS'20:restart}.

This analysis framework for weighted strategy introduces more severe issues in GLB, due to its more enriched and complicated structure. Specifically,~\cite{arXiv'21:faury-driftingGLB} studied the drifting GLB and designed a highly complex projection operation to control bias and variance parts following the way of~\cite{NIPS'19:weighted-LB} to mimic sliding-window analysis, and finally attained an $\Ot(d^{\sfrac{9}{10}} P_T^{\sfrac{1}{5}}T^{\sfrac{4}{5}})$ dynamic regret. Unfortunately, this \mbox{\emph{cannot}} recover the $\Ot(d^{\sfrac{7}{8}} P_T^{\sfrac{1}{4}}T^{\sfrac{3}{4}})$ bound enjoyed by the weight-based algorithm for drifting LB (a special case of GLB)~\cite{NIPS'19:weighted-LB}. Subsequently,~\cite{AISTATS'21:SCB-forgetting} investigated the non-stationary Self-Concordant Bandits (SCB), a subclass of GLB with many attractive structures. They can only conduct analysis under the piecewise-stationary setting, whereas they failed in the more challenging drifting setting, due to technical difficulties in bounding bias using conventional analysis. Moreover, since the weighted version of the self-normalized concentration for LB~\cite{NIPS'19:weighted-LB} could not be extended to the SCB setting, they further redesigned a weighted version specifically for SCB, building on the self-normalized concentration for stationary SCB~\cite{ICML'20:logistic-bandits}, which introduced substantial additional complexity into the analysis. As such, two open questions are proposed in their papers: (i) how to extend weight-based algorithms to drifting SCB; and (ii) how to replicate recent progress in stationary SCB~\cite{AISTATS'21:optimal-logistic-bandits} to improve dependence on $c_\mu$ in non-stationary SCB.

\parag{Our Results.~~} In this paper, we revisit the weighted strategy for non-stationary parametric bandits and MDPs. We discover that the earlier analysis framework for the weighted strategy may be inappropriate due to its reliance on mimicking the sliding-window analysis and the specifically designed weighted version of self-normalized concentration, which requires bounding the bias and variance components using \emph{different} local norms, and designing new weighted versions of this concentration tool for every new setting further introduces significant and unnecessary analytical complexity. As a result, there is no unified analysis framework that can be applied directly across different settings. To address this, we propose a \emph{refined analysis framework} for the weighted strategy. In our framework, a new analysis for the bias part is presented, while the variance part analysis only relies on the standard self-normalized concentration~\cite{NIPS'11:AY-linear-bandits} without the need for an additional weighted version and enables the use of a \emph{single} local norm to analyze both the bias and variance components. This refinement simplifies the analysis of the weighted strategy and makes the approach more broadly applicable to other decision-making settings. It also brings several benefits to algorithm design, including improved efficiency for LB and a resolution to the projection issue encountered in GLB and SCB. Furthermore, our analysis framework is not limited to the bandit setting and can be extended to online Markov decision processes (MDPs) scenarios. In this paper, we extend our results to two fundamental classes of MDPs: (i) non-stationary linear mixture MDPs, and (ii) non-stationary multinomial logit (MNL) mixture MDPs. Table~\ref{table:results} summarizes our main results compared with the best-known results for weight-based algorithms. Specifically, based on our refined analysis framework, we achieve: (i) for LB, our approach only needs to maintain one covariance instead of two and still enjoys the same regret as~\cite{NIPS'19:weighted-LB}; (ii) for GLB, our approach enjoys an $\Ot(k_\mu^{\sfrac{5}{4}} c_\mu^{-\sfrac{3}{4}} d^{\sfrac{3}{4}} P_T^{\sfrac{1}{4}} T^{\sfrac{3}{4}})$ regret bound, whose order of $d$, $P_T$ and $T$ matches that in LB case;  (iii) for SCB, we achieve an $\Ot(k_\mu^{\sfrac{5}{4}} c_\mu^{-\sfrac{1}{2}} d^{\sfrac{3}{4}} P_T^{\sfrac{1}{4}}T^{\sfrac{3}{4}})$ regret bound, and for piecewise stationary SCB, our approach achieves an $\Ot(d^{\sfrac{2}{3}}\Gamma_T^{\sfrac{1}{3}}T^{\sfrac{2}{3}})$ regret bound that can get rid of the influence of $c_\mu^{-1}$, resolving the second open problem asked by~\cite{AISTATS'21:SCB-forgetting}; (iv) for Linear Mixture MDP, we achieve an $\Ot(Hd\Delta^{\sfrac{1}{4}}T^{\sfrac{3}{4}})$ regret bound that enjoys the same regret as~\cite{arxiv'22:ns_lin_mdp} that was achieved by the restarted strategy; and (v) for MNL Mixture MDP, we establish the first dynamic regret bound of $\Ot(Hd\Delta^{\sfrac{1}{4}}T^{\sfrac{3}{4}})$ in the literature.

Compared with our earlier conference version~\cite{AISTATS'23:revisiting-weighted}, this extended version presents additional results, along with a simpler, clearer analysis and refined presentation. Firstly, this extended version further simplifies the analysis compared to our conference version~\cite{AISTATS'23:revisiting-weighted}. The earlier approach~\cite{NIPS'19:weighted-LB} relied on three key components: an extra covariance matrix, a weighted self-normalized concentration inequality, and a weighted potential lemma. In our conference version~\cite{AISTATS'23:revisiting-weighted}, we removed the need to maintain an additional covariance matrix. In this extended version, we take it a step further by showing that the standard self-normalized concentration inequality is sufficient for analyzing the weighted strategy. As a result, the only essential component for the weighted strategy analysis is the weighted potential lemma. The maintenance of two covariance matrices and the use of weighted self-normalized concentration, as done in previous works~\cite{NIPS'19:weighted-LB,AISTATS'23:revisiting-weighted}, are unnecessary. This simplification makes our analysis and algorithm both much simpler and more general. Secondly, this simplification makes our approach much more scalable and easier to extend to other bandit settings. Both earlier work~\cite{AISTATS'21:SCB-forgetting} and our conference version~\cite{AISTATS'23:revisiting-weighted} required designing a new weighted self-normalized concentration inequality when adapting the method to a new setting (e.g., SCB), which limited their generality. Our refined analysis removes this need, allowing the same framework to be applied across different problems without requiring problem-specific weighted concentration results. Thirdly, we extend our results to two fundamental settings of online MDPs with function approximation: linear mixture MDPs and multinomial logit (MNL) mixture MDPs. Notably, we provide the first dynamic regret guarantee for MNL mixture MDPs, demonstrating both the effectiveness and the broad applicability of our refined analytical framework for weighted strategy. 

\iffullversion \else Due to space constraints, we provide complete proofs only for linear bandits. Full proofs for generalized linear bandits, self-concordant bandits, Linear Mixture MDP, MNL Mixture MDP, and Bandits over Bandits can be found in the full version~\cite{arXiv'25:Revisiting_MDP}.\fi

\section{Related Work}
\label{sec:related_work}

\parag{Linear Bandits.}
The non-stationary LB problem was first studied by~\cite{AISTATS'19:window-LB}. They established an $\Omega(d^{\sfrac{2}{3}}P_T^{\sfrac{1}{3}}T^{\sfrac{2}{3}})$ minimax lower bound and then proposed \LBwindow algorithm based on the sliding-window strategy. Then \cite{NIPS'19:weighted-LB} proposed the \LBweight algorithm based on a weighted strategy, and \cite{AISTATS'20:restart} proposed the \LBrestart algorithm based on a restart strategy. Note that the three works proved an $\Ot(d^{\sfrac{2}{3}} P_T^{\sfrac{1}{3}}T^{\sfrac{2}{3}})$ regret bound, but there exists a subtle technical gap in the regret analysis as identified by~\cite{arxiv'21:NSLB_revisit_note}. After fixing the technical gap, all three aforementioned algorithms achieve an $\Ot(d^{\sfrac{7}{8}} P_T^{\sfrac{1}{4}}T^{\sfrac{3}{4}})$ regret bound~\cite{arxiv'21:NSLB_revisit_note, arXiv'21:restartucb}. However, to achieve this result, all three algorithms require the knowledge of the path length $P_T$ as an input at the beginning of algorithmic implementation, which is undesired. To address this,~\cite{AISTATS'19:window-LB} proposed the bandits-over-bandits (BOB) strategy as a meta-algorithm to learn the unknown parameter $P_T$, which can be combined with the above algorithms to remove the requirement of this prior knowledge. Afterward,~\cite{COLT'21:black-box} proposed the \MASTER algorithm with theoretically optimal~$\Ot(\min\{d\sqrt{\Gamma_T T}, d P_T^{\sfrac{1}{3}} T^{\sfrac{2}{3}}\})$ regret bound, also without requiring the non-stationarity level of environments (that is, $\Gamma_T$ and $P_T$) in advance, but requires fixed arm set assumption. Most recently, there has also been some new progress in the non-stationary (linear) bandits~\cite{arXiv'22:VanRoy, colt'22:most_significant_arm, arxiv'22:new_look, arXiv'23:LB_memory}. Furthermore, \cite[Remark 4]{Manage'22:window-LB} bypassed the aforementioned technical gap by restarting adversarial LB algorithms. However, it is important to note that this only applies to LB and requires fixed arm set assumption and known $P_T$.

\parag{Generalized Linear Bandits.}
The GLB problem was first introduced by~\cite{NIPS10:GLM-infinite}. They proposed \GLBstatic algorithm, achieving an $\Ot(k_\mu c_\mu^{-1}d\sqrt{T})$ regret bound where $k_\mu$, $c_\mu$ are the problem-dependent constants and $k_\mu/c_\mu$ represents the nonlinearity of the generalized linear model.~\cite{arXiv'21:faury-driftingGLB} extended the stationary GLB to the drifting case, and proposed \GLBweight algorithm with $\Ot(k_\mu^{2} c_\mu^{-1}d^{\sfrac{9}{10}} P_T^{\sfrac{1}{5}}T^{\sfrac{4}{5}})$ regret bound.~\cite{ICML'20:logistic-bandits} studied a specific instance of GLB called Logistic Bandits (LogB). They first pointed out that under the GLB setting, the problem-dependent constant $1/c_\mu$ could be very large in some cases like LogB, then they proposed the Logistic-UCB-1 algorithm with an $\Ot(c_\mu^{-\sfrac{1}{2}}d\sqrt{T})$ regret bound and the Logistic-UCB-2 algorithm with an $\Ot(d\sqrt{T}+c_\mu^{-1})$ regret bound. Subsequently, \cite{AISTATS'21:optimal-logistic-bandits} established an $\Omega(d\sqrt{\dmu(X_*^\T\theta_*)T})$ regret lower bound for logistic bandits and provided an optimal algorithm OFULog. \cite{AISTATS'21:SCB-forgetting} generalized the logistic bandits to self-concordant bandits and considered the piecewise-stationary case; their algorithm enjoys an $\Ot(c_\mu^{-\sfrac{1}{3}}d^{\sfrac{2}{3}} \Gamma_T^{\sfrac{1}{3}}T^{\sfrac{2}{3}})$ regret bound. To deal with $P_T$-unknown cases,~\cite{arXiv'21:faury-driftingGLB}  proposed a parameter-free algorithm by combining \GLBweight with the BOB strategy, but the final result is still suboptimal. Meanwhile, the black-box algorithm~\cite{COLT'21:black-box} can adaptively restart the stationary algorithm \GLBstatic~\cite{NIPS10:GLM-infinite} and achieve an~$\Ot(\min{k_\mu c_\mu^{-1}\sqrt{\Gamma_T T}, k_\mu^{4/3}c_\mu^{-1} d P_T^{\sfrac{1}{3}} T^{\sfrac{2}{3}}})$ regret, which matches the lower bound for non-stationary LB in terms of $P_T$ and $T$, and therefore optimal for non-stationary GLBs, since LB is a special case of GLB (i.e. $\mu(x) = x$). Recently, there has been notable progress in GLB regarding its efficiency and regret optimality in terms of non-linearity. Readers can refer to~\cite{NIPS'25:GLB_one_pass_update} for the latest advancements. Nonetheless, these results focus on the static regret setting.

\parag{MDP with Function Approximation.} Reinforcement learning with function approximation has attracted significant attention recently~\cite{L4DC'20:lin_mix_mdp,ICML'20:lin_mix_mdp,arxiv'22:ns_lin_mdp,AAAI'23:MNL_MDP,NIPS'24:MNL_MDP_efficient}, with two fundamental approaches: linear function approximation and generalized linear function approximation. Among these, Linear Mixture MDP was first introduced by~\cite{L4DC'20:lin_mix_mdp,ICML'20:lin_mix_mdp}, which is a representative model for linear function approximation. They proposed the UCRL-VTR algorithm, achieving a regret bound of $\Ot(d\sqrt{H^3T})$, where $H$ is the episode horizon. Building on this,~\cite{arxiv'22:ns_lin_mdp} extended the stationary Linear Mixture MDP to drifting case, they establish $\Omega(d^{5/6} \Delta^{1/3} H^{2/3} T^{2/3})$ minimax lower bound for non-stationary linear mixture MDPs, and then proposed the SW-LSVI-UCB algorithm, which achieves a regret bound of $\Ot(Hd \Delta^{\sfrac{1}{4}}T^{\sfrac{3}{4}})$, where $\Delta$ quantifies the cumulative variation of the underlying parameters. To better capture the probabilistic nature of transition dynamics,~\cite{AAAI'23:MNL_MDP} explored a class of generalized function approximation models and introduced the MNL Mixture MDP, which leverages the MNL function to model transitions. They proposed the UCRL-MNL algorithm, achieving a regret bound of $\Ot(\kappa^{-1}d\sqrt{H^3 T})$, where $\kappa$ represents the nonlinearity of the MNL model, $H$ is the episode horizon, and $K$ is the total number of episodes. 
Later, \cite{NIPS'24:MNL_MDP_efficient} further achieved an $\Ot(d\sqrt{H^3 T} + \kappa^{-1} d^2 H^2)$ regret bound for MNL Mixture MDP.

\section{Linear Bandit}
\label{sec:LB}
In this section, we first introduce the problem setting of non-stationary LB, and describe our \LBweightours algorithm and its theoretical guarantee. Then we present a proof sketch of Lemma~\ref{lemma:LB-estimation-error} to illustrate our proposed analysis framework in detail. Notably, our algorithm achieves the same regret bound as the best-known weight-based algorithm~\cite{NIPS'19:weighted-LB} without relying on a specially designed weighted version of self-normalized concentration and can be more efficient.

\subsection{Problem Setting}
\label{sec:LB-problem-setting}
At each round $t$, the learner chooses an arm $X_t$ from a feasible set $\X \subseteq \R^d$ and receives a reward $r_t$ such that
\begin{equation}
  \label{eq:LB-model}
  r_t = X_t^\T\theta_t + \eta_t,
\end{equation}
where $\theta_t \in \R^d$ is the unknown time-varying parameter and $\eta_t$ is the $R$-sub-Gaussian noise. The goal of the learner is to minimize the following (pseudo) \emph{dynamic regret}:
\begin{equation}
  \label{eq:LB-regret}
  \DReg_T = \sum_{t=1}^{T} \max_{\x \in \X} \x^{\T}\theta_t - \sum_{t=1}^{T} X_t^{\T}\theta_t,
\end{equation}
which is the cumulative regret against the optimal strategy that has full information of the unknown parameter. Here we consider the drifting case where we use path length $P_T = \sum_{t=2}^{T} \norm{\theta_{t-1} - \theta_{t}}_2$ as the non-stationarity measure. \blue{Notice that in this paper we focus on a fixed arm set $\mathcal{X}$. A time-varying arm set $\mathcal{X}_t$ does not introduce any additional difficulty for our weight-based algorithm. The only difference is that the optimal comparator in the dynamic regret~\eqref{eq:LB-regret} would need to be updated to $\max_{\x \in \mathcal{X}_t} \x^{\T} \theta_t$, and the arm selection step~\eqref{eq:LB-select-criteria} would be performed over the time-varying arm set $\mathcal{X}_t$ instead of $\mathcal{X}$. This does not affect the analysis. For simplicity, we stick to the fixed arm set setting in this paper.}

We work under the following standard boundedness assumption~\cite{NIPS'11:AY-linear-bandits,AISTATS'19:window-LB,NIPS'19:weighted-LB,AISTATS'20:restart}.

\begin{myAssum} \label{ass:bounded-norm} 
  \textnormal{
  The feasible set and unknown parameters are assumed to be bounded: $\forall \x \in \X$, $\norm{\x}_2 \leq L$, and $\theta_t \in \Theta$ holds for all $t \in [T]$ where $\Theta \define \{\theta \mid \norm{\theta}_2 \leq S\}$.}
\end{myAssum}
\subsection{Algorithm and Regret Guarantee}
\label{sec:LB-algorithm}
We propose the \LBweightours algorithm, which attains the same regret guarantees as previous methods while enjoying better efficiency. We first give the employed estimator and then derive its estimation error upper bound by our refined analysis framework, which is the key for algorithm design and regret analysis. Based on the estimation error bound, we propose our selection criterion and finally give the theoretical guarantee on its dynamic regret.

\parag{Estimator.~} \label{alg:LB-estimator}
We adopt a weighted regularized least square estimator similar to \LBweight~\cite{NIPS'19:weighted-LB}, the estimator $\thetah_t$ is the solution to the following problem,
\begin{equation}
  \label{eq:LB-estimator}
\min_{\theta}~\lambda\norm{\theta}_2^2 + \sum_{s = 1}^{t-1}w_{t-1,s}\sbr{X_s^\T \theta -r_s}^2,
\end{equation}
where $\lambda > 0$ is the regularization
coefficient and $\forall t\in[T], s\in[t-1], w_{t-1,s}$ is the weighted factor. To deal with non-stationarity, we set $w_{t,s} = \gamma^{t-s}$, where $\gamma \in (0, 1)$ is the discounted factor. This approach assigns lower weights to older data while giving higher weights to more recent data, thereby better adapting to changes over time. Clearly, $\thetah_t$ admits a \blue{closed-form solution} $\thetah_t = V_{t-1}^{-1}(\sum_{s=1}^{t-1}w_{t-1,s}r_sX_s)$, where $V_{t} = \lambda I_d + \sum_{s=1}^{t} w_{t,s} X_s X_s^\T, V_0  = \lambda I_d$ is the covariance matrix. Note that this closed-form solution can be further transformed into a recursive formula such that $V_t = \gamma V_{t-1} + X_tX_t^\top  + (1-\gamma) \lambda I_d$ where we set $w_{t,s} = \gamma^{t-s}$. This allows it to be updated online without storing historical data, which is another important computational advantage of the weighted strategy over the sliding-window strategy.

\parag{Upper Confidence Bounds.~~}\label{alg:LB-UCB}
For estimator~\eqref{eq:LB-estimator}, we provide the following estimation error bound. Notably, this is \emph{different} from the previous result~{\cite[\blue{Appendix} B.3, second and third steps in Proof of Theorem 2]{NIPS'19:weighted-LB}}. This difference is key to our algorithm's improved efficiency, as we discuss later.
\begin{myLemma}
  \label{lemma:LB-estimation-error}
  For any $\x \in \X$, $\gamma\in(0,1)$ and $\delta \in (0,1)$, with probability at least $1-\delta$, the following holds for all $t \in [T]$
  \begin{align}
      &\abs{\x^\T(\thetah_t-\theta_t)} \label{eq:LB-estimation-error}\\
      &\leq L^2\sqrt{\frac{d}{\lambda}} \sum_{p=1}^{t-1}\sqrt{\sum_{s=1}^{p}w_{t-1,s}}\norm{\theta_p -\theta_{p+1}}_2 + \beta_{t-1}\norm{\x}_{V_{t-1}^{-1}}, \notag
  \end{align}
  where $\beta_t$ is the radius of the confidence region set by
  \begin{equation}
    \label{eq:LB-confidence-radius}
    \beta_{t}=\sqrt{\lambda}S+R\sqrt{2\log\frac{1}{\delta}+d\log\sbr{1+\frac{L^2 \sum_{s=1}^{t}w_{t,s}}{\lambda d}}}.
  \end{equation}
  \end{myLemma}
  \begin{algorithm}[!t]
    \caption{\LBweightours}
    \label{alg:LB-WeightUCB}
  \begin{algorithmic}[1]
  \REQUIRE time horizon $T$, discounted factor $\gamma$, confidence $\delta$, regularizer $\lambda$, parameters $S$, $L$ and $R$\\
  \STATE Set $V_0 = \lambda I_d$, $\thetah_1 = \mathbf{0}$ and compute $\beta_{0}$ by~\eqref{eq:LB-confidence-radius}
  \FOR{$t = 1,2,...,T$}
    \STATE Select $X_t = \argmax_{\x \in \X} \bbr{ \langle \x,\thetah_t\rangle + \beta_{t-1} \norm{\x}_{V_{t-1}^{-1}}}$
    \STATE Receive the reward $r_t$
    \STATE Update $V_{t} = \gamma V_{t-1} + X_t X_t^\T +(1-\gamma)\lambda I_d$ \label{LB:update}
    \STATE Compute $\thetah_{t+1}$ by~\eqref{eq:LB-estimator} and $\beta_{t}$ by~\eqref{eq:LB-confidence-radius} with $w_{t,s} = \gamma^{t-s}$
  \ENDFOR
  \end{algorithmic}
  \end{algorithm}
  The proof of Lemma~\ref{lemma:LB-estimation-error} is presented in Appendix~\ref{sec:LB-estimation-error-proof}. Based on Lemma~\ref{lemma:LB-estimation-error}, we can specify the arm selection criterion as
  \begin{equation}
    \label{eq:LB-select-criteria}
    \begin{split}
        X_t = {}&\argmax_{\x \in \X} \bbr{ \commoninner{\x}{\thetah_t} + \beta_{t-1}\norm{\x}_{V_{t-1}^{-1}}}.
    \end{split}
  \end{equation}
The overall algorithm is summarized in Algorithm~\ref{alg:LB-WeightUCB}. From the update procedure in Line~\ref{LB:update} of Algorithm~\ref{alg:LB-WeightUCB}, we can observe that our algorithm needs to maintain a \emph{single} covariance matrix $V_{t-1} \in \R^{d\times d}$. By contrast, the selection criterion of the algorithm proposed in~\cite{NIPS'19:weighted-LB} is 
$$X_t = \argmax_{\x \in \X} \bbr{ \commoninner{\x}{\thetah_t} + \beta_{t-1}\norm{\x}_{V_{t-1}^{\prime-1}\Vt_{t-1}V_{t-1}^{\prime-1}}},$$ where $\beta_{t-1}$ is similar to those in our selection criterion~\eqref{eq:LB-select-criteria}, $V^{\prime}_{t-1} = \lambda I_d + \sum_{s=1}^{t-1} \gamma^{t-s-1} X_s X_s^\T \in \R^{d\times d} $, and $\Vt_{t-1} = \lambda I_d + \sum_{s=1}^{t-1} \gamma^{2(t-s-1)} X_s X_s^\T \in \R^{d\times d} $ is an extra covariance matrix. Thus, our algorithm is more efficient than their algorithm since it only needs to maintain one covariance matrix instead of two. This owes to the fact that our analysis of Lemma~\ref{lemma:LB-estimation-error} only uses $V_{t-1}^{-1}$ as the local norm to analyze both bias and variance parts, but the algorithm of~\cite{NIPS'19:weighted-LB} requires to use $l_2$-norm and $V_{t-1}^{-1}\Vt_{t-1}V_{t-1}^{-1}$-norm to control bias and variance parts, respectively. In Section~\ref{sec:LB-analysis-framework}, we provide a sketch of the analysis framework for Lemma~\ref{lemma:LB-estimation-error}, and a more detailed discussion is presented in Appendix~\ref{sec:LB-review}. Furthermore, we prove that our algorithm enjoys the same (even slightly better in $d$) regret as the algorithm of~\cite{NIPS'19:weighted-LB}.
\begin{myThm}
  \label{thm:LB-regret}
  Let the weighted factor $w_{t,s} = \gamma^{t-s}$, where $\gamma\in(1/T,1)$, and let $\lambda = d$, the dynamic regret of \LBweightours (Algorithm~\ref{alg:LB-WeightUCB}) is bounded with probability at least $1-1/T$, by 
  \begin{equation}\nonumber
  \begin{split}
      \DReg_T\leq{}& \Ot\sbr{\frac{1}{(1-\gamma)^{\sfrac{3}{2}}}P_T + d(1-\gamma)^{\sfrac{1}{2}}T}.
  \end{split}
  \end{equation}
  Furthermore, by setting the discounted factor optimally as $\gamma = 1- \max\{1/T, \sqrt{P_T/(dT)}\}$, \LBweightours ensures
  \begin{equation}\nonumber
      \DReg_T \leq
      \begin{cases}
      \Ot\sbr{d^{ \sfrac{3}{4}} P_T^{ \sfrac{1}{4}} T^{ \sfrac{3}{4}}} & \mbox{ when } P_T \geq d/T,\vspace{2mm}\\
      \Ot(d\sqrt{T}) & \mbox{ when } P_T < d/T.
      \end{cases} 
  \end{equation}
\end{myThm}
Compared to previous works~\cite{AISTATS'19:window-LB,NIPS'19:weighted-LB,AISTATS'20:restart}, our approach improves from $\Ot(d^{ \sfrac{7}{8}} P_T^{ \sfrac{1}{4}} T^{ \sfrac{3}{4}})$ to $\Ot(d^{ \sfrac{3}{4}} P_T^{ \sfrac{1}{4}} T^{ \sfrac{3}{4}})$ when $P_T \geq d/T$. We remark that this improved dimensional dependence is simply owing to the more refined tuning of the discounted factor than the one used by~\cite{NIPS'19:weighted-LB}, who did not take the dimension into the tuning. Their algorithm and regret can also benefit from the refined tuning. The proof of~\pref{thm:LB-regret} is in Appendix~\ref{sec:LB-regret-proof}.

Further, notice that the optimal choice of discounted factor $\gamma$ requires knowing $P_T$ in advance. To achieve a parameter-free result for unknown $P_T$ case, our algorithm can be combined with the BOB strategy~\cite{AISTATS'19:window-LB} and achieves an $\Ot(d^{ \sfrac{3}{4}} P_T^{ \sfrac{1}{4}} T^{ \sfrac{3}{4}})$ bound.\iffullversion We provide the BOB version of \LBweightours and detailed regret analysis in Appendix~\ref{sec:BOB}.\fi However, this bound is not optimal, and it is possible to design an adaptive weight-based algorithm based on our result, in the spirit of~\cite{COLT'21:black-box}, to further achieve an optimal dynamic regret without prior knowledge of $P_T$. This is very challenging since at each round $t \in [T]$, we can only receive one data pair $(X_t,r_t)$, which is not adequate for the learner to real-time update the discounted factor $\gamma_{t}$. At the same time, \MASTER algorithm~\cite{COLT'21:black-box} can be considered as a special case of the adaptive weight-based algorithm since it only includes two circumstances: setting $\gamma_t = 0$ to restart at time $t$ and setting $\gamma_t = 1$ to keep going. However, for the adaptive weight-based algorithm, the choice of the discounted factor $\gamma_t$ can be continuous in $[0,1]$, which is more difficult than a binary decision. We leave this as an important open question for future study. \blue{Additionally, we note that our approach can handle time-varying arm set settings, whereas MASTER relies on the fixed arm set assumption. It remains unclear whether optimal regret can be achieved under time-varying arm set.}

\subsection{Refined analysis framework}
\label{sec:LB-analysis-framework}
In this section, we present a proof sketch for Lemma~\ref{lemma:LB-estimation-error} (estimation error analysis for weighted linear bandits), which also serves as a description of our proposed analysis framework.

\begin{proof}[Proof Sketch]
From the model assumption~\eqref{eq:LB-model} and the estimator~\eqref{eq:LB-estimator}, the estimation error can be split into two parts,
\begin{equation*}
  \begin{split}
  \thetah_t - \theta_t  ={}&\underbrace{V_{t-1}^{-1}\sbr{\sum_{s=1}^{t-1}w_{t-1,s}  X_s X_s^\T \sbr{\theta_s -\theta_t}}}_{\bias} \\{}&+ \underbrace{V_{t-1}^{-1}\sbr{ \sum_{s=1}^{t-1}w_{t-1,s}\eta_sX_s -\lambda\theta_t}}_{\variance},
  \end{split}
\end{equation*}
where the \emph{bias} part is caused by the parameter drifting, and the \emph{variance} part is due to the stochastic noise. Then, by the Cauchy-Schwarz inequality, for any $\x\in\X$,
\begin{equation}
\begin{split}
  \label{eq:LB-decompose}
    |\x^\T (\thetah_t - \theta_t)| \leq \norm{\x}_{V_{t-1}^{-1}}(A_t+B_t),\\
\end{split}
\end{equation}
where $A_t = \|\sum_{s=1}^{t-1}w_{t-1,s}  X_s X_s^\T \sbr{\theta_s -\theta_t}\|_{V_{t-1}^{-1}}$ and $B_t = \|\sum_{s=1}^{t-1}w_{t-1,s}\eta_sX_s -\lambda\theta_t\|_{V_{t-1}^{-1}}$.

Choosing an appropriate local norm for~\eqref{eq:LB-decompose} is the key to simplifying and improving the estimation error analysis. Note that the previous analysis~\cite{NIPS'19:weighted-LB} had to use \emph{different} local norms: using $l_2$-norm in the bias part, and $V_{t-1}^{\prime-1}\Vt_{t-1}V_{t-1}^{\prime-1}$-norm in the variance part, namely,
\begin{equation}
  \begin{split}
    \label{eq:previous-decompose}
      |\x^\T (\thetah_t - \theta_t)| \leq \norm{\x}_{2}A_t'+\norm{\x}_{V_{t-1}^{\prime-1}\Vt_{t-1}V_{t-1}^{\prime-1}}B_t',\\
  \end{split}
\end{equation}
where we have $A_t' = \|V_{t-1}^{\prime-1}\sum_{s=1}^{t-1}\gamma^{t-s-1}  X_s X_s^\T \sbr{\theta_s -\theta_t}\|_{2}$, $B_t' = \|\sum_{s=1}^{t-1}\gamma^{t-s-1}\eta_sX_s -\lambda\theta_t\|_{\Vt_{t-1}^{-1}}$ and $V^{\prime}_{t} = \lambda I_d + \sum_{s=1}^{t-1} \gamma^{t-s} X_s X_s^\T$, $\Vt_{t} = \lambda I_d + \sum_{s=1}^{t} \gamma^{2(t-s)} X_s X_s^\T$. Due to the need for using sliding-window analysis to analyze the bias part, they have to use $l_2$-norm to get the format of $A_t'$. For the variance part, to use weighted version of self-normalized concentration, they use the $V_{t-1}^{\prime-1}\Vt_{t-1}V_{t-1}^{\prime-1}$-norm to control $\x$ term so that $B_t'$ term can be normed by $\Vt_{t-1}^{-1}$.

As an improvement, we directly use the \emph{same} $V_{t-1}^{-1}$-norm to control both parts, which benefits from our new analysis for the bias part and modified analysis for the variance part. 

\parag{Bias Part Analysis.}~~ The key step of bias part analysis is to extract the variations of underlying parameters as follows,
\begin{align*}
  A_t &\leq L \sum_{p=1}^{t-1} \sum_{s=1}^{p}w_{t-1,s}\norm{X_s}_{V_{t-1}^{-1}}\|\theta_p -\theta_{p+1}\|_2\\&\leq L\sqrt{d} \sum_{p=1}^{t-1} \sqrt{\sum_{s=1}^{p}w_{t-1,s}}\norm{\theta_p -\theta_{p+1}}_2.
\end{align*} 
Based on that, we can obtain an upper bound for bias part related to the path length \strut $P_T = \sum_{t=2}^T \norm{\theta_{t-1} - \theta_t}_2$. A precise proof for the above argument can be found in Lemma~\ref{lemma:LB-A_t-bound}.

\parag{Variance Part Analysis.~~} The key lies in analyzing the following self-normalized term with weighted factor $w_{t-1,s}$,
  \begin{align*}
    B_t &\leq \norm{ \sum_{s=1}^{t-1}w_{t-1,s}\eta_sX_s }_{V_{t-1}^{-1}}+\sqrt{\lambda}S \\&= \norm{ \sum_{s=1}^{t-1}\etat_s\Xt_s}_{V_{t-1}^{-1}}+\sqrt{\lambda}S.
  \end{align*}
  Here, notice that $V_{t} = \lambda I_d + \sum_{s=1}^{t} \Xt_s \Xt_s^\T$, where we define $\etat_s \define \sqrt{w_{t-1,s}}\eta_s$ and $\Xt_s \define \sqrt{w_{t-1,s}} X_s$. Notably, for all $t \in [T]$ and $s \in [t-1]$, it holds that $\abs{w_{t-1,s}} \leq 1$, which ensures that $\etat_s$ remains $R$-sub-Gaussian, and a precise argument can be found in Lemma~\ref{lemma:LB-B_t-bound}. Consequently, we can directly apply the self-normalized concentration (Theorem~\ref{thm:snc-AY}) to control the variance term, without requiring the weighted version of the self-normalized concentration proposed in Theorem 1 of ~\cite{NIPS'19:weighted-LB}.

  Combining the analysis for bias and variance parts, we can finish the proof of Lemma~\ref{lemma:LB-estimation-error}.
\end{proof}
The estimation error analysis for weighted strategies involves first decomposing the estimation error into bias and variance parts, then analyzing them separately. \cite{NIPS'19:weighted-LB} used different local norms to decompose estimation errors, mimicking sliding window analysis for the bias term and specifically designing a weighted version of self-normalized concentration for the variance term. Our refined analysis framework shows that such complexity is unnecessary. Bias and variance can be decomposed using the same local norm, with a dedicated bias analysis for the weighted strategy, and the variance term no longer requires specially designed concentrations or additional local norms. With the estimation error bound, we proceed to the regret analysis, where we need to use a weighted potential lemma to bound the regret.  

\parag{Weighted Potential Lemma.}~Term $\norm{\x}_{V_{t-1}^{-1}}$ in~\eqref{eq:LB-decompose} induces a summation term $\sum_{t=1}^T\|X_t\|_{V_{t-1}^{-1}}$ in the variance part of regret analysis. Since $V_{t-1} = \lambda I_d + \sum_{s=1}^{t-1} w_{t-1,s} X_s X_s^\T$ incorporates the weighted factors, we cannot directly apply the standard potential lemma. Instead, we need to use a weighted potential lemma (see Lemma~\ref{lemma:potential-lemma}) for regret analysis with weighted factor $w_{t,s} = \gamma^{t-s}$, such that
\begin{equation}
\begin{split}
  \label{eq:LB-weighted-potential-sketch}
  {}&\sum_{t=1}^T\|X_t\|_{V_{t-1}^{-1}} = \Ot\sbr{T\sqrt{d\log\frac{1}{\gamma}}}.
\end{split}
\end{equation}
The next is to choose the discounted factor $\gamma$ appropriately, so that the bias and variance terms in the regret bound are well balanced. A smaller $\gamma$ corresponds to faster forgetting, which helps reduce the bias caused by non-stationarity. However, a smaller $\gamma$ will also increase the variance by the key term $\O(\sqrt{\log \frac{1}{\gamma}})$ shown in~\eqref{eq:LB-weighted-potential-sketch}. For details on how to optimally select $\gamma$ to balance these two parts, please refer to Appendix~\ref{sec:LB-regret-proof}.

To summarize, for non-stationary LB analysis, the weighted strategy is as simple as the restarted or sliding-window strategies, \emph{with only the difference being the requirement of the weighted potential lemma}  for regret analysis, without the need for more complicated deviation results.

\begin{myRemark}
    The key step~\eqref{eq:LB-decompose} in our analysis framework also resolves the projection issue in GLB. Specifically, after the projection step, the bias-variance decomposition can only be performed in $V_{t-1}^{-1}$-norm. To accommodate previous analysis~\eqref{eq:previous-decompose},~\cite{arXiv'21:faury-driftingGLB} has to inject a highly complex projection operation in the algorithm, whereas our framework already satisfies this condition owing to the usage of the same $V_{t-1}^{-1}$-norm for the bias and variance parts.
\end{myRemark}

\section{Generalized Linear Bandit}
\label{sec:GLB}
In this section, we apply the weighted strategy to drifting GLB. Compared to the best-known weight-based algorithm for drifting GLB~\cite{arXiv'21:faury-driftingGLB}, our algorithm is simpler and meanwhile has a better dynamic regret. Additionally, we consider a key class of GLB, known as the Self-Concordant Bandit (SCB), and further improve the theoretical guarantees in this setting.

\subsection{Problem Setting}
\label{sec:GLB-problem-setting}
GLB assumes an inverse link function $\mu:\R\rightarrow \R$ such that $r_t = \mu(X_t^\T \theta_t) + \eta_t$, where $\theta_t \in \R^d$ is the unknown parameter and can change over time. Similar to LB, we define \emph{dynamic regret} for GLB as follows:
\begin{equation}
  \label{eq:GLB-regret}
  \DReg_T = \sum_{t=1}^{T} \sbr{\max_{\x \in \X} \mu(\x^{\T}\theta_t) - \mu( X_t^{\T}\theta_t)}.
\end{equation}
Under GLB, we make the same assumptions as those of LB, which include $R$-sub-Gaussian noise, boundedness of feasible set and unknown regression parameters (Assumption~\ref{ass:bounded-norm}). In addition, we work under the standard boundedness assumption of the inverse link function~\cite{NIPS10:GLM-infinite, ICML18:GLM-finite, arXiv'21:faury-driftingGLB}.

\begin{myAssum} \label{ass:link-function} 
  The inverse link function $\mu: \R\rightarrow \R$ is $k_\mu$-Lipschitz, and continuously differentiable with $$c_\mu \triangleq \inf_{\{\theta \in \Theta, \x \in \X\}} \dmu(\theta^\T \x)>0,\quad \Theta = \{\theta \mid \norm{\theta}_2\leq S \}.$$ 
\end{myAssum}
Previous works~\cite{AISTATS'20:restart,Manage'22:window-LB} define a similar parameter $\widetilde{c}_\mu \triangleq \inf_{\{\theta\in \R^d, \x \in \X\}} \dmu(\theta^\T \x)>0$ and obtain regret upper bound scaling with $1/\widetilde{c}_\mu$. Clearly, $\widetilde{c}_\mu$ is smaller than our defined $c_\mu$ (and can be much smaller) as $c_\mu$ is defined on $\Theta$ while $\widetilde{c}_\mu$ is defined on $\R$. Therefore, $\widetilde{c}_\mu$ is less attractive to appear in the regret upper bound.
  
  \subsection{Algorithm and Regret Guarantee}
\label{sec:GLB-algorithm} 
We propose \GLBweightours, which is a simpler algorithm with better theoretical guarantee compared to previous weight-based algorithm~\cite{arXiv'21:faury-driftingGLB}. The key improvement is owing to our refined analysis framework, which is compatible with a simple projection step.

\parag{Estimator.} At iteration $t$, we first adopt the quasi-maximum likelihood estimator (QMLE) without considering the projection onto the feasible domain. Specifically, the estimator $\thetah_t$ is the solution of the following weighted regularized equation:
\begin{equation}
  \label{eq:GLB-estimator}
    \lambda c_\mu \theta +\sum_{s=1}^{t-1}w_{t-1,s}\sbr{\mu(X_s^\T \theta) - r_s}X_s = 0.
\end{equation}
Similar to the Estimator~\eqref{eq:LB-estimator}, we set $w_{t,s} = \gamma^{t-s}$, where $\gamma \in (0, 1)$ is the discounted factor. Given that $\thetah_t$ may not belong to the feasible set $\Theta$ and $c_\mu$ is defined over the parameter $\theta \in \Theta$, we need to perform the following projection step
\begin{equation}
  \label{eq:GLB-projection}
    \thetat_t = \argmin_{\theta \in \Theta}\|g_t(\thetah_t) - g_t(\theta)\|_{V_{t-1}^{-1}},
\end{equation}
where $V_{t} = \lambda I_d + \sum_{s=1}^{t} w_{t,s} X_s X_s^\T$ and $g_t(\theta)$ is
\begin{equation}
  \begin{split}
      \label{eq:GLB-gt}
      g_t(\theta) \define \lambda c_\mu \theta + \sum_{s=1}^{t-1}w_{t-1,s}\mu(X_s^\T \theta)X_s.
  \end{split}
\end{equation}
However, previous work~\cite{arXiv'21:faury-driftingGLB} cannot conduct the same simple projection in the drifting case as stationary GLB or piecewise-stationary GLB, since they use different local norms to measure the bias and variance parts separately for estimation error analysis. Consequently, they have to design a complicated projection to ensure that the bias and variance parts could be measured by different local norms  \iffullversion(see \cite[Section 4.1]{arXiv'21:faury-driftingGLB}, and our restatements in Appendix~\ref{sec:GLB-previous-work}). \else(see \cite[Section 4.1]{arXiv'21:faury-driftingGLB}).\fi

Our refined analysis framework is compatible with this projection operation, thanks to our analysis framework utilizing the same local norm for the bias and variance parts.

\parag{Upper Confidence Bounds.~~} For estimator~\eqref{eq:GLB-estimator} with projection~\eqref{eq:GLB-projection}, we construct following estimation error bound.
\begin{myLemma}
  \label{lemma:GLB-estimation-error}
  For any $\x \in \X$, $\gamma\in(0,1)$ and $\delta \in (0,1)$, with probability at least $1-\delta$, the following holds for all $t \in [T]$
  \begin{equation*}
    \begin{split}
      {}&\abs{\mu(\x^\T\thetat_t) - \mu(\x^\T\theta_t)} \\\leq{}& \frac{2k_\mu}{c_\mu}\sbr{\sum_{p=1}^{t-1}C(p)\norm{\theta_p - \theta_{p+1}}_2 + \betab_{t-1}\norm{\x}_{V_{t-1}^{-1}}},
    \end{split}
  \end{equation*}
where $C(p) \define k_\mu L^2\sqrt{\frac{d}{\lambda}}\sqrt{\sum_{s=1}^{p}w_{t-1,s}}$ and $\betab_t$ is the radius of confidence region set by
\begin{equation}
  \begin{split}
    \label{eq:GLB-confidence-radius}
    \sqrt{\lambda}c_\mu S+R\sqrt{2\log\frac{1}{\delta}+d\log\sbr{1+\frac{L^2 \sum_{s=1}^{t}w_{t,s}}{\lambda d}}}.
  \end{split}
\end{equation}
\end{myLemma}
\iffullversion The proof of Lemma~\ref{lemma:GLB-estimation-error} is in Appendix~\ref{sec:GLB-estimation-error-proof}.\fi Then, based on Lemma~\ref{lemma:GLB-estimation-error}, we can specify the arm selection criterion as
\begin{equation}
\label{eq:GLB-select-criteria}
\begin{split}
    X_t ={}&\argmax_{\x \in \X} \bbr{ \mu(\x^\T \thetat_t)+ \frac{2k_\mu}{c_\mu}\betab_{t-1}\norm{\x}_{V_{t-1}^{-1}}}.
\end{split}
\end{equation}
The overall algorithm is summarized in Algorithm~\ref{alg:GLB-WeightUCB}. 

\begin{algorithm}[!t]
  \caption{\GLBweightours}
  \label{alg:GLB-WeightUCB}
\begin{algorithmic}[1]
\REQUIRE time horizon $T$, discounted factor $\gamma$, confidence $\delta$, regularizer $\lambda$, link function $\mu$, parameters $S$, $L$ and $R$\\
\STATE Set $V_0 = \lambda I_d$, $\thetah_1 = \mathbf{0}$, compute $k_\mu, c_\mu$ and $\betab_0$ by~\eqref{eq:GLB-confidence-radius}
\FOR{$t = 1,2,...,T$}
  \IF{$\|\thetah_t\|_2\leq S$} 
  \STATE let $\thetat_t = \thetah_t$
  \ELSE 
  \STATE Do the projection and get $\thetat_t$ by~\eqref{eq:GLB-projection}
  \ENDIF
  \STATE Select $X_t$ by~\eqref{eq:GLB-select-criteria}
  \STATE Receive the reward $r_t$
  \STATE Update $V_{t} = \gamma V_{t-1} + X_t X_t^\T +(1-\gamma)\lambda I_d$
  \STATE Compute $\thetah_{t+1}$ according to~\eqref{eq:GLB-estimator} with $w_{t,s} = \gamma^{t-s}$
  \STATE Compute $\betab_{t}$ by~\eqref{eq:GLB-confidence-radius} with $w_{t,s} = \gamma^{t-s}$
\ENDFOR
\end{algorithmic}
\end{algorithm}

Notice that the estimation equation~\eqref{eq:GLB-estimator} and the confidence radius~\eqref{eq:GLB-confidence-radius} are the same as those used in Algorithm 1 of~\cite{arXiv'21:faury-driftingGLB}. But importantly, the final (projected) estimators of the two approaches are significantly different. With a simpler projection operation and our refined analysis framework, we can immediately attain an improved regret guarantee for weight-based algorithm.

\begin{myThm}
  \label{thm:GLB-regret}
  Let the weighted factor $w_{t,s} = \gamma^{t-s}$, where $\gamma\in(1/T,1)$, and let $\lambda = d/c_\mu^2$, the regret of \GLBweightours (Algorithm~\ref{alg:GLB-WeightUCB}) is bounded with probability at least $1-1/T$ by
  \begin{equation}\nonumber
  \begin{split}
      \DReg_T\leq{}& \Ot\sbr{k_\mu^2\frac{1}{(1-\gamma)^{\sfrac{3}{2}}}P_T + \frac{k_\mu}{c_\mu}d(1-\gamma)^{\sfrac{1}{2}}T }.
  \end{split}
  \end{equation}
  By optimally setting $\gamma = 1 - \max\{1/T, \sqrt{k_\mu c_\mu P_T/(dT)}\}$, \GLBweightours achieves the following dynamic regret,
  \begin{equation}\nonumber
      \DReg_T \leq
      \begin{cases}
      \Ot\sbr{\frac{k_\mu^{5/4}}{c_\mu^{\sfrac{3}{4}}}d^{\sfrac{3}{4}} P_T^{\sfrac{1}{4}}T^{\sfrac{3}{4}}}  \mbox{ when } P_T \geq \frac{d}{k_\mu c_\mu T},&\vspace{2mm}\\
      \Ot\sbr{\frac{k_\mu}{c_\mu}d\sqrt{T}}  \quad\quad\quad \mbox{ when } 0\leq P_T < \frac{d}{k_\mu c_\mu T}.&
      \end{cases} 
  \end{equation}
\end{myThm}
Compared to \GLBweight (the best-known weight-based algorithm for drifting GLB)~\cite{arXiv'21:faury-driftingGLB}, focusing on the dependence on $d$, $P_T$, and $T$, we can see that our approach improves the regret from $\Ot(d^{\sfrac{9}{10}} P_T^{\sfrac{1}{5}}T^{\sfrac{4}{5}})$ to $\Ot(d^{\sfrac{3}{4}} P_T^{\sfrac{1}{4}}T^{\sfrac{3}{4}})$. Furthermore, our result also improves their result upon the $c_\mu$ dependence from $c_\mu^{-1}$ to $c_\mu^{-\sfrac{3}{4}}$.

\subsection{Self-Concordant Bandits}
\label{sec:SCB}
This section studies Self-Concordant Bandits (SCB), an important subclass of GLB with many attractive structures. For SCB, the reward's distribution belongs to a canonical exponential family: ${\P_\theta\mbr{r \givenn \x}} = \exp(r\x^\T\theta - b(\x^\T\theta) +c(r))$ where $b(\cdot)$ is a twice continuously differentiable function and $c(\cdot)$ is a real-valued function. Owing to the benign properties of exponential families, we have $\E\mbr{r\given \x} = b^{\prime}(\x^\T\theta)$ and $\Var\mbr{r\given \x} = b^{\prime\prime}(\x^\T\theta)$ where $b^\prime$ denotes the first derivative of the function $b$, and $b^{\prime\prime}$ denotes its second derivative. Then, we can introduce the (inverse) link function $\mu(\cdot) \define b^{\prime}(\cdot)$ such that
\begin{equation}
  \label{eq:SCB-model-assume}
       \E\mbr{r_t\given X_t} = \mu(X_t^\T \theta_t), \Var\mbr{r_t\given X_t} = \dmu(X_t^\T \theta_t).
\end{equation}
SCB requires the link function satisfy $|\ddmu| \leq \dmu$, usually referred to general self-concordant property. We further introduce the notation $\eta_t = r_t - \mu(X_t^\T \theta_t)$ to denote the noise. SCB successfully models many important real-world applications and captures the reward structure. For example, choosing $\mu(x) = (1+e^{-x})^{-1}$ yields the Logistic Bandits (LogB), which is often adopted to model the binary-feedback reward in recommendation system~\cite{ICML'16:ZhangYJXZ16, NIPS'17:online-GLB, COLT'19:TS_LogB}.

We make several standard assumptions same as LB and GLB, including boundedness of feasible set and unknown parameters (Assumption~\ref{ass:bounded-norm}), and non-linearity measure on the (inverse) link function (Assumption~\ref{ass:link-function}). In addition, similar to~\cite{AISTATS'21:SCB-forgetting}, we need assumptions on boundedness of reward, and for the convenience of analysis we let $L = 1$ which means $\|\x\|_2 \leq 1$ for all $\x \in \X$.
\begin{myAssum} \label{ass:bounded-rewards}
  The reward received at each round satisfies $0\leq r_t \leq m$ for all $t\in [T]$ and some constant $m>0$.
\end{myAssum}

\parag{Algorithm.~} We propose the \SCBweightours algorithm. Compared to GLB, we use a new local norm for projection and regret analysis which is the key to improving the order of $c_\mu^{-1}$. At iteration $t$, we first adopt the same maximum likelihood estimator as GLB which is defined in~\eqref{eq:GLB-estimator}. Different from GLB, here we use a new local norm to perform the projection onto the feasible set $\Theta$,
\begin{equation}
  \label{eq:SCB-projection}
    \thetat_t = \argmin_{\theta \in \Theta}\norm{g_t(\thetah_t) - g_t(\theta)}_{H_{t}^{-1}(\theta)},
\end{equation}
where $g_t(\theta)$ is the same as~\eqref{eq:GLB-gt} while $H_t(\theta)$ is defined as
\begin{equation}
  \begin{split}
      \label{eq:SCB-Ht}
      H_t(\theta) \define \lambda c_\mu I_d + \sum_{s=1}^{t-1}w_{t-1,s}\dmu(X_s^\T \theta)X_sX_s^\T.
  \end{split}
\end{equation}
Notably, compared to $V_t$, $H_t(\theta)$ depends on the function curvature along the dynamics and thus captures more \emph{local} information. Combining this projection step with the standard self-normalized concentration restated in~\pref{thm:self-normalized-weight-SCB} removes a constant $c_\mu^{-\sfrac{1}{2}}$ in the regret. For estimator~\eqref{eq:GLB-estimator} with projection~\eqref{eq:SCB-projection}, we construct the following estimation error bound.
\begin{myLemma}
  \label{lemma:SCB-estimation-error}
  For any $\x \in \X$, $\gamma\in(0,1)$ and $\delta \in (0,1)$, with probability at least $1-\delta$, the following holds for all $t \in [T]$
  \begin{equation}\nonumber
    \begin{split}
      {}&\abs{\mu(\x^\T\thetat_t) - \mu(\x^\T\theta_t)} \\ \leq {}&\frac{\sqrt{4+8S} k_\mu}{\sqrt{c_\mu}}\sbr{\sum_{p=1}^{t-1}\frac{C(p)}{\sqrt{c_\mu}}\norm{\theta_p - \theta_{p+1}}_2 + \betat_{t-1}\|\x\|_{V_{t-1}^{-1}}},
    \end{split}
  \end{equation}
  where $C(p) \define k_\mu L^2\sqrt{\frac{d}{\lambda}}\sqrt{\sum_{s=1}^{p}w_{t-1,s}}$, and $\betat_t$ is the radius of confidence region set by
\begin{equation}
  \label{eq:SCB-confidence-radius}
  \begin{split}
    \betat_t &={} \frac{\sqrt{\lambda c_\mu}}{2 m }+\frac{2 m }{\sqrt{\lambda c_\mu}}\sbr{\log\frac{1}{\delta}+d \log 2}\\
    {}&+\frac{d m }{\sqrt{\lambda c_\mu}} \log \sbr{1+\frac{ L^2k_\mu\sum_{s=1}^{t}w_{t,s}}{\lambda c_\mu d}}+\sqrt{\lambda c_\mu} S.
  \end{split}
\end{equation}
\end{myLemma}

\iffullversion The proof of Lemma~\ref{lemma:SCB-estimation-error} is in Appendix~\ref{sec:SCB-estimation-error-proof}.\fi Based on Lemma~\ref{lemma:SCB-estimation-error}, we can select the arm $X_t$ as
\begin{equation}
\label{eq:SCB-select-criteria}
    \argmax_{\x \in \X} \bigg\{ \mu(\x^\T \thetat_t) + 2\sqrt{1+2S}\frac{k_\mu}{\sqrt{c_\mu}}\betat_{t-1}\|\x\|_{V_{t-1}^{-1}} \bigg\}.
\end{equation}
Our algorithm for SCB (named \SCBweightours) follows the same procedure of Algorithm~\ref{alg:GLB-WeightUCB}, and the difference is that $\thetat_t$ is computed by~\eqref{eq:SCB-projection}, $\betat_{t-1}$ is computed by~\eqref{eq:SCB-confidence-radius} and $X_t$ is computed by~\eqref{eq:SCB-select-criteria}. Further, we have the following guarantee for \SCBweightours algorithm.
\begin{myThm}
  \label{thm:SCB-regret}
  For all $\gamma\in(1/T,1)$, $\lambda = d\log(T)/c_\mu$, the dynamic regret of \SCBweightours is bounded with probability at least $1-1/T$, by
  \begin{equation}\nonumber
  \begin{split}
    \DReg_T\leq\Ot\sbr{\frac{k_\mu^2}{\sqrt{c_\mu}}\frac{1}{(1-\gamma)^{\sfrac{3}{2}}}P_T + \frac{k_\mu}{\sqrt{c_\mu}}d(1-\gamma)^{\sfrac{1}{2}}T}.
  \end{split}
  \end{equation}
  By setting $\gamma = 1-\max\{1/T, \sqrt{k_\mu P_T/(dT)}\}$, we achieve
  \begin{equation}\nonumber
      \DReg_T \leq
      \begin{cases}
      \Ot\sbr{\frac{k_\mu^{\sfrac{5}{4}}}{c_\mu^{\sfrac{1}{2}}}d^{\sfrac{3}{4}} P_T^{\sfrac{1}{4}}T^{\sfrac{3}{4}}} & \mbox{ when } P_T \geq \frac{d}{k_\mu T},\vspace{2mm}\\
      \Ot\sbr{\frac{k_\mu}{c_\mu^{\sfrac{1}{2}}}d\sqrt{T}} & \mbox{ when } 0\leq P_T < \frac{d}{k_\mu T}.
      \end{cases} 
  \end{equation}
\end{myThm}
\noindent Compared to GLB, we improve the order of $c_\mu$ from $c_\mu^{-1}$ to $c_\mu^{-\sfrac{1}{2}}$  by exploiting the self-concordant properties. In near-stationary environments ($P_T$ is small enough), our result can recover to the performance of \SCBstatic algorithm~\cite{ICML'20:logistic-bandits}. \iffullversion Proof of~\pref{thm:SCB-regret} is presented in Appendix~\ref{sec:SCB-regret-proof}.\fi

In addition, for the piecewise-stationary SCB, we propose \SCBweightourspw algorithm that gets rid of influence of $c_\mu$ and thus directly improves upon~\cite{AISTATS'21:SCB-forgetting}.
\begin{myThm}
  \label{thm:SCB-PW-regret}       
  For all $\gamma\in(1/2,1)$, $D = \log(T)/\log(1/\gamma)$ and $\lambda = d\log(T)/c_\mu$, the regret of SCB-PW-WeightUCB is bounded with probability at least $1-1/T$, by
  \begin{equation}\nonumber
  \begin{split}
    \DReg_T\leq{}& \Ot\sbr{\frac{1}{1-\gamma}\Gamma_T + \frac{1}{\sqrt{1-\gamma}} + d\sqrt{(1-\gamma)}T}.
  \end{split}
  \end{equation}
  By setting $\gamma = 1-\max\{1/T, \sbr{\Gamma_T/(dT)}^{\sfrac{2}{3}}\}$, we achieve
  \begin{equation}\nonumber
      \DReg_T \leq
      \begin{cases}
      \Ot\sbr{d^{\sfrac{2}{3}}\Gamma_T^{\sfrac{1}{3}}T^{\sfrac{2}{3}}} & \mbox{ when } \Gamma_T \geq d/\sqrt{T},\vspace{2mm}\\
      \Ot\sbr{d\sqrt{T}} & \mbox{ when } 0\leq \Gamma_T < d/\sqrt{T}.
      \end{cases} 
  \end{equation}
\end{myThm}   
\iffullversion \noindent The overall algorithm and analysis are in Appendix~\ref{sec:SCB-PW}.\fi
\section{Linear Mixture MDP}
\label{sec:LMDP}
In this section, we apply the weighted strategy to non-stationary linear mixture MDPs, and describe our \LMDPweightours algorithm and its theoretical guarantee. Our algorithm achieves the same regret bound as the previous restart-based algorithm~\cite{arxiv'22:ns_lin_mdp}.

\subsection{Problem Setting}
\label{sec:LMDP-problem-setting}
We build upon the previously established definition of episodic non-stationary MDPs~\cite{arxiv'22:ns_lin_mdp} and provide the corresponding learning protocol.

\parag{Episodic Non-stationary MDPs.} An episodic MDP is defined by a tuple $M = (\S, \A, H, \P, r)$, where $\mathcal{S}$ is the state space; $\mathcal{A}$ is the action space; $H$ is the length of each episode; $\P = \bbr{\P_h^k}_{h\in[H],k\in[K]}$, where $\P_h^k: \S\times \A \times \S \rightarrow [0,1]$ is the transition probability at $h$-th step of $k$-th episode; and $r = \bbr{r_h^k}_{h\in[H],k\in[K]}$, where $r_h^k: \S\times\A \rightarrow [0,1]$ is the reward function at $h$-th step of $k$-th episode. A policy is defined as $\pi = \bbr{\pi_h^k}_{h\in[H],k\in[K]}$, where each $\pi_h^k:\S\rightarrow\Delta(\A)$ is a function that maps a state $s$ to distributions over action space $\A$ at stage $h$ of the $k$-th episode. 

\parag{Learning Protocol.} At the beginning of each episode $k$, the learner chooses a policy $\pi^k = \bbr{\pi_{h}^k}_{h=1}^H$. At each stage $h\in[H]$, starting from the initial stage $s_{1}^k$, the learner observes the state $s_{h}^k$, chooses an action $a_{h}^k$ sampled from $\pi_h^k(s_{h}^k)$, obtains reward $r_h^k(s_{h}^k,a_{h}^k)$ and transitions to the next state $s_{h+1}^k\sim\P_h^k(\cdot\given s_{h}^k,a_{h}^k)$ for $h\in[H]$. The episode ends when $s_{H+1}^k$ is reached; when this happens, no action is taken, and the reward is equal to zero. For any policy $\pi = \bbr{\pi_h^k}_{h\in[H],k\in[K]}$ and $(s,a)\in\S\times \A$, we define the action-value function $Q_h^{k,\pi}$ and value function $V_h^{k,\pi}$ as 
\begin{equation*}
    \begin{split}
        Q_h^{k,\pi}(s, a)&={} \mathbb{E}\left[\sum_{h^{\prime}=h}^H r_{h^{\prime}}^k\big(s_{h^{\prime}}^k, \pi_{h^\prime}^k\left(s_{h^{\prime}}^k\right)\big) \givenn s_h^k=s, a_h^k=a\right]\\
        V_h^{k,\pi}(s) &={} \E_{a\sim\pi_h^k(\cdot\given s)}\mbr{Q_h^{k,\pi}(s,a)}.
    \end{split}
\end{equation*}
We define $\forall V: \S\rightarrow \R$, $\mbr{\P_h^k V}(s,a) = \E_{s^\prime\sim\P_h^k(\cdot\given s,a)} V(s^\prime)$, and the Bellman equation for policy $\pi$ is given by
\begin{equation*}
    \begin{split}
        Q_h^{k,\pi}(s,a) &=r_h^k(s, a)+\mbr{\P_h^k V_{h+1}^{k,\pi}}(s,a)\\
        V_{h}^{k,\pi}(s) &=\E_{a\sim\pi_h^k(\cdot\given s)}\mbr{Q_h^{k,\pi}(s,a)},\quad V_{H+1}^{k,\pi} = 0.
    \end{split}
\end{equation*}
The learner's goal is to minimize the following dynamic regret,
\begin{equation}\label{eq:LMDP-regret}
    \begin{split}
        \DReg_T=\sum_{k=1}^K V_1^{k,\pi_*^k}\left(s_1^k\right)-\sum_{k=1}^K V_1^{k,\pi^k}\left(s_1^k\right),
    \end{split}
\end{equation}
where we denote $T \define H\cdot K$, for consistency with the bandits notations. The dynamic regret measures the difference between the learner's policy and the optimal policy, namely, $\pi_*^k = \argmax_{\pi}V_{1}^{k,\pi}(s_1^k)$.

\parag{Non-stationary Linear Mixture MDP.~~} An MDP instance $M = (\S, \A, H, \P, r)$ is a linear mixture MDP if there exist known feature maps $\phi:\S\times\A \rightarrow \R^d$ and $\psi: \S\times\A\times\S\rightarrow \R^d$ and for any $k\in[K],h\in[H]$, there exist unknown vectors $\theta_h^k\in\R^d$ and $\w_h^k\in\R^d$ such that 
\begin{equation}\label{eq:LMDP-model}
    \begin{split}
        r_h^k(s,a) &= \inner{\phi(s,a)}{\theta_h^k}\\ \P_h^k(s^\prime \given s,a) &= \inner{\psi(s^\prime\given s,a)}{\w_h^k},
    \end{split}
\end{equation}
here we consider the drifting case where we use path length $P_T^\theta =  \sum_{k=2}^{K} \sum_{h=1}^{H}\norm{\theta_h^{k-1}-\theta_h^k}_2$ to measure the non-stationarity of $\theta_h^k$ and $P_T^\w = \sum_{k=2}^{K} \sum_{h=1}^{H}\norm{\w_h^{k-1}-\w_h^k}_2$ to measure the non-stationarity of $\w_h^k$, and we define $\Delta = P_T^\theta + P_T^\w$ as the total path length. We work under the following standard boundedness assumption~\cite{arxiv'22:ns_lin_mdp}.
\begin{myAssum}\label{ass:LMDP-bounded-norm} 
The feasible set and unknown parameters are assumed to be bounded: $\forall s\in\S, a\in\A, \norm{\phi(s,a)}\leq L_{\phi}$; for any bounded function $V: \S\rightarrow [0,1]$ and $\forall s\in\S, a\in\A$, $\norm{\psi_V(s,a)}_2\leq L_\psi, \text{where~} \psi_V(s,a) \define\sum_{s^\prime\in\S}\psi(s^\prime\given s,a) V(s^\prime)$; $\theta_h^k\in \Theta$ holds for all $k\in[K],h\in[H]$ where $\Theta \define \bbr{\theta \given \norm{\theta}_2\leq S_\theta}$ and $\w_h^k\in \W$ holds for all $k\in[K],h\in[H]$ where $\W \define \bbr{\w \given \norm{\w}_2\leq S_\w}$.
\end{myAssum}

\subsection{Algorithm and Regret Guarantee}
\label{sec:LMDP-algorithm}
We propose the \LMDPweightours algorithm in this section. We first give the employed reward estimator and transition estimator, and then derive its estimation error upper bound by our refined analysis framework, which is the key for algorithm design and regret analysis. Based on the estimation error bound, we propose our selection criterion and finally give the theoretical guarantee on its dynamic regret.

\parag{Reward Estimator.~}
At $h$-th stage of $k$-th episode, we use the following weighted least square estimator $\th_{h}^k$ to estimate the unknown parameter $\theta_h^k$, which is the minimizer of
\begin{equation*}
    \begin{split}
        \frac{\lamt}{2}\norm{\theta}_2^2 + \sum_{j=1}^{k-1}w_{k-1,j}\sbr{r_h^j(s_h^j,a_h^j) - \phi(s_h^j,a_h^j)^\T\theta}^2,
    \end{split}
\end{equation*}
where $\lambda_\theta > 0$ is the regularization coefficient and $w_{k-1,j}$ is the weighted factor. Similar to the Estimator~\eqref{eq:LB-estimator}, we set $w_{k,j} = \gamma^{k-j}$, where $\gamma \in (0, 1)$ is the discounted factor. Clearly, $\th_{h}^k$ has a closed-form solution:
\begin{equation}\label{eq:LMDP-estimate-r}
    \begin{split}
        \th_{h}^k = \sbr{\Lambda_{h}^{k-1}}^{-1}\sbr{\sum_{j=1}^{k-1}w_{k-1,j}r_h^j(s_h^j,a_h^j)\phi(s_h^j,a_h^j)},
    \end{split}
\end{equation}
where $\Lambda_{h}^k = \lamt I_d + \sum_{j=1}^{k}w_{k,j}\phi(s_h^j,a_h^j)\phi(s_h^j,a_h^j)^\T$.

\parag{Transition Estimator.} First notice that, based on model assumption~\eqref{eq:LMDP-model}, we have
\begin{align}
        \mbr{\mathbb{P}_h^k V_{h+1}^k}\sbr{s_h^k, a_h^k} ={}& \sum_{s^\prime}\inner{\psi(s^\prime\given s_h^k, a_h^k)}{\w^k_h}V_{h+1}^k(s^\prime)\nonumber\\={}& \inner{\sum_{s^\prime}\psi(s^\prime\given s_h^k, a_h^k)V_{h+1}^k(s^\prime)}{\w^k_h}\nonumber \\={}& \inner{\psi_{h+1}^k\left(s_h^k, a_h^k\right)}{ \w^k_h},\label{eq:LMDP-linear-transtion}
\end{align}
where we denote $\psi_{h+1}^k\left(s,a\right) \define \psi_{V_{h+1}^k}(s,a)$ for simplicity. Based on Assumption~\ref{ass:LMDP-bounded-norm}, we know that $\forall s\in\S, a\in\A, \norm{\psi_{h+1}^k\left(s,a\right)}_2\leq HL_\psi$. We adopt the following weighted least square estimator $\wh_{h}^k$ to estimate the unknown parameter $\w_h^k$, which is the minimizer of
\begin{equation*}
    \sum_{j=1}^{k-1}\alpha_{k-1,j}\sbr{\langle\psi_{h+1}^j(s_h^j, a_h^j),\w\rangle - V_{h+1}^j(s_{h+1}^j)}^2 + \frac{\lamw}{2}\norm{\w}_2^2,
\end{equation*}
where $\lamw>0$ is the regularization coefficient and $\alpha_{k-1,j}$ is the weighted factor, we set $\alpha_{k,j} = \gamma^{k-j}$, where $\gamma \in (0, 1)$ is the discounted factor. For this least square estimator, we have a closed-form solution for $\wh_{h}^k$:
\begin{equation}\label{eq:LMDP-estimate-v}
    \begin{split}
        \sbr{\Sigma_{h}^{k-1}}^{-1}\sbr{\sum_{j=1}^{k-1}\alpha_{k-1,j}V_{h+1}^j(s_{h+1}^j)\psi_{h+1}^j\sbr{s_h^j, a_h^j}},
    \end{split}
\end{equation}
where $\Sigma_{h}^k = \lamw I_d + \sum_{j=1}^{k}\alpha_{k,j}\psi_{h+1}^j(s_h^j, a_h^j)\psi_{h+1}^j(s_h^j, a_h^j)^\T$.

\begin{algorithm}[!t]
    \caption{\LMDPweightours}
    \label{alg:LMDP-weight}
    \begin{algorithmic}[1]
    \REQUIRE episode number $K$, time horizon $H$, discounted factor $\gamma$, confidence $\delta$, regularizer $\lamt,\lamw$, parameters $S_\theta, S_\w$, $L_\phi, L_\psi$\\
    \STATE Initialize $\bbr{\pi_h^0}_{h=1}^H$ as uniform distribution policies, $\bbr{Q_h^0}_{h=1}^H$ as zero functions.
    \STATE Set $\forall h\in[H], \Lambda_{h}^0 = \lamt I_d, \Sigma_{h}^0 = \lamw I_d$ for $k=1$
    \FOR{$k = 1,2,...,K$}
    \STATE Receive the initial state $s_1^k$
    \STATE Initialize $V_{H+1}^k$ as zero function
    \FOR{$h = H,H-1,...,1$}
    \STATE $\psi_{h+1}^k\left(\cdot, \cdot\right) = \sum_{s^\prime}\psi(s^\prime\given \cdot, \cdot)V_{h+1}^k(s^\prime)$
    \STATE Update $\Lambda_{h}^k = \gamma \Lambda_{h}^{k-1} + \phi(s_h^k,a_h^k)\phi(s_h^k,a_h^k)^\T +(1-\gamma) \lamt I_d $
    \STATE Update $\Sigma_{h}^k = \psi_{h+1}^k\sbr{s_h^k, a_h^k}\psi_{h+1}^k\sbr{s_h^k, a_h^k}^\T + \gamma\Sigma_{h}^{k-1} + (1-\gamma) \lamw I_d$
    \STATE Compute $\th_{h}^k$ by~\eqref{eq:LMDP-estimate-r} and $\wh_{h}^k$ by~\eqref{eq:LMDP-estimate-v}
    \STATE Compute optimistic value function $Q_h^k(s,a)$ and $V_h^k(s)$ by~\eqref{eq:LMDP-optimistic-Q}
    \ENDFOR
    \FOR{$h = 1,...,H$}
    \STATE Choose policy as $\pi_h^k(s)=\argmax_{a\in\A}Q_{h}^k(s,a)$
    \STATE Take action $a_h^k\sim \pi_h^k(s_h^k)$, then observe the reward $r_h^k(s_h^k,a_h^k)$ and receive the next state $s_{h+1}^k$
    \ENDFOR
    \ENDFOR
    
    \end{algorithmic}
    \end{algorithm}

\parag{Upper Confidence Bounds.} 
For estimator~\eqref{eq:LMDP-estimate-r} and~\eqref{eq:LMDP-estimate-v}, we provide the estimation error bounds, respectively.
\begin{myLemma}
    \label{lemma:LMDP-r-ucb}
    For any $s\in\S, a\in\A$, the following holds for all $k \in [K], h\in[H]$,
    \begin{equation*}
        \begin{split}
            \abs{\phi(s,a)^\T\sbr{\th_{h}^k -\theta_{h}^k}} \leq \Gamma_{h,\theta}^{k-1} + \beta_\theta\norm{\phi(s,a)}_{\sbr{\Lambda_{h}^{k-1}}^{-1}},
        \end{split}
    \end{equation*}
    where $\Gamma_{h,\theta}^{k-1} \define L_{\phi}^2\sqrt{\frac{d}{\lamt}}\sum_{p=1}^{k-1}\sqrt{\sum_{j=1}^{p}w_{k-1,j}}\norm{\theta_h^p-\theta_h^{p+1}}_2$, $\beta_\theta \define \sqrt{\lamt}S_\theta$. 
\end{myLemma}

\begin{myLemma}
    \label{lemma:LMDP-v-ucb}
    For any $s\in\S, a\in\A$, and $\delta \in (0,1)$, with probability at least $1-\delta$, the following holds for all $k \in [K], h\in[H]$
    \begin{align*}
    &\abs{\psi_{h+1}^k\sbr{s, a}^\T \sbr{\wh_{h}^k - \w_h^k}} \\
    & \leq \Gamma_{h,\w}^{k-1} + \beta_{\w}^{k-1}\norm{\psi_{h+1}^k\sbr{s, a}}_{\sbr{\Sigma_{h}^{k-1}}^{-1}},
    \end{align*}
    $\Gamma_{h,\w}^{k-1} \define  H^2L_{\psi}^2\sqrt{\frac{d}{\lamw}}\sum_{p=1}^{k-1}\sqrt{\sum_{j=1}^{p}\alpha_{k-1,j}}\norm{\w_h^p-\w_h^{p+1}}_2$, and $\beta_{\w}^{k}$ is the radius of confidence region defined by
    \begin{equation*}
      \begin{split}
      H \sqrt{\frac{1}{2}\log \frac{1}{\delta}+\frac{d}{4}\log\left(1+\frac{H^2 L_\psi^2\sum_{j=1}^{k}\alpha_{k,j}}{\lamw d}\right)} +\sqrt{\lamw}S_\w.
    \end{split}
    \end{equation*}
\end{myLemma}
\iffullversion \noindent Proof of Lemma~\ref{lemma:LMDP-r-ucb} and Lemma~\ref{lemma:LMDP-v-ucb} are in Appendix~\ref{app:LMDP-lemma-r-ucb},~\ref{app:LMDP-lemma-v-ucb}.\fi

\parag{Arm Selection.~} Based on Lemma~\ref{lemma:LMDP-r-ucb} and Lemma~\ref{lemma:LMDP-v-ucb}, we define the optimistic value function $Q_h^k(s,a)$ and $V_{h}^k(s)$ as follows,
\begin{align}
      Q_{h}^k(s,a) &\define \min\Bigg\{H, \phi(s,a)^\T\th_{h}^k+\beta_\theta\norm{\phi(s,a)}_{\sbr{\Lambda_{h}^{k-1}}^{-1}}\nonumber \\&\hspace{-2em}+ \psi_{h+1}^k\sbr{s, a}^\top \wh_{h}^k+\beta_{\w}^{k-1}\norm{\psi_{h+1}^k\sbr{s, a}}_{\sbr{\Sigma_{h}^{k-1}}^{-1}}\Bigg\}\nonumber\\
      \label{eq:LMDP-optimistic-Q} V_{h}^k(s) &\define \max_{a\in\A}Q_{h}^k(s,a) = \E_{a\sim \pi_h^k(\cdot\given s)}\mbr{Q_{h}^k(s,a)}.
\end{align}
At state $s_h^k$, we can specify the action selection criterion of our policy $\pi_h^k(s_h^k)$ as $a_h^k=\argmax_{a\in\A}Q_{h}^k(s_h^k,a)$. The overall algorithm is summarized in Algorithm~\ref{alg:LMDP-weight}. We show that our algorithm enjoys the following regret guarantee.
\begin{myThm}\label{thm:LMDP-regret-bound}
    Let $T = KH$, $\delta = 1/(4T)$, $\lambda_\theta = d$, and $\lambda_\w = H^2 d$, $\forall k,j\in [K], w_{k,j} = \alpha_{k,j} = \gamma^{k-j}$, $\gamma\in(1/K,1)$, the dynamic regret $ \DReg_T$ is bounded with probability at least $1-1/T$ by
    \begin{equation*}
        \begin{split}
            \O\Bigg(Hd\sbr{\frac{1}{(1-\gamma)^{3 / 2}} \Delta + HK\sqrt{1-\gamma}} + H^{3/2}d\sqrt{HK}\Bigg),
    \end{split}
    \end{equation*}
    Furthermore, by setting the discounted factor optimally as $\gamma = 1-\max\bbr{1/K,\sqrt{\Delta/T}}$, we have
\begin{equation*}
    \DReg_T \leq \begin{cases}\Ot\sbr{Hd \Delta^{1/4}T^{3/4}} & \text { when } \Delta \geq H / K, \\ \Ot\sbr{dH^{3/2}\sqrt{T}}& \text { when } \Delta<H / K .\end{cases}
    \end{equation*}
\end{myThm}
\iffullversion The proof of Theorem~\ref{thm:LMDP-regret-bound} is in Appendix~\ref{app:proof-LMDP-regret}. \fi Compared to previous work~\cite{arxiv'22:ns_lin_mdp}, our results achieve the same order regret guarantees for dynamic regret in non-stationary environments. Furthermore, in near-stationary settings, our results recover the theoretical guarantees established for stationary environments~\cite{L4DC'20:lin_mix_mdp,ICML'20:lin_mix_mdp}.
\section{Multinomial Logit Mixture MDP}
\label{sec:MNL-MDP}
In this section, we explore another class of MDPs, known as the Multinomial Logit (MNL) Mixture MDP, under the non-stationary setting. We introduce the \MNLweightours algorithm, which applies the weighted strategy to non-stationary MNL Mixture MDPs, and provide the first theoretical guarantee for non-stationary MNL Mixture MDP.

\subsection{Problem Setting}
\label{sec:MMDP-problem-setting}
To address the limitation that linear function approximation cannot guarantee valid distribution, a new class called MNL mixture MDPs has been proposed recently~\cite{AAAI'23:MNL_MDP,NIPS'24:MNL_MDP_efficient}. 
The MNL mixture MDPs share the same episodic non-stationary MDPs structure and learning protocol as the linear mixture MDP, with the objective of minimizing \blue{dynamic regret}~\eqref{eq:LMDP-regret}. The key distinction lies in its modeling assumptions for the transition probabilities. Below, we present the formal definition of MNL Mixture MDPs.

\begin{myDef}[Reachable States] For any $(k, h, s, a) \in[K]\times [H] \times \S \times \A$, we define the "reachable states" as the set of states that can be reached from state $s$ taking action $a$ at stage $h$ of $k$-th episode within a single transition, i.e., $\S_{h}^k(s,a) \define\left\{s^{\prime} \in \S \given \P_h^k\left(s^{\prime} \given s, a\right)>0\right\}$. Also, we define $S_{h}^k(s,a) \define\abs{\S_{h}^k(s, a)}$ and further define $U \define \max _{(k, h, s, a)} S_{h}^k(s,a)$ as the maximum number of reachable states.
\end{myDef}

\parag{Non-stationary MNL Mixture MDP.~~} $M = (\S, \A, H, \P, r)$ is a MNL mixture MDP if there exist known feature maps $\phi:\S\times\A \rightarrow \R^d$ and $\psi: \S\times\A\times\S\rightarrow \R^d$ and for any $k\in[K],h\in[H]$, there exist unknown vectors $\theta_h^k\in\R^d$ and $\w_h^k\in\R^d$ such that 
\begin{equation}\label{eq:MNL-model}
    \begin{split}
        r_h^k(s,a) &= \inner{\phi(s,a)}{\theta_h^k}\\ \P_h^k(s^\prime \given s,a) &= \frac{\exp\sbr{\psi(s^\prime\given s,a)^\T \w_h^k}}{\sum_{\tilde{s}\in \S_{h}^k(s,a)}\exp\sbr{\psi(\tilde{s}\given s,a)^\T \w_h^k}}.
    \end{split}
\end{equation}
We work under the following standard assumptions of MNL Mixture MDP~\cite{AAAI'23:MNL_MDP,NIPS'24:MNL_MDP_efficient}.
\begin{myAssum}\label{ass:MNL-bounded-norm} 
    The feasible set and unknown parameters are assumed to be bounded: $\forall s\in\S, a\in\A, \norm{\phi(s,a)}\leq L_{\phi}$; $\forall s^\prime s\in\S, a\in\A, \norm{\psi(s^\prime\given s,a)}\leq L_{\psi}$; $\theta_h^k\in \Theta$ holds for all $k\in[K],h\in[H]$ where $\Theta \define \bbr{\theta \given \norm{\theta}_2\leq S_\theta}$ and $\w_h^k\in \W$ holds for all $k\in[K],h\in[H]$ where $\W \define \bbr{\w \given \norm{\w}_2\leq S_\w}$.
    \end{myAssum}
\begin{myAssum}\label{ass:MNL-kappa}
    There exists $0<\kappa<1$ such that for all $(s,a,h)\in\S\times\A\times[H]$ and $s^\prime,s^{\prime\prime}\in\S_{h}^k(s,a)$, it holds that $\inf_{\w\in \W}p_{s,a}^{s^\prime}(\w)p_{s,a}^{s^{\prime\prime}}(\w) \geq \kappa$.
\end{myAssum}

\subsection{Algorithm and Regret Guarantee}
\label{sec:MNL-algorithm}
We propose the \MNLweightours algorithm, which obtains the first dynamic regret guarantee for non-stationary MNL Mixture MDPs. We begin by presenting the estimator used in our approach and deriving its estimation error upper bound by our refined analysis framework, which is the key for algorithm design and regret analysis. Building on the estimation error bound, we propose our selection criterion and finally give the theoretical guarantee on its dynamic regret.

\parag{Reward Estimator.~} Since we use the same linear function as~\eqref{eq:LMDP-model} to model the reward function, here we still use estimator~\eqref{eq:LMDP-estimate-r} to estimate the unknown parameter $\theta_h^k$.

\parag{Transition Estimator.~} For the trajectory $\{(s_h^k,a_h^k)\}_{h=1}^H$ at episode $k$, we define the variable: $y_h^k\in\{0,1\}^{S_h^k}$, where $y_h^k(s^\prime) = \indicator(s_{h+1}^k = s^\prime)$ for $s^\prime \in \S_{h}^k\define\S_{h}^k(s_h^k,a_h^k)$ and $S_h^k \define \abs{\S_{h}^k}$. Furthermore, we denote $p_h^k(\psi(s^\prime\given s,a)^\T \w) = \frac{\exp\sbr{\psi(s^\prime\given s,a)^\T \w}}{\sum_{\tilde{s}\in \S_{h}^k(s,a)}\exp\sbr{\psi(\tilde{s}\given s,a)^\T \w}}$. We define $\psib_h^k(s^\prime) \define \psi(s^\prime\given s_h^k,a_h^k)$. Then $y_h^k$ is a sample from the multinomial distribution:
\begin{equation*}
    \begin{split}
        \operatorname{multinomial}\left(1,\left[p_h^k(\psib_h^k(s_1)^\T \w), \ldots, p_h^k(\psib_h^k(s_{N_{h}^k})^\T \w)\right]\right).
    \end{split}
\end{equation*}
We use the following weighted maximum likelihood estimation (MLE) $\wh_{h}^k$ to estimate the unknown parameter $\w_h^k$, which is the minimizer of
\begin{equation*}
    \begin{split}
        \frac{\lamw \kappa}{2}\norm{\w}_2^2 + \sum_{j=1}^{k-1}\alpha_{k-1,j} \sum_{s^\prime \in \S_{h}^j} - y_{h}^{j}(s^\prime)\log p_h^j(\psib_h^j(s^\prime)^\T \w),
    \end{split}
\end{equation*}
where $\lamw>0$ is the regularization coefficient and $\alpha_{k-1,j}$ is the weighted factor. We set $\alpha_{k,j} = \gamma^{k-j}$, where $\gamma \in (0, 1)$ is the discounted factor. Specifically the estimator $\wh_{h}^k$ is the solution of the following equation:
\begin{equation}\label{eq:MNL-estimate-p}
    \begin{split}
        \sum_{j=1}^{k-1}\alpha_{k-1,j} \sum_{s^\prime \in \S_{h}^j}\sbr{p_h^j(\psib_h^j(s^\prime)^\T \w) - y_{h}^{j}(s^\prime)} \psib_h^j(s^\prime)& \\
        + \lamw \kappa\w = 0&.
    \end{split}
\end{equation}
Given that $\wh_h^k$ may not belong to the feasible set $\W$ and $\kappa$ is defined over the parameter $\w \in \W$, we need to perform the following projection step
\begin{equation}\label{eq:MNL-projection}
    \begin{split}
        \wt_h^k = \argmin_{\w\in\W}\norm{g_h^k(\wh_h^k)-g_h^k(\w)}_{\sbr{\Sigmab_h^{k-1}}^{-1}},
    \end{split}
\end{equation}
where $\Sigmab_{h}^k = \lamw I_d + \sum_{j=1}^{k}\alpha_{k,j}\sum_{s^\prime \in \S_{h}^j} \psib_h^j(s^\prime) \psib_h^j(s^\prime)^\T$ and $g_h^k(\w)$ is defined as 
\begin{equation*}
    \begin{split}
        g_h^k(\w) \define \lamw \kappa\w+\sum_{j=1}^{k-1}\alpha_{k-1,j} \sum_{s^\prime \in \S_{h}^j} p_h^j(\psib_h^j(s^\prime)^\T \w) \psib_h^j(s^\prime).
    \end{split}
\end{equation*}

\begin{algorithm}[!t]
    \caption{\MNLweightours}
    \label{alg:MNL-weight}
    \begin{algorithmic}[1]
    \REQUIRE episode number $K$, time horizon $H$, discounted factor $\gamma$, confidence $\delta$, regularizer $\lamt,\lamw$, parameters $S_\theta, S_\w$, $L_\phi, L_\psi$\\
    \STATE Initialize $\bbr{\pi_h^0}_{h=1}^H$ as uniform distribution policies, $\bbr{\Qb_h^0}_{h=1}^H$ as zero functions.
    \FOR{$k = 1,2,...,K$}
    \STATE Receive the initial state $s_1^k$
    \STATE Initialize $\Vb_{H+1}^k$ as zero function
    \STATE Set $\forall h\in[H], \Lambda_{h}^0 = \lamt I_d, \Sigmab_{h}^0 = \lamw I_d$
    \FOR{$h = H,H-1,...,1$}
    \STATE Update $\Lambda_{h}^k = \gamma \Lambda_{h}^{k-1} + \phi(s_h^k,a_h^k)\phi(s_h^k,a_h^k)^\T +(1-\gamma) \lamt I_d $
    \STATE Update $\Sigmab_{h}^k = \sum_{s^\prime \in \S_{h}^j} \psib_h^j(s^\prime\given s_h^k,a_h^k) \psib_h^j(s^\prime\given s_h^k,a_h^k)^\T + \gamma\Sigmab_{h}^{k-1} +  (1-\gamma) \lamw I_d$
    \STATE Compute $\th_{h}^k$ by~\eqref{eq:LMDP-estimate-r} and $\wh_h^k$ by~\eqref{eq:MNL-estimate-p} 
    \IF{$\|\wh_h^k\|_2\leq S_\w$} 
    \STATE let $\wt_h^k = \wh_h^k$
    \ELSE 
    \STATE Do the projection and get $\wt_{h}^k$ by~\eqref{eq:MNL-projection}
    \ENDIF
    \STATE Compute optimistic value function $\Qb_h^k(s,a)$ and $\Vb_h^k(s)$ by~\eqref{eq:MNL-action-selection}
    \ENDFOR
    \FOR{$h = 1,...,H$}
    \STATE Choose policy as $\pi_h^k(s)=\argmax_{a\in\A}\Qb_{h}^k(s,a)$
    \STATE Take action $a_h^k\sim \pi_h^k(s_h^k)$, then observe the reward $r_h^k(s_h^k,a_h^k)$ and receive the next state $s_{h+1}^k$
    \ENDFOR
    \ENDFOR
    \end{algorithmic}
    \end{algorithm}

\parag{Upper Confidence Bounds.} 
Estimator~\eqref{eq:LMDP-estimate-r} can directly apply Lemma~\ref{lemma:LMDP-r-ucb} for UCB construction. And for estimator~\eqref{eq:MNL-estimate-p}, we provide the following estimation error bounds. For simplicity, we define for any function $V: \S\rightarrow \R$, $\mbr{\P_h^k V}(s,a) = \sum_{s^\prime\in\S_h^k}p_h^k(\psi(s^\prime\given s,a)^\T \w_h^k)V(s^\prime)$,~$\mbr{\Pt_h^k V}(s,a) = \sum_{s^\prime\in\S_h^k}p_h^k(\psi(s^\prime\given s,a)^\T \wt_h^k)V(s^\prime)$.

\begin{myLemma}
    \label{lemma:MNL-v-UCB}
    For any $\x \in \X$, and $\delta \in (0,1)$, $\forall k,j\in [K], \alpha_{k,j}\leq 1$, with probability at least $1-\delta$, the following holds for all $k \in [K], h\in[H]$
    \begin{align*}
        &\abs{\mbr{\Pt_h^k V}(s,a)-\mbr{\P_h^k V}(s,a)} \\&\qquad\leq {} \frac{H}{\kappa}\sbr{\Gamma_{h,\w}^{k-1} + \betab_{\w}^{k-1}\max_{s^\prime \in \S_h^k}\norm{\psi(s^\prime\given s,a)}_{\sbr{\Sigmab_{h}^{k-1}}^{-1}}},
    \end{align*}
    $\Gamma_{h,\w}^{k-1}\define  L_{\psi}^2\sqrt{\frac{d}{\lamw}}\sum_{p=1}^{k-1}\sqrt{\sum_{j=1}^{p}\alpha_{k-1,j}}\norm{\w_h^p-\w_h^{p+1}}_2$, and $\betab_{\w}^{k}$ is the radius of confidence region defined by
    \begin{equation*}
        \begin{split}
         \sqrt{\frac{1}{2}\log \frac{1}{\delta}+\frac{d}{4}\log\left(1+\frac{U L_\psi^2\sum_{j=1}^{k}\alpha_{k,j}}{\lamw  d}\right)} +\sqrt{\lamw}\kappa S_\w.
        \end{split}
    \end{equation*}
\end{myLemma}
\iffullversion \noindent Proof of Lemma~\ref{lemma:MNL-v-UCB} is in Appendix~\ref{app:MNL-lemma-v-ucb}.\fi

\parag{Action Selection.~} Based on Lemma~\ref{lemma:LMDP-r-ucb} and Lemma~\ref{lemma:MNL-v-UCB}, we construct the optimistic value function $\Qb_h^k(s,a)$, $\Vb_h^k(s)$ and the action selection criteria as follow,
\begin{align}
      \Qb_{h}^k(s,a) &= \min\Bigg\{H, \phi(s,a)^\T\th_{h}^k+\beta_\theta\norm{\phi(s,a)}_{\sbr{\Lambda_{h}^{k-1}}^{-1}}\nonumber \\
      &\hspace{-2em} + [\Pt_h^k\Vb_{h+1}^k](s,a)+\frac{H}{\kappa}\betab_{\w}^{k-1}\max_{s^\prime \in \S_h^k}\norm{\psib_h^k(s^\prime)}_{\sbr{\Sigmab_{h}^{k-1}}^{-1}}\Bigg\}\nonumber\\
      \Vb_{h}^k(s) &= \max_{a\in\A}\Qb_{h}^k(s,a) = \E_{a\sim \pi_h^k(\cdot\given s)}\mbr{\Qb_{h}^k(s,a)}\nonumber\\
      \label{eq:MNL-action-selection}\pi_h^k(s)&=\argmax_{a\in\A}\Qb_{h}^k(s,a).
\end{align}
The overall algorithm is summarized in Algorithm~\ref{alg:MNL-weight}. We show that our algorithm has the following regret guarantee.
\begin{myThm}\label{thm:MNL-regret-bound}
    Let $\delta = 1/(4T)$, $\lambda_\theta = d$, and $\lambda_\w = d$, $\forall k,j\in [K], w_{k,j} = \alpha_{k,j} = \gamma^{k-j}$, $\gamma\in(1/K,1)$, the dynamic regret $\DReg_T$ is bounded with probability at least $1-1/T$, by
    \begin{equation*}
        \begin{split}
            \O\Bigg(\frac{Hd}{\kappa}\sbr{\frac{1}{(1-\gamma)^{3 / 2}} \Delta  + HK\sqrt{1-\gamma}}+ H^{3/2}d\sqrt{HK}\Bigg)&.
    \end{split}
    \end{equation*}
    Furthermore, by setting the discounted factor optimally as $\gamma = 1-\max\bbr{1/K,\sqrt{\Delta/T}}$, we have
    \begin{equation*}
        \DReg_T \leq \begin{cases}\Ot\sbr{\kappa^{-1}Hd \Delta^{1/4}T^{3/4}} & \text { when } \Delta \geq H / K, \\ \Ot\sbr{\kappa^{-1}dH^{3/2}\sqrt{T}}& \text { when } \Delta<H / K .\end{cases}
        \end{equation*}
\end{myThm}
\iffullversion \noindent Proof of Theorem~\ref{thm:MNL-regret-bound} is in Appendix~\ref{app:MNL-thm-regret-bound}.\fi 
\section{Experiments}
\label{sec:experiments}
In this section, we further empirically examine the performance of our proposed algorithms. We present two synthetic experiments on drifting LB and GLB, respectively. For each experiment, we set the dimension of the feature space to $d=2$, the number of rounds to $T=6000$, and the number of arms to $n=50$. The features of each arm are sampled from the normal distribution $\mathcal{N}(0,1)$ and subsequently rescaled to satisfy $L=1$. We initialize the time-varying parameter $\theta_t$ to $[1,0]$ and rotate it uniformly counterclockwise around the unit circle, completing one full revolution from $0$ to $2\pi$ over the course of $T$ rounds and returning to the starting point $[1,0]$.

\begin{figure*}[!t]
    \centering
    \subfloat[LB: Cumulative regret]{ \label{figure:LB-regret} 
    \includegraphics[width=0.3\textwidth, valign=t]{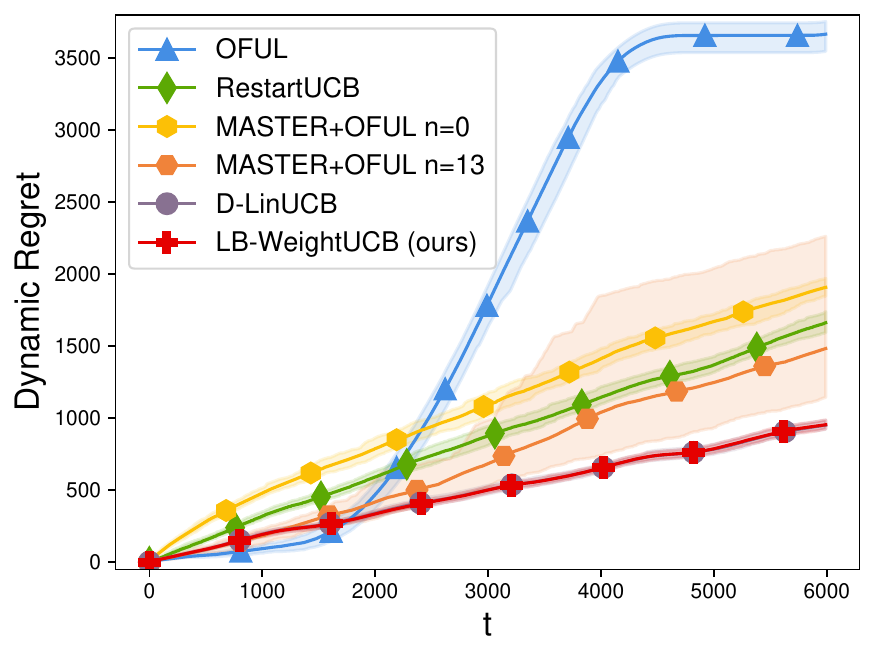}}
    \subfloat[GLB Algorithms ($S=1$)]{ \label{figure:GLB1} 
        \includegraphics[width=0.3\textwidth, valign=t]{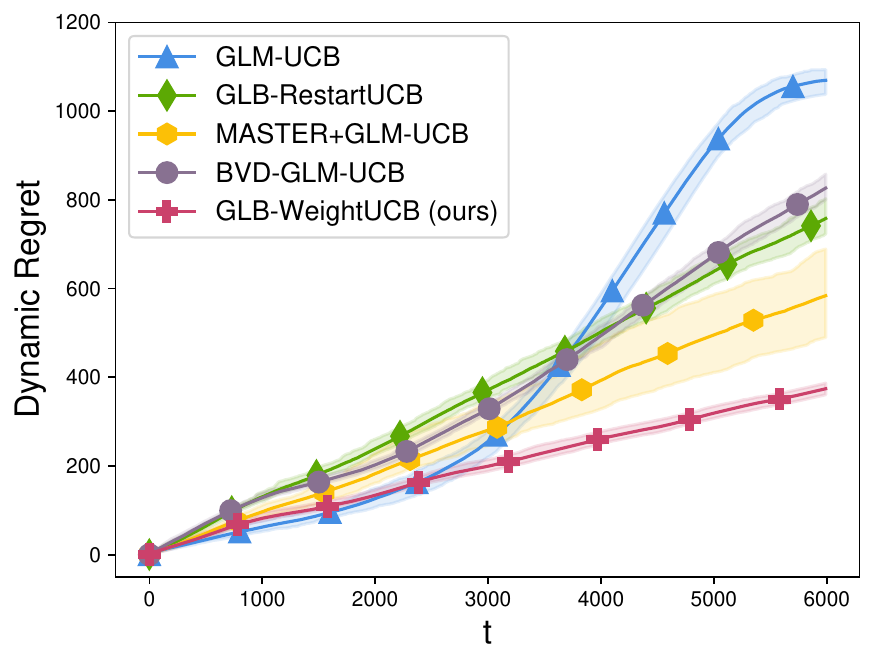}}
    \subfloat[SCB Algorithms ($S=1$)]{ \label{figure:SCB1}
        \includegraphics[width=0.3\textwidth, valign=t]{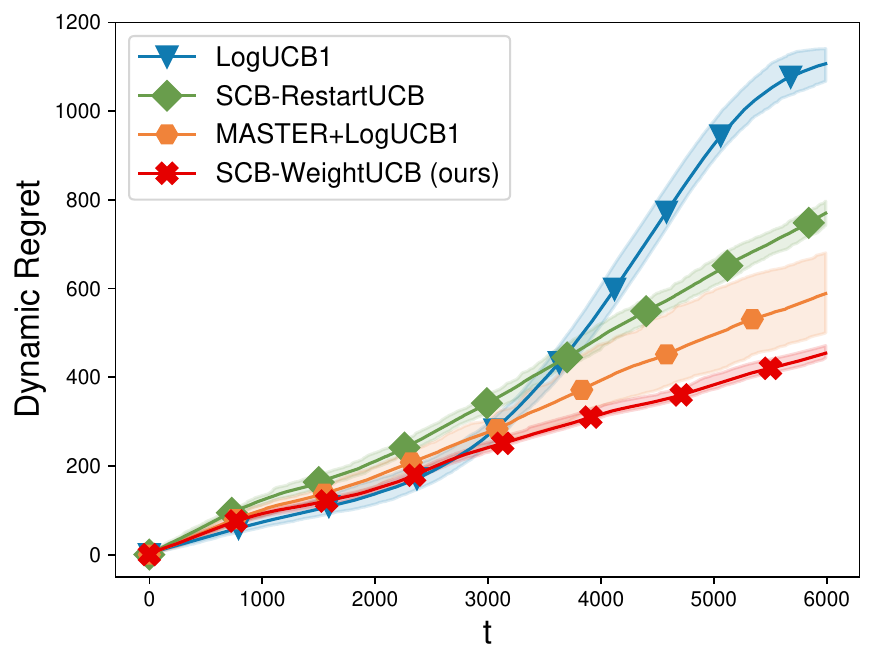}}\\
        \subfloat[LB: Average running time]{ \label{figure:LB-running-time} 
        \includegraphics[width=0.3\textwidth, valign=t]{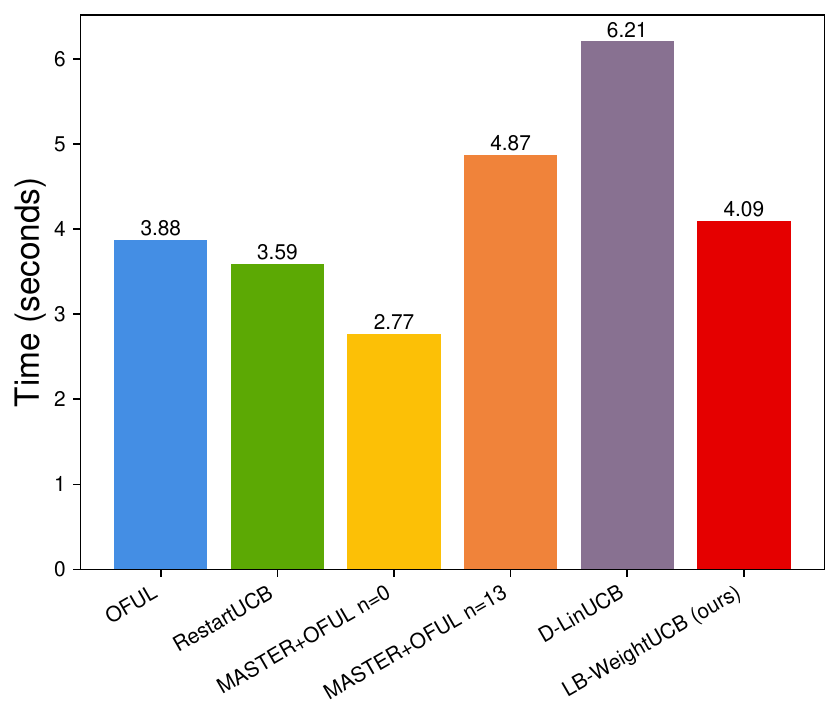}}
    \subfloat[GLB Algorithms ($S=5$)]{ \label{figure:GLB5} 
        \includegraphics[width=0.3\textwidth, valign=t]{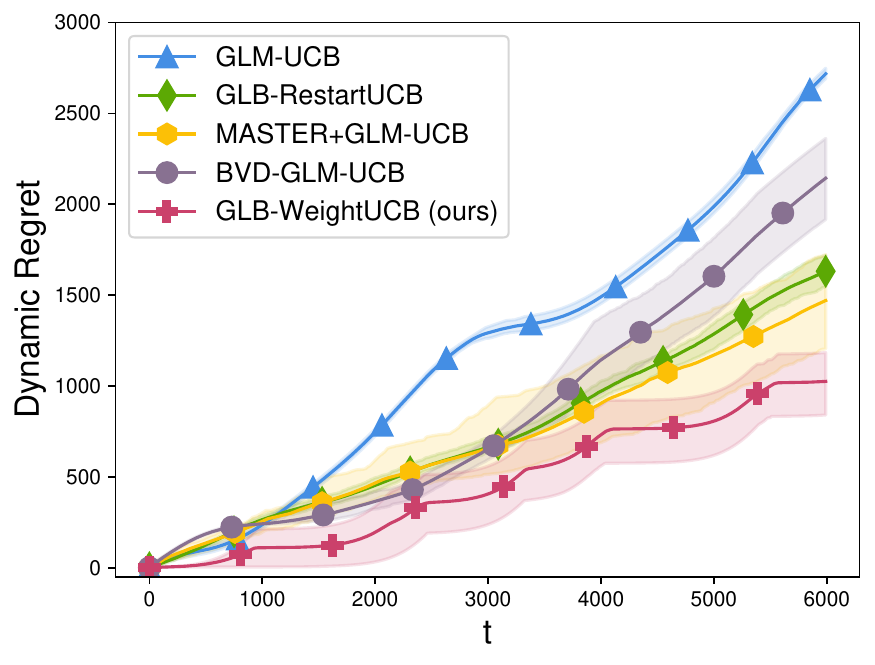}
        \vphantom{\includegraphics[width=0.3\textwidth,valign=t]{figure/lb_running_time.pdf}}}
    \subfloat[SCB Algorithms ($S=5$)]{ \label{figure:SCB5}
        \includegraphics[width=0.3\textwidth, valign=t]{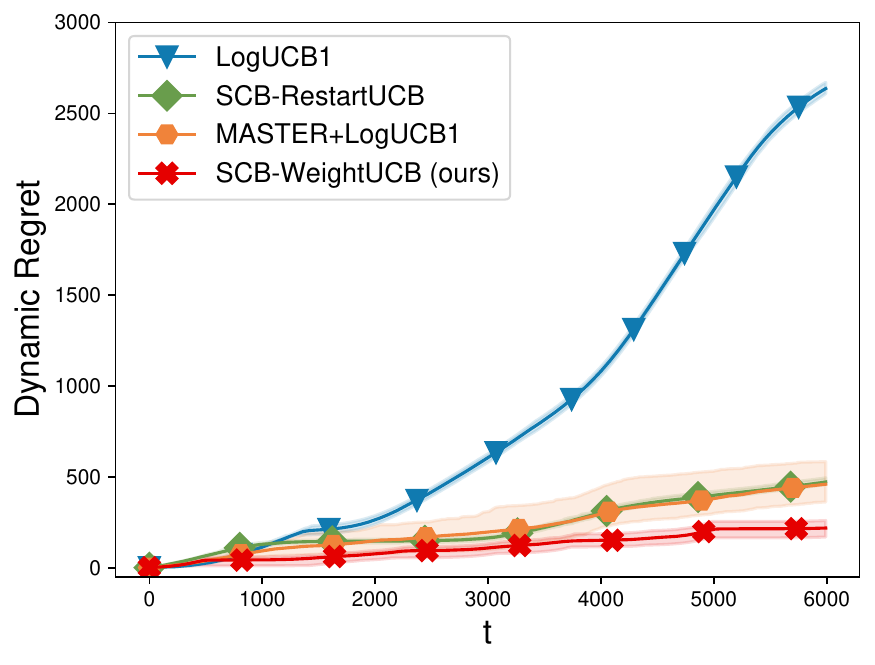}
        \vphantom{\includegraphics[width=0.3\textwidth,valign=t]{figure/lb_running_time.pdf}}} 
    \caption{Experiments of generalized linear bandits.}
    \label{figure:GLB}
\end{figure*}

\subsection{Linear Bandits}
\parag{Setting.}
We consider the linear model $r_t = X_t^\T\theta_t +\eta_t$ where the random noise $\eta_t$ is drawn from the normal distribution $\N(0,1)$ at each time $t$ independently. We compare the performance of our proposed \LBweightours algorithm to: (a) the static algorithm \LBstatic~\cite{NIPS'11:AY-linear-bandits}; (b) the restart-based algorithm \LBrestart~\cite{AISTATS'20:restart}; (c) the weight-based algorithm \LBweight~\cite{NIPS'19:weighted-LB}; and (d) the adaptive restart algorithm \LBMASTER~\cite{COLT'21:black-box}.~\blue{Among these algorithms, \LBrestart and \LBweight both require prior knowledge of $P_T$, whereas MASTER+OFUL does not. Although MASTER+OFUL operates under weaker requirements, we still include it in our comparison because it achieves theoretically optimal regret with a faster convergence rate, making it an important benchmark.} Since $P_T$ is computable, we set the discounted factor $\gamma = 1- \max\{1/T, \sqrt{P_T/(dT)}\}$ for \LBweightours and \LBweight, and set the window size $W$ and restarting period $H$ as $W = H = d^{1/4}\sqrt{T/(1+P_T)}$. For \MASTER, there is a parameter $n$ representing the initial value of a multi-scale exploration parameter (see the input of Procedure 1 in~\cite{COLT'21:black-box}), and the origin \MASTER algorithm lets it start from $0$ (i.e., $n = 0,1,...$). However, a small initial value of $n$ leads to frequent restarts and thus poor performance. To this end, we experiment with a larger initial value of $n = 13$, which yields substantially improved performance in our case. 

\parag{Results.}
The experimental results are averaged over 20 independent trials. Fig.~\ref{figure:LB-regret} shows the cumulative dynamic regret performance, where the shaded area denotes the variance of experimental results. Fig.~\ref{figure:LB-running-time} reports the average time per run, with each run containing 6000 rounds. Our \LBweightours algorithm performs as well as \LBweight but is significantly more efficient, with over 1.5 times speedup. Fig.~\ref{figure:LB-regret} also shows that when equipped with a fine-tuned $n$, \LBMASTER ($n=13$) performs better than \LBrestart, whereas a vanilla \LBMASTER ($n=0$) performs worse due to overly active restarts at the beginning. However, a larger initial value of $n$ results in greater time overhead, since at each restart, \LBMASTER needs to do Procedure~1 once, resulting in an $\mathcal{O}(n2^n)$ time complexity. More importantly, neither adaptive restart (\LBMASTER) nor periodical restart (\LBrestart) outperforms our weighted strategy in slowly-evolving environments.

\subsection{Generalized Linear Bandits}
\parag{Setting.}
We employ the logistic model in the GLB experiment, i.e., the reward satisfies $r_t \sim \text{Bernoulli}(\mu(X_t^\T\theta_t))$ with logistic function $\mu(x) = (1+e^{-x})^{-1}$. We consider two cases of $S=1$ and $S=5$, respectively. We compare the performance of our proposed \GLBweightours and \SCBweightours algorithm to: (a) \GLBstatic, static algorithm for GLB~\cite{NIPS10:GLM-infinite}; (b) \SCBstatic, static algorithm for LogB~\cite{ICML'20:logistic-bandits}; (c) \GLBweight, weight-based algorithm for GLB~\cite{arXiv'21:faury-driftingGLB}; (d) \GLBrestart, restart algorithm for GLB~\cite{AISTATS'20:restart}; (e) \SCBrestart, restart algorithm for SCB~\cite{AISTATS'20:restart}; (f) \GLBMASTER, adaptive restart algorithm for GLB~\cite{COLT'21:black-box}; and (g) \SCBMASTER, adaptive restart algorithm for LogB~\cite{COLT'21:black-box}. We set discounted factor $\gamma = 1- \max\{1/T, \sqrt{c_\mu P_T/(dT)}\}$ for \GLBweightours, $\gamma = 1- (P_T/(\sqrt{d}T))^{\sfrac{2}{5}}$ for \GLBweight and $\gamma = 1- \max\{1/T, \sqrt{P_T/(dT)}\}$ for \SCBweightours. We set restarting period $H = d^{1/4}\sqrt{T/(1+P_T)}$ for both \GLBrestart and \SCBrestart. We set regularizer $\lambda = d$ for \GLBstatic, \GLBweight, \GLBrestart and \GLBMASTER, $\lambda = d/c_\mu^2$ for \GLBweightours and $\lambda = d\log T/c_\mu$ for \SCBstatic, \SCBrestart, \SCBMASTER and \SCBweightours. Note that for LogB, $k_\mu = 1/4 < 1$, so we don't need to control the order of $k_\mu$. For two \MASTER algorithms, we set $n = 13$.

\parag{Results.}
We present the average cumulative dynamic regret results of our experiments on $20$ independent trials in Fig.~\ref{figure:GLB}. When $S$ is small ($S=1, c_\mu^{-1} \approx 5$), all of the weight-based algorithms outperform the static algorithms, and our \GLBweightours and \SCBweightours are better than \GLBweight. When $S$ is large ($S=5, c_\mu^{-1} \approx 152$), \SCBweightours significantly outperforms \GLBweightours, demonstrating the importance of considering the self-concordant property (recall that LogB is an instance of SCB). In contrast, the performance of \GLBweight drops dramatically, as it does not take the $c_\mu^{-1}$ issue into account. Similar to LB, the experimental results of GLB also demonstrate the empirical advantage of the weighted strategy over (adaptive) restart strategy in slowly-evolving environments. Specifically, we observe that \GLBweightours consistently outperforms \GLBMASTER, and \SCBweightours consistently outperforms \SCBMASTER.

\section{Conclusion}
\label{sec:conclusion}
This paper revisits the weight-based algorithms for three non-stationary parametric bandit models (LB, GLB, SCB) and two non-stationary MDP settings (Linear Mixture MDP, MNL Mixture MDP). We identify that the inadequacies of the previous work are due to the inadequate analysis of the estimation error. We thus propose a refined analysis framework that enables the usage of the same local norm for both the bias and variance parts in estimation error analysis. Our framework ensures more efficient algorithms for all three bandit models and two RL models, improves the regret bounds for GLB and SCB settings, and establishes the first dynamic regret bound for MNL Mixture MDP. 

The importance of our work lies in the fact that we have now made the weight-based algorithms for non-stationary parametric bandits and MDPs as competitive as the restart-based algorithms, in terms of both computational efficiency and regret guarantee.~\blue{Note that the current window-based, restart-based, and weight-based algorithms can only achieve a regret bound of $\widetilde{\mathcal{O}}(P_T^{1/4} T^{3/4})$, which does not match the optimal rate $\widetilde{\mathcal{O}}(P_T^{1/3} T^{2/3})$ attained by the MASTER algorithm, an adaptive restart strategy~\cite{COLT'21:black-box}. In the spirit of this best-known result, it is essential to design \mbox{\emph{adaptive}} weight-based algorithms that can achieve the optimal dynamic regret bound without requiring prior knowledge of the environment's non-stationarity, given that weighted strategies are particularly effective in gradually drifting environments, which are commonly encountered in real-world applications.} \blue{The current lower bound of $\Omega(P_T^{1/3} T^{2/3})$ is established under the fixed arm set assumption~\cite{Manage'22:window-LB}. The MASTER algorithm~\cite{COLT'21:black-box} matches this rate with the same assumption, making $\Theta(P_T^{1/3} T^{2/3})$ the minimax optimal rate for the fixed arm set case. However, the minimax rate remains open for time-varying arm sets.}

In this work, we employ $P_T = \sum_{t=2}^{T} \norm{\theta_{t-1} - \theta_{t}}_2$ as a measure to capture the gradually changing environment. However, this metric may not be precise enough in capturing only the gradual changes in the environment, as it can also include other types of variations, such as abrupt changes and restless changes~\cite{TIT'12:restless,TIT'24:restless2}. This might be able to explain why weight-based algorithms do not exhibit a significant theoretical advantage, yet perform remarkably well in experiments on gradually changing environments compared to restart-based algorithms. To overcome this limitation, future research could explore more refined characterizations of gradual changes, drawing inspiration from the ideas behind Sobolev or Holder classes~\cite{NIPS'19:TV_bound} or other information-theoretic tools~\cite{arXiv'22:VanRoy}.

{\appendices
\section{Analysis of \LBweightours}
\label{app:LB}
In this section, we provide the analysis for \LBweightours algorithm. In Appendix~\ref{sec:LB-review}, we review the \LBweight algorithm proposed by~\cite{NIPS'19:weighted-LB} and restate their estimation error analysis. In Appendix~\ref{sec:LB-estimation-error-proof}, we present our own estimation error analysis for the proposed \LBweightours algorithm, which is captured in Lemma~\ref{lemma:LB-estimation-error}. Finally, in Appendix~\ref{sec:LB-regret-proof}, we provide an analysis of dynamic regret, as stated in Theorem~\ref{thm:LB-regret}.

\subsection{Review Estimation Error Analysis of \LBweight Algorithm}
\label{sec:LB-review}
In this part, we review the previous estimation error analysis of the \LBweight algorithm~\cite{NIPS'19:weighted-LB}, which has the same estimator as ours~\eqref{eq:LB-estimator}. The first step is to divide the estimation error into the bias and variance parts, where the bias part represents the error caused by parameter drift and the variance part represents the error caused by stochastic noise. Based on the reward model assumption and estimator (same as~\eqref{eq:LB-model} and~\eqref{eq:LB-estimator}), the estimation error of \LBweight algorithm can be decomposed as
  \begin{align}
    \thetah_t - \theta_t = {}& V_{t-1}^{\prime-1}\sbr{\sum_{s=1}^{t-1}\gamma^{t-s-1}r_sX_s}- \theta_t\nonumber\\
    = {}& V_{t-1}^{\prime-1}\sbr{\sum_{s=1}^{t-1}\gamma^{t-s-1}\sbr{ X_s^\T \theta_s + \eta_s}X_s}\nonumber\\
    {}&- V_{t-1}^{\prime-1}\sbr{\lambda I_d + \sum_{s=1}^{t-1} \gamma^{t-s-1} X_s X_s^\T}\theta_t\nonumber\\
    = {}& V_{t-1}^{\prime-1}\sbr{\sum_{s=1}^{t-1}\gamma^{t-s-1}  X_s X_s^\T \theta_s + \sum_{s=1}^{t-1}\gamma^{t-s-1}\eta_sX_s }\nonumber\\{}&- V_{t-1}^{\prime-1}\sbr{\lambda I_d + \sum_{s=1}^{t-1} \gamma^{t-s-1} X_s X_s^\T}\theta_t\nonumber\\
    = {}& \underbrace{V_{t-1}^{\prime-1}\sbr{\sum_{s=1}^{t-1}\gamma^{t-s-1}  X_s X_s^\T \sbr{\theta_s -\theta_t}}}_{\bias}\nonumber\\ \label{eq:estimation_error_decomposition}{}&+ \underbrace{V_{t-1}^{\prime-1}\sbr{ \sum_{s=1}^{t-1}\gamma^{t-s-1}\eta_sX_s -\lambda\theta_t}}_{\variance},
  \end{align}
  where $V^{\prime}_{t} = \lambda I_d + \sum_{s=1}^{t-1} \gamma^{t-s} X_s X_s^\T$. Afterward,~\cite{NIPS'19:weighted-LB} uses different local norms (we will explain the reason for using different local norms later) for the bias and variance parts as 
  \begin{equation}
    \begin{split}
      \label{Russac:decompose}
        |\x^\T (\thetah_t - \theta_t)| \leq \norm{\x}_{2}A_t'+\norm{\x}_{V_{t-1}^{\prime-1}\Vt_{t-1}V_{t-1}^{\prime-1}}B_t',\\
    \end{split}
  \end{equation}  
  where $\Vt_{t} = \lambda I_d + \sum_{s=1}^{t} \gamma^{2(t-s)} X_s X_s^\T$ and 
  \begin{equation}\nonumber
      \begin{split}
        A_t' &= \norm{V_{t-1}^{\prime-1}\sum_{s=1}^{t-1}\gamma^{t-s-1}  X_s X_s^\T \sbr{\theta_s -\theta_t}}_{2}\\ B_t' &= \norm{\sum_{s=1}^{t-1}\gamma^{t-s-1}\eta_sX_s -\lambda\theta_t}_{\Vt_{t-1}^{-1}}.
      \end{split}
  \end{equation}
For the bias part,~\cite{NIPS'19:weighted-LB} divide it into two parts on the timeline by introducing a virtual window size $D$,
\begin{align*}
  A_t' \leq {}&\underbrace{\norm{\sum_{s=t-D}^{t-1}V_{t-1}^{\prime-1}\gamma^{t-s-1}  X_s X_s^\T \sbr{\theta_s -\theta_t}}_{2}}_{\mathtt{virtual~window}} \\&+ \underbrace{\norm{\sum_{s=1}^{t-D-1}V_{t-1}^{\prime-1}\gamma^{t-s-1}  X_s X_s^\T \sbr{\theta_s -\theta_t}}_{2}}_{\mathtt{small~term}},
\end{align*}
The first term can be considered as a virtual window containing the most recent data obtained after time $t-D$, and can be directly analyzed by the analysis of \LBwindow~\cite{AISTATS'19:window-LB} since it corresponds to the bias part of the estimation error of the window strategy, and this is why they use $l_2$-norm for the bias part. The second term reflects the influence formed by the outdated data obtained before time $t-D$. Since $\gamma^{t-s-1}$ will be very small when $s \leq t-D-1$, this small term is dominated by the first virtual window term, which means the bias part is actually controlled by the virtual window size $D$. 

For the variance part,~{\cite{NIPS'19:weighted-LB}} extend the previous self-normalized concentration~{\cite[Theorem 1]{NIPS'11:AY-linear-bandits}} to the weighted version. This concentration requires the use of $\Vt_t$ as the local norm. To this end, ~\cite{NIPS'19:weighted-LB} split the variance part as
\begin{equation}\nonumber
  \begin{split}
    {}&\abs{\x^\T V_{t-1}^{\prime-1}\sbr{ \sum_{s=1}^{t-1}\gamma^{t-s-1}\eta_sX_s -\lambda\theta_t}}\leq \norm{\x}_{V_{t-1}^{\prime-1}\Vt_{t-1}V_{t-1}^{\prime-1}}C_t',
  \end{split}
\end{equation} 
where
\begin{equation}\nonumber
  \begin{split}
    C_t' = {}&\norm{V_{t-1}^{\prime-1}\sbr{ \sum_{s=1}^{t-1}\gamma^{t-s-1}\eta_sX_s -\lambda\theta_t}}_{V_{t-1}^\prime\Vt_{t-1}^{-1}V_{t-1}^\prime} \\
    ={}& \norm{\sum_{s=1}^{t-1}\gamma^{t-s-1}\eta_sX_s -\lambda\theta_t}_{\Vt_{t-1}^{-1}} \\
    \leq {}&\norm{\sum_{s=1}^{t-1}\gamma^{t-s-1}\eta_sX_s}_{\Vt_{t-1}^{-1}} +\sqrt{\lambda} S.
  \end{split}
\end{equation} 
Then term $\|\sum_{s=1}^{t-1}\gamma^{t-s-1}\eta_sX_s\|_{\Vt_{t-1}^{-1}}$ can be bounded by the weighted version self-normalized concentration.
Finally, based on this analysis, \LBweight needs to use the following action selection criterion, which only depends on the variance part since the bias part doesn't contain $\x$, 
\begin{equation}\nonumber
  \begin{split}
    X_t = \argmax_{\x \in \X} \bbr{ \langle \x,\thetah_t \rangle + \beta_{t-1}\norm{\x}_{V_{t-1}^{\prime-1}\Vt_{t-1}V_{t-1}^{\prime-1}}},
  \end{split}
\end{equation}
where $\beta_{t-1}$ is the upper bound of $B_t'$ which is the same as~\eqref{eq:LB-confidence-radius}. From this selection criterion, it can be seen that \LBweight needs to maintain two covariance matrices, namely, $V_t^\prime$ and $\Vt_t$ at round $t$ during the algorithm running.

In the next section, we present our proof for the estimation error upper bound. The difference between our analysis and \LBweight's analysis mainly starts at step~\eqref{Russac:decompose}, which is the key step of analysis, and our new analysis framework allows us to employ \emph{same} local norm for both bias and variance parts.

\subsection{Proof of Lemma~\ref{lemma:LB-estimation-error}}
\label{sec:LB-estimation-error-proof}

\begin{proof}
  Using the same derivation in~\eqref{eq:estimation_error_decomposition}, the estimation error of \LBweightours algorithm can also be decomposed as
  \begin{align*}
    \thetah_t - \theta_t = {}&\underbrace{V_{t-1}^{-1}\sbr{\sum_{s=1}^{t-1}w_{t-1,s}  X_s X_s^\T \sbr{\theta_s -\theta_t}}}_{\bias} \\&+ \underbrace{V_{t-1}^{-1}\sbr{ \sum_{s=1}^{t-1}w_{t-1,s}\eta_sX_s -\lambda\theta_t}}_{\variance}.
  \end{align*}
  Therefore, by the Cauchy-Schwarz inequality, we know that for any $\x\in\X$,
  \begin{equation}
  \begin{split}
      \label{eq:LB-bound-cauchy}
      \abs{\x^\T\sbr{\thetah_t-\theta_t}} \leq \norm{\x}_{V_{t-1}^{-1}}(A_t+B_t),\\
  \end{split}
  \end{equation}
  where
  \begin{equation}\nonumber
      \begin{split}
      A_t &= \norm{\sum_{s=1}^{t-1}w_{t-1,s}  X_s X_s^\T \sbr{\theta_s -\theta_t}}_{V_{t-1}^{-1}}\\ B_t &= \norm{ \sum_{s=1}^{t-1}w_{t-1,s}\eta_sX_s -\lambda\theta_t}_{V_{t-1}^{-1}}.
      \end{split}
  \end{equation}
  The above two terms can be bounded separately, as summarized in the following two lemmas,
  \begin{myLemma}
    \label{lemma:LB-A_t-bound}
    For any $t \in [T]$, we have 
    \begin{equation}\nonumber
      \begin{split}
        \label{eq:LB-A_t-bound}
        {}&\norm{\sum_{s=1}^{t-1}w_{t-1,s}  X_s X_s^\T \sbr{\theta_s -\theta_t}}_{V_{t-1}^{-1}}\\ \leq{}& L\sqrt{d} \sum_{p=1}^{t-1}  \sqrt{\sum_{s=1}^{p}w_{t-1,s}}\norm{\theta_p -\theta_{p+1}}_2.
        \end{split}
    \end{equation}
  \end{myLemma}

  \begin{myLemma}
    \label{lemma:LB-B_t-bound}
    For any $\delta \in (0,1)$, with probability at least $1-\delta$, the following holds for all $t \in [T]$,
    \begin{equation}\nonumber
      \begin{split}
        \label{eq:LB-B_t-bound}
        {}&\norm{ \sum_{s=1}^{t-1}w_{t-1,s}\eta_sX_s -\lambda\theta_t}_{V_{t-1}^{-1}}\\ \leq{}& \sqrt{\lambda}S+R\sqrt{2\log\frac{1}{\delta}+d\log\sbr{1+\frac{L^2\sum_{s=1}^{t-1}w_{t-1,s}}{\lambda d}}},
        \end{split}
    \end{equation}
  \end{myLemma}
  \noindent Based on the inequality~\eqref{eq:LB-bound-cauchy}, Lemma~\ref{lemma:LB-A_t-bound}, Lemma~\ref{lemma:LB-B_t-bound}, the boundedness assumption of the feasible set and the fact that for any $\x$, $\norm{\x}_{V_{t-1}^{-1}} \leq \norm{\x}_2/\sqrt{\lambda} $ since $V_{t-1} \succeq \lambda I_d$, for any $\x \in \X$, $\gamma\in (0,1)$ and $\delta \in (0,1)$, with probability at least $1-\delta$, the following holds for all $t \in [T]$,
  \begin{equation}\nonumber
    \begin{split}
      {}&|\x^\T(\th_t-\theta_t)|\\\leq{}& L^2\sqrt{\frac{d}{\lambda}} \sum_{p=1}^{t-1}  \sqrt{\sum_{s=1}^{p}w_{t-1,s}}\norm{\theta_p -\theta_{p+1}}_2 + \beta_{t-1}\|\x\|_{V_{t-1}^{-1}},
    \end{split}
  \end{equation}
  where $\beta_t \define \sqrt{\lambda}S+R\sqrt{2\log \frac{1}{\delta}+d\log\left(1+\frac{L^2\sum_{s=1}^{t}w_{t-1,s}}{d\lambda}\right)}$ is the confidence radius used in \LBweightours. Hence, we complete the proof.
\end{proof}

\begin{proof}[{Proof of Lemma~\ref{lemma:LB-A_t-bound}}]
  \label{sec:LB-A_t-bound-proof}
  The first step is to extract the variations of the parameter $\theta_t$ as follows, 
  \begin{align*}
    {}&\norm{\sum_{s=1}^{t-1}w_{t-1,s}  X_s X_s^\T \sbr{\theta_s -\theta_t}}_{V_{t-1}^{-1}}\\ = {} & \norm{\sum_{s=1}^{t-1}w_{t-1,s}  X_s X_s^\T \sum_{p=s}^{t-1}\sbr{\theta_p -\theta_{p+1}}}_{V_{t-1}^{-1}} \\
      = {} & \norm{\sum_{p=1}^{t-1}\sum_{s=1}^{p}w_{t-1,s}  X_s X_s^\T \sbr{\theta_p -\theta_{p+1}}}_{V_{t-1}^{-1}} \\
      \leq {} & \sum_{p=1}^{t-1} \norm{\sum_{s=1}^{p}w_{t-1,s}  X_s \|X_s\|_2 \|\theta_p -\theta_{p+1}\|_2}_{V_{t-1}^{-1}} \\
      \leq {} & L \sum_{p=1}^{t-1} \sum_{s=1}^{p}w_{t-1,s}\norm{X_s}_{V_{t-1}^{-1}}\|\theta_p -\theta_{p+1}\|_2,
  \end{align*}
  and term $\sum_{s=1}^{p}w_{t-1,s}\norm{X_s}_{V_{t-1}^{-1}}$ can further derive to an expression about the discounted factor $\gamma$ as follows,
  \begin{align}
    \sum_{s=1}^{p}w_{t-1,s} \norm{X_s}_{V_{t-1}^{-1}}  \leq{}&  \sqrt{\sum_{s=1}^{p}w_{t-1,s}}\sqrt{\sum_{s=1}^{p}w_{t-1,s}\norm{X_s}^2_{V_{t-1}^{-1}}}\nonumber  \\
        \leq{}& \sqrt{d}\sqrt{\sum_{s=1}^{p}w_{t-1,s}}.\label{eq:LB-At-gammapart}
  \end{align} 
  In above, the second last step holds by the Cauchy-Schwarz inequality. Besides, the last step follows the Lemma~\ref{lemma:d} by letting $A_s = \sqrt{w_{t-1,s}}X_s$ and $U_{t-1} = V_{t-1}$. Hence, we complete the proof.
\end{proof}

\begin{proof}[{Proof of Lemma~\ref{lemma:LB-B_t-bound}}]
  \label{sec:LB-B_t-bound-proof}
  \begin{equation}\nonumber
    \begin{split}
      {}&\norm{\sum_{s=1}^{t-1}w_{t-1,s}\eta_sX_s -\lambda\theta_t}_{V_{t-1}^{-1}}\\\leq{}&\norm{ \sum_{s=1}^{t-1}w_{t-1,s}\eta_sX_s}_{V_{t-1}^{-1}}+\sqrt{\lambda}S.
    \end{split}
\end{equation}
We define $\etat_s \define \sqrt{w_{t-1,s}}\eta_s$ and $\Xt_s \define \sqrt{w_{t-1,s}} X_s$, and notice that $\forall t\in[T] , s\in [t-1], \abs{w_{t-1,s}}\leq 1$, then $\etat_s$ is still $R$-sub-Gaussian, then by Theorem~\ref{thm:snc-AY}, we have 
    \begin{equation}\nonumber
        \begin{split}
          {}&\norm{ \sum_{s=1}^{t-1}w_{t-1,s}\eta_sX_s }_{V_{t-1}^{-1}} \\\leq{}& \sqrt{2R^2\log\left(\frac{\det(V_{t-1})^{\frac{1}{2}}\det(V_0)^{-\frac{1}{2}}}{\delta}\right)}.
        \end{split}
    \end{equation}
    Then, based on Lemma~\ref{lemma:det-inequality} and $\det(V_{0}) = \lambda^d$, we have 
    \begin{equation}\nonumber
        \begin{split}
            {}&\norm{ \sum_{s=1}^{t-1}w_{t-1,s}\eta_sX_s }_{V_{t-1}^{-1}} \\\leq{}& R\sqrt{2\log \frac{1}{\delta}+d\log\left(1+\frac{L^2\sum_{s=1}^{t-1}w_{t-1,s}}{d\lambda}\right)}.
        \end{split}
    \end{equation}
  which completes the proof.
\end{proof}

\subsection{Proof of~\pref{thm:LB-regret}}
\label{sec:LB-regret-proof}

\begin{proof}
  Let $X_t^* \define\argmax_{\x \in \X} \x^\T \theta_t$. Due to Lemma~\ref{lemma:LB-estimation-error} and the fact that $X_t^*,X_t\in \X$, each of the following holds with probability at least $1-\delta$,
  \begin{equation}\nonumber
    \begin{split}
        \forall t \in [T], X_t^{*\T}\theta_t \leq{}&  X_t^{*\T}\thetah_t + \beta_{t-1}\|X_t^*\|_{V_{t-1}^{-1}}\\{}&+L^2\sqrt{\frac{d}{\lambda}} \sum_{p=1}^{t-1}\sqrt{\sum_{s=1}^{p}w_{t-1,s}}\norm{\theta_p -\theta_{p+1}}_2,\\ 
        \forall t \in [T], X_t^{\T}\theta_t \geq{}&  X_t^{\T}\thetah_t - \beta_{t-1}\|X_t\|_{V_{t-1}^{-1}}\\{}&-L^2\sqrt{\frac{d}{\lambda}} \sum_{p=1}^{t-1} \sqrt{\sum_{s=1}^{p}w_{t-1,s}}\norm{\theta_p -\theta_{p+1}}_2 .
    \end{split}
    \end{equation}
    By the union bound, the following holds with probability at least $1-2\delta$, $\forall t \in [T]$,
  \begin{equation}\nonumber
  \begin{split}
     {}&X_t^{*\T}\theta_t - X_t^{\T}\theta_t \\\leq{}&  X_t^{*\T}\thetah_t -X_t^{\T}\thetah_t +2L^2\sqrt{\frac{d}{\lambda}} \sum_{p=1}^{t-1}\sqrt{\sum_{s=1}^{p}w_{t-1,s}}\norm{\theta_p -\theta_{p+1}}_2 \\
     &+ \beta_{t-1}(\|X_t^*\|_{V_{t-1}^{-1}}+\|X_t\|_{V_{t-1}^{-1}})\\
      \leq{}& 2L^2\sqrt{\frac{d}{\lambda}} \sum_{p=1}^{t-1} \sqrt{\sum_{s=1}^{p}w_{t-1,s}}\norm{\theta_p -\theta_{p+1}}_2 + 2\beta_{t-1}\|X_t\|_{V_{t-1}^{-1}},
  \end{split}
  \end{equation}
  where the last step comes from the arm selection criterion~\eqref{eq:LB-select-criteria} such that $$X_t^{*\T}\thetah_t +\beta_{t-1}\|X_t^*\|_{V_{t-1}^{-1}} \leq X_t^{\T}\thetah_t+\beta_{t-1}\|X_t\|_{V_{t-1}^{-1}}.$$
  Hence, the following dynamic regret bound holds with probability at least $1-2\delta$ and can be divided into two parts,
  \begin{equation}\nonumber
  \begin{split}
      \DReg_T = {}&  \sum_{t=1}^T \sbr{X_t^{*\T} \theta_t - X_t^\T \theta_t} \\\leq{}&  \underbrace{2L^2\sqrt{\frac{d}{\lambda}}\sum_{t=1}^T\sum_{p=1}^{t-1} \sqrt{\sum_{s=1}^{p}w_{t-1,s}}\norm{\theta_p -\theta_{p+1}}_2}_{\bias}\\
      &+\underbrace{2\beta_{T}\sum_{t=1}^T\|X_t\|_{V_{t-1}^{-1}}}_{\variance},
  \end{split}
  \end{equation}
  where $\beta_T = \sqrt{\lambda}S+R\sqrt{2\log\frac{1}{\delta}+d\log\sbr{1+\frac{L^2 (1-\gamma^{2T})}{\lambda d(1-\gamma^2)}}}$ is the confidence radius.

  Now we derive the upper bound for the bias and variance parts separately.

  \parag{Bias Part.} Notice that $w_{t-1,s} = \gamma^{t-s-1}$ with $\gamma \in (1/T,1)$. For the bias part, we need to extract the path length $P_T$ and show the control of the discounted factor $\gamma$ on $P_T$.
  \begin{align*}
    {}&2L^2\sqrt{\frac{d}{\lambda}}\sum_{t=1}^T\sum_{p=1}^{t-1} \sqrt{\sum_{s=1}^{p}w_{t-1,s}}\norm{\theta_p -\theta_{p+1}}_2 \\= {}&2L^2\sqrt{\frac{d}{\lambda}}\sum_{p=1}^{T-1}  \sum_{t=p+1}^{T} \sqrt{\sum_{s=1}^{p}w_{t-1,s}}\norm{\theta_p -\theta_{p+1}}_2\\
    = {}&2L^2\sqrt{\frac{d}{\lambda}}\sum_{p=1}^{T-1}  \sum_{t=p+1}^{T}\gamma^{\frac{t-1}{2}} \sqrt{\sum_{s=1}^{p}\gamma^{-s}}\norm{\theta_p -\theta_{p+1}}_2\\
    = {}&2L^2\sqrt{\frac{d}{\lambda}}\sum_{p=1}^{T-1} \frac{\gamma^{\frac{p}{2}}-\gamma^{\frac{T}{2}}}{1-\gamma^{\frac{1}{2}}}\sqrt{\frac{\gamma^{-p}-1}{1-\gamma}}\norm{\theta_p -\theta_{p+1}}_2\\
    \leq {}&2L^2\sqrt{\frac{d}{\lambda}}\sum_{p=1}^{T-1} \frac{\gamma^{\frac{p}{2}}\gamma^{-\frac{p}{2}}}{(1-\gamma^{\frac{1}{2}})\frac{1+\gamma^{\frac{1}{2}}}{2}\sqrt{1-\gamma}}\norm{\theta_p -\theta_{p+1}}_2\\
    \leq {}&4L^2\sqrt{\frac{d}{\lambda}}\frac{1}{(1-\gamma)^{3/2}}P_T.
\end{align*}
So for the bias part, we have
\begin{equation}
  \begin{split}
  \label{eq:LB-regret-bias-bound}
  {}&2L^2\sqrt{\frac{d}{\lambda}}\sum_{t=1}^T\sum_{p=1}^{t-1} \sqrt{\sum_{s=1}^{p}w_{t-1,s}}\norm{\theta_p -\theta_{p+1}}_2 \\\leq{}& 4L^2\sqrt{\frac{d}{\lambda}}\frac{1}{(1-\gamma)^{3/2}}P_T.\\
  \end{split}
\end{equation}

\parag{Variance Part.}
First, use the Cauchy-Schwarz inequality, we have
$
    2\beta_{T}\sum_{t=1}^T\|X_t\|_{V_{t-1}^{-1}}
    \leq 2\beta_{T}\sqrt{T\sum_{t=1}^T\|X_t\|_{V_{t-1}^{-1}}^2}.
$
Then by Lemma~\ref{lemma:potential-lemma} (weighted potential lemma) with $w_{t,t} = \gamma^{t-t} = 1, c = w_{t,s}/w_{t,s-1} = \gamma$, we have the following upper bound:
\begin{equation}
  \label{eq:LB-regret-variance-bound}
\begin{split}
  {}&2\beta_{T}\sum_{t=1}^T\|X_t\|_{V_{t-1}^{-1}}\leq 2\beta_{T}\sqrt{2\max\{1,\frac{L^2}{\lambda}\}dT}\\{}&\hspace{5em}\cdot\sqrt{T\log\frac{1}{\gamma}+\log\sbr{1+ \frac{L^2\sum_{s=1}^T w_{T,s}}{d\lambda}}}.
\end{split}
\end{equation}
Notice that the main differences between weighted LB and standard LB in analysis are the need for \pathlength analysis and the use of the weighted potential lemma. Further we have $\sum_{s=1}^t w_{t,s} = \frac{1-\gamma^{t}}{1-\gamma}\leq \frac{1}{1-\gamma}$. Combining the upper bounds of the bias and variance parts and with confidence level $\delta = 1/(2T)$, by union bound, we have the following dynamic regret bound with probability at least $1-1/T$,
\begin{equation}\nonumber
\begin{split}
    \DReg_T\leq{}& 4L^2\sqrt{\frac{d}{\lambda}}\frac{1}{(1-\gamma)^{\sfrac{3}{2}}}P_T+ 2\beta_{T}\sqrt{2\max\bbr{1,\frac{L^2}{\lambda}}dT}\\{}&\cdot\sqrt{T\log\frac{1}{\gamma}+\log\sbr{1+\frac{L^2}{\lambda d(1-\gamma)}}}, 
\end{split}
\end{equation}
$\beta_T = \sqrt{\lambda}S+R\sqrt{2\log T+2\log 2+d\log\sbr{1+\frac{L^2 (1-\gamma^{T})}{\lambda d(1-\gamma)}}}$. Since the regret bound contains a term $T \sqrt{\log (1/\gamma)}$, we cannot let $\gamma$ close to $0$, so we set $\gamma \geq 1/T$ and have $\log(1/\gamma) \leq C (1-\gamma)$,
where $C = \log T/(1-1/T)$. Then, ignoring logarithmic factors in time horizon $T$, and let $\lambda = d$, we finally obtain
\begin{equation}\label{eq:dregret-lb}
  \begin{split}
      \DReg_T\leq{}& \Ot\sbr{\frac{1}{(1-\gamma)^{\sfrac{3}{2}}}P_T + d(1-\gamma)^{\sfrac{1}{2}}T}.
  \end{split}
  \end{equation}
  When $P_T< d/T$ (which corresponds to a small amount of non-stationarity), we simply set $\gamma = 1-1/T$ and achieve an $\Ot(d\sqrt{T})$ regret bound.  Besides, when coming to the non-degenerated case ($P_T\geq d/T$), We set the discounted factor optimally as $1-\gamma = \sqrt{P_T/(dT)}$ and attain an $\Ot(d^{\sfrac{3}{4}}P_T^{\sfrac{1}{4}}T^{\sfrac{3}{4}})$ regret bound, which completes the proof.
\end{proof}

\begin{myLemma}[Weighted Version Potential Lemma]
  \label{lemma:potential-lemma} 
  Suppose $V_{t} = \sum_{s=1}^{t}w_{t,s}X_s X_s^\T + \lambda I_d, V_{0} = \lambda I_d$, the weight satisfies that, $\forall t\in[T], s\in[t-1], w_{t,s}/w_{t-1,s} = c\leq 1$, $\forall t, s\in[T], w_{t,s} \in(0,1), w_{t,t} = 1$ and $\norm{X_t}_2 \leq L$ for all $t \geq 1$, then the following inequality holds,
  \begin{equation*}
  \begin{split}
      {}& \sum_{t=1}^T\|w_{t,t}X_t\|_{V_{t-1}^{-1}}^2 \\
      {}& \leq 2\max\bbr{1,\frac{L^2}{\lambda}}d \Bigg(T\log\frac{1}{c}+\log\Big(1+ \frac{L^2\sum_{s=1}^T w_{T,s}}{d\lambda}\Big)\Bigg).
  \end{split}
  \end{equation*}
\end{myLemma}
\begin{proof}[Proof of Lemma~\ref{lemma:potential-lemma}]
  \begin{align*}
          V_{t} &= \sum_{s=1}^{t}w_{t,s} X_s X_s^\T + \lambda I_d\\
          &= \sum_{s=1}^{t-1}w_{t,s} X_s X_s^\T + w_{t,t} X_t X_t^\T + \lambda I_d\\
          &= c\sum_{s=1}^{t-1}w_{t-1,s} X_s X_s^\T +  X_t X_t^\T + \lambda I_d \tag*{($w_{t,t} = 1$)}\\
          &\succeq c\sbr{\sum_{s=1}^{t-1}w_{t-1,s}X_s X_s^\T +  X_t X_t^\T + \lambda I_d},\quad \tag*{($c<1$)}\\
          &= c\sbr{V_{t-1} + X_t X_t^\T}\\
          &= c V_{t-1}^{1/2}\sbr{I_d +V_{t-1}^{-1/2}X_t X_t^\T V_{t-1}^{-1/2}}V_{t-1}^{1/2},
      \end{align*}
  Taking the determinant on both sides and we obtain,
  \begin{equation*}
    \begin{split}
      \det(V_t) &\geq \det(cV_{t-1})\det\sbr{I_d +V_{t-1}^{-1/2}X_t X_t^\T V_{t-1}^{-1/2}}\\
      \det(V_t) &\geq c^d\det(V_{t-1})\sbr{1+\norm{X_t}_{V_{t-1}^{-1}}^2}\\
      \log\det(V_t) &\geq d\log c + \log\det(V_{t-1}) + \log\sbr{1+\norm{X_t}_{V_{t-1}^{-1}}^2}\\
      d\log \frac{1}{c}&+\log\frac{\det(V_t)}{det(V_{t-1})} \geq \log\sbr{1+\norm{X_t}_{V_{t-1}^{-1}}^2}
    \end{split}
    \end{equation*}
      Then summing from $1$ to $T$, and telescoping we have,
      \begin{equation*}
          \begin{split}
              {}&dT\log \frac{1}{c}+\log\sbr{\frac{\det(V_T)}{\det(V_0)}} \\ \geq{}& \sum_{t=1}^T\log\sbr{1+\norm{X_t}_{V_{t-1}^{-1}}^2}\\
              \geq{}& \sum_{t=1}^T\log\sbr{1+\frac{1}{\max\bbr{1,L^2/\lambda}}\norm{X_t}_{V_{t-1}^{-1}}^2}\\
              \geq{}& \frac{1}{2\max\bbr{1,L^2/\lambda}}\sum_{t=1}^T\norm{X_t}_{V_{t-1}^{-1}}^2
          \end{split}
          \end{equation*}
          So we have 
          \begin{equation*}
              \begin{split}
                {}&\sum_{t=1}^T\norm{X_t}_{V_{t-1}^{-1}}^2\\{}&\leq 2\max\bbr{1,\frac{L^2}{\lambda}}\cdot\sbr{dT\log\frac{1}{c}+\log\sbr{\frac{\det(V_T)}{\det(V_0)}}}.
              \end{split}
              \end{equation*}          
  Finally, by using Lemma~\ref{lemma:det-inequality} and the fact $\det(V_0) = \lambda^d$, we complete the proof. 
\end{proof}

\iffullversion
\section{Analysis of \GLBweightours}
\label{sec:GLB-regret}
In this section, we provide analysis for \GLBweightours algorithm. In Appendix~\ref{sec:GLB-previous-work}, we review the projection issue of GLB and restate the \GLBweight algorithm of~\cite{arXiv'21:faury-driftingGLB}. In Appendix~\ref{sec:GLB-estimation-error-proof}, we present the proof of the estimation error upper bound of our \GLBweightours algorithm (namely, Lemma~\ref{lemma:GLB-estimation-error}). Finally, in Appendix~\ref{sec:GLB-regret-proof}, we provide the proof of dynamic regret upper bound as stated in Theorem~\ref{thm:GLB-regret}.

\subsection{Review Projection Step of \GLBweight Algorithm}
\label{sec:GLB-previous-work}
As mentioned in Section~\ref{sec:GLB-algorithm}, the main difficulty of GLB is that the result of MLE or QMLE estimator $\thetah_t$ may not belong to the feasible set $\Theta$ and $c_\mu$ is defined over the parameter $\theta \in \Theta$. Under stationary environments,~\cite{NIPS10:GLM-infinite} overcame this difficulty by introducing a projection step as
\begin{equation}
  \label{eq:simple-projection}
    \thetat_t = \argmin_{\theta \in \Theta}\|g_t(\thetah_t) - g_t(\theta)\|_{V_{t-1}^{-1}},
\end{equation}
where $V_{t} = \lambda I_d + \sum_{s=1}^{t}X_s X_s^\T$ and $g_t(\theta) = \lambda c_\mu \theta + \sum_{s=1}^{t-1}\mu(X_s^\T \theta)X_s$ are the static version (by setting $\gamma = 1$). Based on the QMLE, we know that 
  \begin{equation}
      g_t(\thetah_t) = \lambda c_\mu \thetah_t + \sum_{s=1}^{t-1}\mu(X_s^\T \thetah_t)X_s = \sum_{s=1}^{t-1}r_sX_s,
  \end{equation}
  and then by the mean value theorem, we know that 
  \begin{equation}
      g_t(\theta_1) - g_t(\theta_2)= G_t(\theta_1, \theta_2)(\theta_1 - \theta_2),
  \end{equation}
  where $G_t(\theta_1, \theta_2) \triangleq \int_{0}^1 \nabla g_t(s\theta_2+(1-s)\theta_1)\diff{s}\in \R^{d\times d}$. Notice that for any $\theta \in \Theta$, the gradient of $g_t$ satisfies 
  \begin{equation}\nonumber
      \nabla g_t(\theta) = \lambda c_\mu I_d + \sum_{s=1}^{t-1}\dmu(X_s^\T \theta)X_sX_s^\T \succeq c_\mu V_{t-1},
  \end{equation}
  which clearly implies $\forall \theta_1, \theta_2 \in \Theta, G_t(\theta_1,\theta_2)\succeq c_\mu V_{t-1}$. By this projection step,~\cite{NIPS10:GLM-infinite} can analyze the estimation error like,
\begin{align*} {}&|\mu(\x^\T\thetat_t) - \mu(\x^\T\theta_t)|  \leq k_\mu |\x^\T(\thetat_t - \theta_t)|\\
      = {}& k_\mu |\x^\T G_t^{-1}(\theta_t,\thetat_t)(g_t(\thetat_t) - g_t(\theta_t))|\\
      \leq {}& k_\mu \|\x\|_{G_t^{-1}(\theta_t,\thetat_t)}\|g_t(\thetat_t) - g_t(\theta_t)\|_{G_t^{-1}(\theta_t,\thetat_t)}\\
      \leq {}& \frac{k_\mu}{c_\mu} \|\x\|_{V_{t-1}^{-1}}\|g_t(\thetat_t) - g_t(\theta_t)\|_{V_{t-1}^{-1}}\\
      \leq {}& \frac{2k_\mu}{c_\mu} \|\x\|_{V_{t-1}^{-1}}\|g_t(\thetah_t) - g_t(\theta_t)\|_{V_{t-1}^{-1}},
  \end{align*}
  where the last step comes from the projection step. After doing the projection step, term $g_t(\thetah_t) - g_t(\theta_t)$ is the estimation error of the MLE without projection. 
Notice that in piecewise-stationary case,~\cite{AISTATS'21:SCB-forgetting} can also use this projection step.~\cite{arXiv'21:faury-driftingGLB} believe that these two previous works could use this projection operation mainly due to their stationary or piecewise-stationary setting. They mention that for the drifting case, the estimation error is always divided into the bias (tracking error) and variance (learning error) part, and this simple projection operation ignores the bias part which needs to be generalized to adapt to the two sources of deviation. In the analysis, the problem is that after the projection step estimation error term $g_t(\thetah_t) - g_t(\theta_t)$ need to be separate into the bias part and variance parts, and~\cite{arXiv'21:faury-driftingGLB} need to use $l_2$-norm for bias part and $V_{t-1}^{-1}$ for variance part. But the whole estimation error is already normed by $V_{t-1}^{-1}$, which means they cannot use the previous analysis of the window strategy for the bias part.

\begin{algorithm}[!t]
  \caption{\GLBweight~\cite{arXiv'21:faury-driftingGLB}}
  \label{alg:BVD-GLM-UCB}
\begin{algorithmic}[1]
\REQUIRE time horizon $T$, discounted factor $\gamma$, confidence $\delta$, regularizer $\lambda$, inverse link function $\mu$, parameters $S$, $L$ and $R$\\
\STATE Set $V_0 = \lambda I_d$, $\thetah_1 = \mathbf{0}$ and compute $k_\mu$ and $c_\mu$
\FOR{$t = 1,2,...,T$}
  \STATE Solving $ \theta_t^p \in \argmin_{\theta \in \mathbb{R}^d}\Big\{\norm{g_t(\theta)-g_t(\hat{\theta}_t)}_{V_{t}^{-2}}$ s.t $\Theta \cap \mathcal{E}_t^\delta(\theta) \neq \emptyset\Big\}$
  \STATE Select $\thetat_t \in \Theta \cap \mathcal{E}_t^\delta(\theta_t^p)$ where $\mathcal{E}_t^\delta(\theta):=\left\{\theta^{\prime} \in \R^d \givenn \norm{g_t\left(\theta^{\prime}\right)-g_t(\theta)}_{\Vt_{t}^{-1}} \leq \betab_t(\delta)\right\}$
  \STATE Get $\betab_{t-1} = R\sqrt{2\log\frac{1}{\delta}+d\log\sbr{1+\frac{L^2 (1-\gamma^{2t-2})}{\lambda d(1-\gamma^2)}}}+\sqrt{\lambda}c_\mu S$
  \STATE Get $X_t =\argmax_{\x \in \X} \bbr{ \mu(\x^\T \thetat_t)+ \frac{2k_\mu}{c_\mu}\betab_{t-1}\norm{\x}_{V_{t-1}^{-1}}}$
  \STATE Receive the reward $r_t$
  \STATE Update $V_{t} = \gamma V_{t-1} + X_t X_t^\T +(1-\gamma)\lambda I_d$, $\Vt_{t} = \gamma^2 V_{t-1} + X_t X_t^\T +(1-\gamma^2)\lambda I_d$
  \STATE Compute $\thetah_{t+1}$ by $\sum_{s=1}^{t}\gamma^{t-s}\sbr{\mu(X_s^\T \theta) - r_s}X_s+\lambda c_\mu \theta  = 0$
\ENDFOR
\end{algorithmic}
\end{algorithm}

To this end,~\cite{arXiv'21:faury-driftingGLB} propose the \GLBweight algorithm for drifting generalized linear bandits, as restated in Algorithm~\ref{alg:BVD-GLM-UCB}, where a new projection step is devised to solve this problem. Specifically, at each round $t$, the first step is to construct the confidence set $\mathcal{E}_t^\delta(\theta)$ which represents the influence of the stochastic noise.
\begin{equation}
  \label{Faury:confidence-set}
  \mathcal{E}_t^\delta(\theta):=\left\{\theta^{\prime} \in \R^d \givenn \norm{g_t\left(\theta^{\prime}\right)-g_t(\theta)}_{\Vt_{t}^{-1}} \leq \betab_t(\delta)\right\}.
\end{equation}
The second step is to find a confidence set $\mathcal{E}_t^\delta(\theta_t^p)$ that intersects with the feasible set, and the gap between $\theta_t^p$ and $\thetah_t$ represents the influence of parameter drift.
\begin{equation}
  \begin{split}
  \label{Faury:projection}
  \theta_t^p \in {}&\argmin_{\theta \in \mathbb{R}^d}\norm{g_t(\theta)-g_t(\hat{\theta}_t)}_{V_{t}^{-2}}\\{}&\text { s.t } \Theta \cap \mathcal{E}_t^\delta(\theta) \neq \emptyset\Bigg.
  \end{split}
\end{equation}
After obtaining the solution $\theta_t^p$ via computing the optimization problem~\eqref{Faury:projection}, the third step is to select $\thetat_t$ from $\Theta \cap \mathcal{E}_t^\delta(\theta_t^p)$. Based on this projection step,~\cite{arXiv'21:faury-driftingGLB} can separate the bias and variance parts before projection as follows,
\begin{align*}\nonumber
  {}&|\mu(\x^\T\thetat_t) - \mu(\x^\T\theta_t)| \leq k_\mu |\x^\T(\thetat_t - \theta_t)|\\
      = {}& k_\mu |\x^\T G_t^{-1}(\theta_t,\thetat_t)(g_t(\thetat_t) - g_t(\theta_t))|\\
      \leq {}& k_\mu |\x^\T G_t^{-1}(\theta_t,\thetat_t)(g_t(\thetat_t) - g_t(\theta_t^p)+ g_t(\theta_t^p)- g_t(\thetah_t)\\&+g_t(\thetah_t)-g_t(\thetab_t)+g_t(\thetab_t) - g_t(\theta_t))|\\
      \leq {}& \underbrace{k_\mu |\x^\T G_t^{-1}(\theta_t,\thetat_t)(g_t(\thetat_t) - g_t(\theta_t^p)+g_t(\thetah_t)-g_t(\thetab_t))|}_{\bias}\\
      {}&+\underbrace{k_\mu |\x^\T G_t^{-1}(\theta_t,\thetat_t)(g_t(\theta_t^p)- g_t(\thetah_t)+g_t(\thetab_t) - g_t(\theta_t))|}_{\variance}.
  \end{align*}
Their bias-variance decomposition motivates the choice of \emph{different} local norms for bounding bias and variance parts in their algorithm and analysis. Notably, due to the complications of the projection step (see~\eqref{Faury:confidence-set} and~\eqref{Faury:projection}), the overall algorithm is fairly complicated and less attractive for practical implementations, and moreover, it needs to maintain two covariance matrices $V_t$ and $\Vt_t$ (due to the constructed confidence region~\eqref{Faury:confidence-set}) at each round $t$ during the algorithm running. In the next section, we will show that the simple projection used in the stationary GLB~\eqref{eq:simple-projection} can be sufficient for coping with drifting GLB via our refined analysis framework.

\subsection{Proof of Lemma~\ref{lemma:GLB-estimation-error}}
\label{sec:GLB-estimation-error-proof}
\begin{proof}
  Based on the estimator equation~\eqref{eq:GLB-estimator}, we know that 
  \begin{equation}
  \begin{split}
  \label{eq:GLB-gt-thetah}
      g_t(\thetah_t) &= \lambda c_\mu \thetah_t + \sum_{s=1}^{t-1}w_{t-1,s}\mu(X_s^\T \thetah_t)X_s \\&= \sum_{s=1}^{t-1}w_{t-1,s}r_sX_s,
  \end{split}
  \end{equation}
  and then by the mean value theorem, we know that 
  \begin{equation}
  \label{eq:GLB-gt-mvt}
      g_t(\theta_1) - g_t(\theta_2)= G_t(\theta_1, \theta_2)(\theta_1 - \theta_2),
  \end{equation}
  where $G_t(\theta_1, \theta_2) \triangleq \int_{0}^1 \nabla g_t(s\theta_2+(1-s)\theta_1)\diff{s}\in \R^{d\times d}$. Notice that for any $\theta \in \Theta$, the gradient of $g_t$ is 
  \begin{equation}\nonumber
  \label{eq:GLB-gt-gradient}
      \nabla g_t(\theta) = \lambda c_\mu I_d + \sum_{s=1}^{t-1} w_{t-1,s}\dmu(X_s^\T \theta)X_sX_s^\T \succeq c_\mu V_{t-1},
  \end{equation}
  which clearly implies $\forall \theta_1, \theta_2 \in \Theta, G_t(\theta_1,\theta_2)\succeq c_\mu V_{t-1}$.

  By Assumption~\ref{ass:link-function}, the mean value theorem~\eqref{eq:GLB-gt-mvt} on $g_t$ and the projection~\eqref{eq:GLB-projection}, we have 
  \begin{align*} {}&|\mu(\x^\T\thetat_t) - \mu(\x^\T\theta_t)| \leq k_\mu |\x^\T(\thetat_t - \theta_t)|\\
      = {}& k_\mu |\x^\T G_t^{-1}(\theta_t,\thetat_t)(g_t(\thetat_t) - g_t(\theta_t))|\\
      \leq {}& k_\mu \|\x\|_{G_t^{-1}(\theta_t,\thetat_t)}\|g_t(\thetat_t) - g_t(\theta_t)\|_{G_t^{-1}(\theta_t,\thetat_t)}\\
      \leq {}& \frac{k_\mu}{c_\mu} \|\x\|_{V_{t-1}^{-1}}\|g_t(\thetat_t) - g_t(\theta_t)\|_{V_{t-1}^{-1}}\\
      \leq {}& \frac{2k_\mu}{c_\mu} \|\x\|_{V_{t-1}^{-1}}\|g_t(\thetah_t) - g_t(\theta_t)\|_{V_{t-1}^{-1}},
  \end{align*}
  then based on the model assumption, the function $g_t$~\eqref{eq:GLB-gt} and $g_t(\thetah_t)$~\eqref{eq:GLB-gt-thetah}, we have,
  \begin{align*}\nonumber
   {}&g_t(\theta_t)-g_t(\thetah_t)\\= {}& \lambda c_\mu\theta_t +\sum_{s=1}^{t-1}w_{t-1,s}\mu(X_s^\T\theta_t)X_s-\sum_{s=1}^{t-1}w_{t-1,s}r_sX_s\\
  = {}&\lambda c_\mu\theta_t +\sum_{s=1}^{t-1}w_{t-1,s}\mu(X_s^\T\theta_t)X_s\\{}&-\sum_{s=1}^{t-1}w_{t-1,s}(\mu(X_s^\T \theta_s) + \eta_s)X_s\\
  = {}&\underbrace{\sum_{s=1}^{t-1}w_{t-1,s}(\mu(X_s^\T\theta_t) - \mu(X_s^\T \theta_s) )X_s}_{\bias} \\{}&+ \underbrace{\lambda c_\mu \theta_t -\sum_{s=1}^{t-1}w_{t-1,s}\eta_sX_s}_{\variance}.
  \end{align*}
  Then, by the Cauchy-Schwarz inequality, we know that for any $\x\in\X$,
  \begin{equation}
  \begin{split}
  \label{eq:GLB-bound-cauchy}
      \abs{\mu(\x^\T\thetat_t) - \mu(\x^\T\theta_t)} \leq \frac{2k_\mu}{c_\mu}\|\x\|_{V_{t-1}^{-1}}\sbr{C_t + D_t},
  \end{split}
  \end{equation}
  where 
  \begin{equation}\nonumber
  \begin{split}
      C_t &= \norm{\sum_{s=1}^{t-1}w_{t-1,s}(\mu(X_s^\T\theta_t) - \mu(X_s^\T \theta_s) )X_s}_{V_{t-1}^{-1}}\\ D_t &= \norm{\sum_{s=1}^{t-1}w_{t-1,s}\eta_sX_s-\lambda c_\mu \theta_t}_{V_{t-1}^{-1}}.
  \end{split}
  \end{equation}
  This two terms can be bounded separately, as summarized in the following lemmas.
  \begin{myLemma}
    \label{lemma:GLB-C_t-bound}
    For any $t \in [T]$, we have 
    \begin{equation}
      \begin{split}
        {}&\norm{\sum_{s=1}^{t-1}w_{t-1,s}(\mu(X_s^\T\theta_t) - \mu(X_s^\T \theta_s) )X_s}_{V_{t-1}^{-1}}\\ \leq{}& Lk_\mu\sqrt{d}\sum_{p=1}^{t-1}\sqrt{\sum_{s=1}^{p}w_{t-1,s}} \norm{\theta_p - \theta_{p+1}}_2.
      \end{split}
    \end{equation}
    \end{myLemma}
    
    \begin{myLemma}
    \label{lemma:GLB-D_t-bound}
    For any $\delta \in (0,1)$, with probability at least $1-\delta$, the following holds for all $t \in [T]$,
    \begin{equation}
      \begin{split}
        {}&\norm{\sum_{s=1}^{t-1}w_{t-1,s}\eta_sX_s-\lambda c_\mu \theta_t}_{V_{t-1}^{-1}}\leq \sqrt{\lambda}c_\mu S\\{}&+R\sqrt{2\log\frac{1}{\delta}+d\log\sbr{1+\frac{L^2 \sum_{s=1}^{t-1}w_{t-1,s}}{\lambda d}}}.
      \end{split}
    \end{equation}
    \end{myLemma}
    \noindent Based on the inequality~\eqref{eq:GLB-bound-cauchy}, Lemma~\ref{lemma:GLB-C_t-bound}, Lemma~\ref{lemma:GLB-D_t-bound}, and the boundedness assumption of the feasible set, we have for any $\x \in \X$, $\gamma \in (0,1)$, $\delta \in (0,1)$, with probability at least $1-\delta$, the following holds for all $t \in [T]$,
    \begin{equation}\nonumber
    \begin{split}
        {}&\abs{\mu(\x^\T\thetat_t) - \mu(\x^\T\theta_t)}\leq \frac{2k_\mu}{c_\mu}\|\x\|_{V_{t-1}^{-1}}\\{}&\hspace{3em}\cdot\sbr{ Lk_\mu\sqrt{d}\sum_{p=1}^{t-1}\sqrt{\sum_{s=1}^{p}w_{t-1,s}} \norm{\theta_p - \theta_{p+1}}_2 + \betab_{t-1} }\\
        \leq{}&\frac{2k_\mu}{c_\mu}\Bigg( L^2k_\mu\sqrt{\frac{d}{\lambda}}\sum_{p=1}^{t-1}\sqrt{\sum_{s=1}^{p}w_{t-1,s}} \norm{\theta_p - \theta_{p+1}}_2 \\{}&\hspace{17em}+\betab_{t-1} \|\x\|_{V_{t-1}^{-1}} \Bigg),
    \end{split}
    \end{equation}
    where $\betab_t \triangleq \sqrt{\lambda}c_\mu S+R\sqrt{2\log\frac{1}{\delta}+d\log\sbr{1+\frac{L^2 \sum_{s=1}^{t}w_{t,s}}{\lambda d}}}$ is the confidence radius used in \GLBweightours. Hence we complete the proof.
\end{proof}
\begin{proof}[{Proof of Lemma~\ref{lemma:GLB-C_t-bound}}]

  Here we need to extract the variations of the time-varying parameter $\theta_t$
  \begin{align*}
  {}&\norm{\sum_{s=1}^{t-1}w_{t-1,s}(\mu(X_s^\T\theta_t) - \mu(X_s^\T \theta_s))X_s}_{V_{t-1}^{-1}}\\
  \leq {}& \norm{\sum_{s=1}^{t-1}w_{t-1,s}\sum_{p=s}^{t-1}(\mu(X_s^\T\theta_p) - \mu(X_s^\T \theta_{p+1}) )X_s}_{V_{t-1}^{-1}}\\
  = {}& \norm{\sum_{p=1}^{t-1}\sum_{s=1}^{p}w_{t-1,s}\alpha(X_s, \theta_p, \theta_{p+1})(X_s^\T\theta_p - X_s^\T \theta_{p+1})X_s}_{V_{t-1}^{-1}}\\
  \leq {}& \sum_{p=1}^{t-1}\norm{\sum_{s=1}^{p}w_{t-1,s}\alpha(X_s, \theta_p, \theta_{p+1})X_s\|X_s\|_2\|\theta_p - \theta_{p+1}\|_2}_{V_{t-1}^{-1}}\\
  \leq {}& L\sum_{p=1}^{t-1}\sum_{s=1}^{p}w_{t-1,s}|\alpha(X_s, \theta_p, \theta_{p+1})| \norm{X_s}_{V_{t-1}^{-1}} \|\theta_p - \theta_{p+1}\|_2\\
  \leq {}& Lk_\mu\sum_{p=1}^{t-1}\sum_{s=1}^{p}w_{t-1,s} \norm{X_s}_{V_{t-1}^{-1}} \|\theta_p - \theta_{p+1}\|_2.
  \end{align*}
  where the fourth equation is due to the mean value theorem where $\alpha(\x, \theta_1, \theta_2) = \int_{0}^1 \dmu(v\x^\T\theta_2+(1-v)x^\T\theta_1)\diff{v}$:
  $$\mu(X_s^\T\theta_p) - \mu(X_s^\T \theta_{p+1}) = \alpha(X_s, \theta_p, \theta_{p+1})(X_s^\T\theta_p - X_s^\T \theta_{p+1}).$$
  Next, the derivation of bound of term $\sum_{s=1}^{p}w_{t-1,s} \norm{X_s}_{V_{t-1}^{-1}}$ is the same as the inequality~\eqref{eq:LB-At-gammapart} in Appendix~\ref{sec:LB-A_t-bound-proof}, hence we complete the proof.
\end{proof}

\begin{proof}[{Proof of Lemma~\ref{lemma:GLB-D_t-bound}}]

  Same as the linear case, we define $\etat_s \define \sqrt{w_{t-1,s}}\eta_s$ and $\Xt_s \define \sqrt{w_{t-1,s}} X_s$, and notice that $\forall t\in[T] , s\in [t-1], \abs{w_{t-1,s}}\leq 1$, then $\etat_s$ is still $R$-sub-Gaussian, then by Theorem~\ref{thm:snc-AY}, we have 
  \begin{align*}
    D_t ={}& \norm{\sum_{s=1}^{t-1}w_{t-1,s}\eta_sX_s-\lambda c_\mu \theta_t}_{V_{t-1}^{-1}}\\
    \leq {}& \norm{\sum_{s=1}^{t-1}w_{t-1,s}\eta_sX_s}_{V_{t-1}^{-1}}+\norm{\lambda c_\mu \theta_t}_{V_{t-1}^{-1}}\\
    \leq {}& \norm{\sum_{s=1}^{t-1}\etat_s\Xt_s}_{V_{t-1}^{-1}} +\sqrt{\lambda}c_\mu S\\
    \leq {}& R\sqrt{2\log \frac{1}{\delta}+d\log\left(1+\frac{L^2\sum_{s=1}^{t-1}w_{t-1,s}}{d\lambda}\right)}+\sqrt{\lambda}c_\mu S.
  \end{align*}
The proof here is the same as the proof of Lemma~\ref{lemma:LB-B_t-bound} in \ref{sec:LB-B_t-bound-proof}, the only difference is an extra $c_\mu$ in the second term.
\end{proof}

\subsection{Proof of~\pref{thm:GLB-regret}}
\label{sec:GLB-regret-proof}
\begin{proof}
  Let $X_t^* \triangleq \argmax_{\x\in \X}\mu(\x^\T \theta_t)$. Due to Lemma~\ref{lemma:GLB-estimation-error} and the fact that $X_t^*,X_t\in \X$, each of the following holds with probability at least $1-\delta$,
  \begin{equation}\nonumber
    \begin{split}
      &\forall t \in [T], \mu(X_t^{*\T}\theta_t) \leq  \mu(X_t^{*\T}\thetat_t)\\&+\frac{2k_\mu}{c_\mu}\Bigg( L^2k_\mu\sqrt{\frac{d}{\lambda}}\sum_{p=1}^{t-1}\sqrt{\sum_{s=1}^{p}w_{t-1,s}} \norm{\theta_p - \theta_{p+1}}_2 \\{}&\hspace{15em} + \betab_{t-1}\|X_t^*\|_{V_{t-1}^{-1}} \Bigg),\\
      &\forall t \in [T], \mu(X_t^{\T}\theta_t) \geq \mu(X_t^{\T}\thetat_t)\\ &-\frac{2k_\mu}{c_\mu}\Bigg( L^2k_\mu\sqrt{\frac{d}{\lambda}}\sum_{p=1}^{t-1}\sqrt{\sum_{s=1}^{p}w_{t-1,s}} \norm{\theta_p - \theta_{p+1}}_2 \\{}&\hspace{15em} + \betab_{t-1}\|X_t\|_{V_{t-1}^{-1}} \Bigg).
    \end{split}
  \end{equation}
  By the union bound, the following holds with probability at least $1-2\delta$: $\forall t\in [T]$
  \begin{align*}
      {}& \mu(X_t^{*\T}\theta_t) - \mu(X_t^{\T}\theta_t) \leq \mu(X_t^{*\T}\thetat_t) -\mu(X_t^{\T}\thetat_t)\\{}&\qquad+\frac{2k_\mu}{c_\mu}\sbr{\betab_{t-1}\|X_t^*\|_{V_{t-1}^{-1}}+\betab_{t-1}\|X_t\|_{V_{t-1}^{-1}} }\\{}&\qquad + \frac{4L^2k_\mu^2}{c_\mu}\sqrt{\frac{d}{\lambda}}\sum_{p=1}^{t-1}\sqrt{\sum_{s=1}^{p}w_{t-1,s}} \norm{\theta_p - \theta_{p+1}}_2\\
      \leq{}& \frac{4L^2k_\mu^2}{c_\mu}\sqrt{\frac{d}{\lambda}}\sum_{p=1}^{t-1}\sqrt{\sum_{s=1}^{p}w_{t-1,s}} \norm{\theta_p - \theta_{p+1}}_2\\{}&\qquad+ \frac{4k_\mu}{c_\mu}\betab_{t-1}\|X_t\|_{V_{t-1}^{-1}},
  \end{align*}
  where the last step comes from the arm selection criterion~\eqref{eq:GLB-select-criteria} such that 
  \begin{align*}
    {}&\mu(X_t^{*\T} \thetat_t)+ \frac{2k_\mu}{c_\mu}\betab_{t-1}\|X_t^*\|_{V_{t-1}^{-1}}\\ \leq{}& \mu(X_t^{\T} \thetat_t)+ \frac{2k_\mu}{c_\mu}\betab_{t-1}\|X_t\|_{V_{t-1}^{-1}}. 
  \end{align*}
  Hence the following dynamic regret bound holds with probability at least $1-2\delta$ and can be divided into two parts,
  \begin{align*}\nonumber
      \DReg_T = {}& \sum_{t=1}^T \max_{\x\in \X}\mu(\x^\T \theta_t) - \mu(X_t^\T \theta_t)\\
      \leq {}& \underbrace{\frac{4L^2k_\mu^2}{c_\mu}\sqrt{\frac{d}{\lambda}}\sum_{t=1}^T\sum_{p=1}^{t-1}\sqrt{\sum_{s=1}^{p}w_{t-1,s}} \norm{\theta_p - \theta_{p+1}}_2}_{\bias}\\{}&+ \underbrace{\frac{4k_\mu}{c_\mu}\betab_T\sum_{t=1}^T\|X_t\|_{V_{t-1}^{-1}}}_{\variance} .
  \end{align*}
  where $\betab_{t}= \sqrt{\lambda}c_\mu S+R\sqrt{2\log\frac{1}{\delta}+d\log\sbr{1+\frac{L^2 (1-\gamma^{2t})}{\lambda d(1-\gamma^2)}}}$ is the confidence radius.

  Now we derive the upper bound for these two parts.

  \parag{Bias Part.~}
  Similar to the proof of inequality~\eqref{eq:LB-regret-bias-bound}, we have
  \begin{align*}
    {}&\frac{4L^2k_\mu^2}{c_\mu}\sqrt{\frac{d}{\lambda}}\sum_{t=1}^T \sum_{p=1}^{t-1}   \sqrt{\sum_{s=1}^{p}w_{t-1,s}}\norm{\theta_p -\theta_{p+1}}_2 \\
    \leq {}& \frac{8L^2k_\mu^2}{c_\mu}\sqrt{\frac{d}{\lambda}}\frac{1}{(1-\gamma)^{\sfrac{3}{2}}}P_T.
  \end{align*}

  \parag{Variance Part.}
  Similar to the proof of inequality~\eqref{eq:LB-regret-variance-bound}, let $C^{\text{GLB}}_T \define \frac{4k_\mu}{c_\mu}\betab_T\sqrt{2\max\bbr{1,L^2/\lambda}dT}$ we have
  \begin{equation}\nonumber
  \begin{split}
      {}&\frac{4k_\mu}{c_\mu}\betab_T\sqrt{T}\sqrt{\sum_{t=1}^T\|X_t\|_{V_{t-1}^{-1}}^2} \\ \leq {}& C^{\text{GLB}}_T\sqrt{ T\log\frac{1}{\gamma}+\log\sbr{1+ \frac{L^2}{\lambda d(1-\gamma)}}}.
  \end{split}
  \end{equation}
  Combine the upper bound for the bias and variance parts, and let $\delta = 1/(2T^2)$, we have the following regret bound with probability at least $1-1/T$,
  \begin{equation}\nonumber
    \begin{split}
      \DReg_T \leq {}& \frac{8L^2k_\mu^2}{c_\mu}\sqrt{\frac{d}{\lambda}}\frac{1}{(1-\gamma)^{\sfrac{3}{2}}}P_T\\{}&+ C^{\text{GLB}}_T\sqrt{ T\log\frac{1}{\gamma}+\log\sbr{1+ \frac{L^2}{\lambda d(1-\gamma)}}}.
  \end{split}
  \end{equation}
  where $\betab_{t}= R\sqrt{4\log T+2\log 2+d\log\sbr{1+\frac{L^2 (1-\gamma^{t})}{\lambda d(1-\gamma)}}}+\sqrt{\lambda}c_\mu S$. We set $\gamma \geq 1/T$ and $\lambda = d/c_\mu^2$, and obtain that,
  \begin{equation}\nonumber
  \begin{split}
      \DReg_T\leq{}& \Ot\sbr{k_\mu^2\frac{1}{(1-\gamma)^{\sfrac{3}{2}}}P_T + \frac{k_\mu}{c_\mu}d(1-\gamma)^{\sfrac{1}{2}}T}.
  \end{split}
  \end{equation}
  When $P_T< d/(k_\mu c_\mu T)$, we set $\gamma = 1-1/T$ and achieve an $\Ot(k_\mu c_\mu^{-1} d\sqrt{T})$ regret bound. When $P_T\geq d/(k_\mu c_\mu T)$, We set $\gamma$ optimally as $1-\gamma = \sqrt{k_\mu c_\mu P_T/(dT)}$ and attain an $\Ot(k_\mu^{\sfrac{5}{4}}c_\mu^{-\sfrac{3}{4}}d^{\sfrac{3}{4}}P_T^{\sfrac{1}{4}}T^{\sfrac{3}{4}})$ regret bound. Notice that, if $k_\mu < 1$, we just let $1-\gamma = \sqrt{c_\mu P_T/(dT)}$ and the regret bound becomes $\Ot(k_\mu^{2}c_\mu^{-\sfrac{3}{4}}d^{\sfrac{3}{4}}P_T^{\sfrac{1}{4}}T^{\sfrac{3}{4}})$.
\end{proof}
\section{Analysis of \SCBweightours}
\label{sec:SCB-regret}
In this section, we first present \SCBweightours algorithm in Algorithm~\ref{alg:SCB-WeightUCB}, Then, in Appendix~\ref{sec:SCB-estimation-error-proof} we present the proof of the estimation error upper bound of our \SCBweightours algorithm (Lemma~\ref{lemma:SCB-estimation-error}). Finally, in Appendix~\ref{sec:SCB-regret-proof}, we provide the proof of dynamic regret upper bound (Theorem~\ref{thm:SCB-regret}).

\begin{algorithm}[!t]
    \caption{\SCBweightours}
    \label{alg:SCB-WeightUCB}
  \begin{algorithmic}[1]
  \REQUIRE time horizon $T$, discounted factor $\gamma$, confidence $\delta$, regularizer $\lambda$, inverse link function $\mu$, parameters $S$, $L$ and $m$\\
  \STATE Set $V_0 = \lambda I_d$, $\thetah_1 = \mathbf{0}$ and compute $k_\mu$ and $c_\mu$
  \FOR{$t = 1,2,...,T$}
    \IF{$\|\thetah_t\|_2\leq S$} 
    \STATE let $\thetat_t = \thetah_t$
    \ELSE 
    \STATE Do the projection and get $\thetat_t$ by~\eqref{eq:SCB-projection}
    \ENDIF
    \STATE Compute $\betat_{t-1}$ by~\eqref{eq:SCB-confidence-radius}
    \STATE Select $X_t$ by~\eqref{eq:SCB-select-criteria}
    \STATE Receive the reward $r_t$
    \STATE Update $V_{t} = \gamma V_{t-1} + X_t X_t^\T +(1-\gamma)\lambda I_d$
    \STATE Compute $\thetah_{t+1}$ according to~\eqref{eq:GLB-estimator}
  \ENDFOR
  \end{algorithmic}
  \end{algorithm}

  \subsection{Proof of Lemma~\ref{lemma:SCB-estimation-error}}
\label{sec:SCB-estimation-error-proof}

\begin{proof}
  Based on the estimator equation~\eqref{eq:GLB-estimator}, we know that 
  \begin{equation}
    \begin{split}
  \label{eq:SCB-gt-thetah}
      g_t(\thetah_t) ={}& \lambda c_\mu \thetah_t + \sum_{s=1}^{t-1}w_{t-1,s}\mu(X_s^\T \thetah_t)X_s\\ ={}& \sum_{s=1}^{t-1}w_{t-1,s}r_sX_s,
  \end{split}
  \end{equation}
  and then by the mean value theorem, we know that 
  \begin{equation}
  \label{eq:SCB-gt-mvt}
      g_t(\theta_1) - g_t(\theta_2)= G_t(\theta_1, \theta_2)(\theta_1 - \theta_2),
  \end{equation}
  where $G_t(\theta_1, \theta_2) \triangleq \int_{0}^1 \nabla g_t(s\theta_2+(1-s)\theta_1)\diff{s}\in \R^{d\times d}$. Notice that for any $\theta \in \Theta$, the gradient of $g_t$ is 
  \begin{equation}\nonumber
  \label{eq:SCB-gt-gradient}
      \nabla g_t(\theta) = \lambda c_\mu I_d + \sum_{s=1}^{t-1} w_{t-1,s}\dmu(X_s^\T \theta)X_sX_s^\T \succeq c_\mu V_{t-1},
  \end{equation}
  which clearly implies $\forall \theta_1, \theta_2 \in \Theta, G_t(\theta_1,\theta_2)\succeq c_\mu V_{t-1}$ and $\forall \theta, H_t(\theta)\succeq c_\mu V_{t-1}$, where $H_t(\theta)$ is defined as
  \begin{equation}
    \label{eq:SCB-H}
    H_t(\theta) \define \lambda c_\mu I_d + \sum_{s=1}^{t-1}w_{t-1,s}\dmu(X_s^\T \theta)X_sX_s^\T.
    \end{equation}
  By Assumption~\ref{ass:link-function}, the mean value theorem~\eqref{eq:GLB-gt-mvt} on $g_t$, the projection~\eqref{eq:SCB-projection} and Lemma~\ref{lemma:SCB-G-H}, we have 
  \begin{align*}
    &|\mu(\x^\T\thetat_t) - \mu(\x^\T\theta_t)| \leq  k_\mu |\x^\T(\thetat_t - \theta_t)|\\
      = {}& k_\mu |\x^\T G_t^{-1}(\theta_t,\thetat_t)(g_t(\thetat_t) - g_t(\theta_t))|\\
      \leq {}& k_\mu \|\x\|_{G_t^{-1}(\theta_t,\thetat_t)}\|g_t(\thetat_t) - g_t(\theta_t)\|_{G_t^{-1}(\theta_t,\thetat_t)}\\
      \leq {}& k_\mu \|\x\|_{G_t^{-1}(\theta_t,\thetat_t)}\Big(\|g_t(\thetat_t) - g_t(\thetah_t)\|_{G_t^{-1}(\theta_t,\thetat_t)}\\{}& \quad+\|g_t(\thetah_t) - g_t(\theta_t)\|_{G_t^{-1}(\theta_t,\thetat_t)}\Big)\\
      \leq {}& \sqrt{1+2S}k_\mu \|\x\|_{G_t^{-1}(\theta_t,\thetat_t)}\Big(\|g_t(\thetat_t) - g_t(\thetah_t)\|_{H_t^{-1}(\thetat_t)}\\{}&\quad +\|g_t(\thetah_t) - g_t(\theta_t)\|_{H_t^{-1}(\theta_t)}\Big)\\
      \leq {}& 2\sqrt{1+2S}\frac{k_\mu}{\sqrt{c_\mu}} \|\x\|_{V_{t-1}^{-1}}\|g_t(\thetah_t) - g_t(\theta_t)\|_{H_t^{-1}(\theta_t)},
  \end{align*}
  then based on the model assumption~\eqref{eq:SCB-model-assume}, the function $g_t$~\eqref{eq:GLB-gt} and the $g_t(\thetah_t)$~\eqref{eq:SCB-gt-thetah}, we have,
  \begin{align*}
   {}&g_t(\theta_t)-g_t(\thetah_t)\\= {}& \lambda c_\mu\theta_t +\sum_{s=1}^{t-1}w_{t-1,s}\mu(X_s^\T\theta_t)X_s-\sum_{s=1}^{t-1}w_{t-1,s}r_sX_s\\
  = {}&\lambda c_\mu\theta_t +\sum_{s=1}^{t-1}w_{t-1,s}\mu(X_s^\T\theta_t)X_s\\{}&\quad-\sum_{s=1}^{t-1}w_{t-1,s}(\mu(X_s^\T \theta_s) + \eta_s)X_s\\
  = {}&\sum_{s=1}^{t-1}w_{t-1,s}(\mu(X_s^\T\theta_t) - \mu(X_s^\T \theta_s) )X_s+ \lambda c_\mu \theta_t \\{}&\quad-\sum_{s=1}^{t-1}w_{t-1,s}\eta_sX_s,
  \end{align*}
  then, by Cauchy-Schwarz inequality, we have 
  \begin{equation}
  \begin{split}
  \label{eq:SCB-bound-cauchy}
      {}&\abs{\mu(\x^\T\thetat_t) - \mu(\x^\T\theta_t)}\\ \leq{}& 2\sqrt{1+2S}\frac{k_\mu}{\sqrt{c_\mu}} \|\x\|_{V_{t-1}^{-1}}\sbr{E_t + F_t},
  \end{split}
  \end{equation}
  where 
  \begin{equation}\nonumber
  \begin{split}
      E_t &= \norm{\sum_{s=1}^{t-1}w_{t-1,s}(\mu(X_s^\T\theta_t) - \mu(X_s^\T \theta_s) )X_s}_{H_t^{-1}(\theta_t)}\\ F_t &= \norm{\sum_{s=1}^{t-1}w_{t-1,s}\eta_sX_s-\lambda c_\mu \theta_t}_{H_t^{-1}(\theta_t)}.
  \end{split}
  \end{equation}
  This two terms can be bounded separately.

  \begin{myLemma}
    \label{lemma:SCB-E_t-bound}
    For any $t \in [T]$, we have 
    \begin{equation}
      \begin{split}
        {}&\norm{\sum_{s=1}^{t-1}w_{t-1,s}(\mu(X_s^\T\theta_t) - \mu(X_s^\T \theta_s) )X_s}_{H_t^{-1}(\theta_t)}\\ \leq{}& L\frac{k_\mu}{\sqrt{c_\mu}}\sqrt{d}\sum_{p=1}^{t-1}\sqrt{\sum_{s=1}^{p}w_{t-1,s}} \norm{\theta_p - \theta_{p+1}}_2.
      \end{split}
    \end{equation}
    \end{myLemma}
    
    \begin{myLemma}
    \label{lemma:SCB-F_t-bound}
    For any  $\delta \in (0,1)$, with probability at least $1-\delta$, we have for all $t \in [T]$,
    \begin{equation}
    \begin{split}
        {}&\norm{\sum_{s=1}^{t-1}w_{t-1,s}\eta_sX_s-\lambda c_\mu \theta_t}_{H_t^{-1}(\theta_t)} \\ \leq{}& \frac{\sqrt{\lambda c_\mu}}{2 m }+\frac{2 m }{\sqrt{\lambda c_\mu}}\log\frac{1}{\delta}+\frac{2 m }{\sqrt{\lambda c_\mu}} d \log (2)+\sqrt{\lambda c_\mu} S\\
        &+\frac{d m }{\sqrt{\lambda c_\mu}} \log \sbr{1+\frac{ L^2k_\mu\sum_{s=1}^{t-1}w_{t-1,s}}{\lambda c_\mu d}},
    \end{split}
    \end{equation}
    \end{myLemma}
    \noindent Based on the inequality~\eqref{eq:SCB-bound-cauchy}, Lemma~\ref{lemma:GLB-C_t-bound} and Lemma~\ref{lemma:GLB-D_t-bound}, and the boundedness assumption of the feasible set, we have for any $\x \in \X$,  $\delta \in (0,1)$, with probability at least $1-\delta$, we have for all $t \in [T]$,
    \begin{equation}\nonumber
    \begin{split}
      {}&\abs{\mu(\x^\T\thetat_t) - \mu(\x^\T\theta_t)}\\\leq {}& 2\sqrt{1+2S}\frac{k_\mu}{\sqrt{c_\mu}}\|\x\|_{V_{t-1}^{-1}}\Bigg(\betat_{t-1}\\{}&\qquad\qquad +L\frac{k_\mu}{\sqrt{c_\mu}}\sqrt{d}\sum_{p=1}^{t-1}\sqrt{\sum_{s=1}^{p}w_{t-1,s}} \norm{\theta_p - \theta_{p+1}}_2\Bigg)\\
        \leq {}& 2\sqrt{1+2S}\frac{k_\mu}{\sqrt{c_\mu}}\Bigg(\betat_{t-1}\|\x\|_{V_{t-1}^{-1}}\\{}&\qquad\quad +L^2\frac{k_\mu}{\sqrt{\lambda c_\mu}}\sqrt{d}\sum_{p=1}^{t-1}\sqrt{\sum_{s=1}^{p}w_{t-1,s}} \norm{\theta_p - \theta_{p+1}}_2\Bigg),\\
    \end{split}
    \end{equation}
    where $\betat_t \define \frac{d m }{\sqrt{\lambda c_\mu}} \log \sbr{1+\frac{ L^2k_\mu\sum_{s=1}^{t-1}w_{t-1,s}}{\lambda c_\mu d}}+\frac{\sqrt{\lambda c_\mu}}{2 m }+\frac{2 m }{\sqrt{\lambda c_\mu}}\log\frac{1}{\delta}+\frac{2 m }{\sqrt{\lambda c_\mu}} d \log (2)+\sqrt{\lambda c_\mu} S$ is the confidence radius used in \SCBweightours. Hence we completes the proof.
\end{proof}

\begin{proof}[{Proof of Lemma~\ref{lemma:SCB-E_t-bound}}]
  Since $\forall \theta, H_t(\theta)\succeq c_\mu V_{t-1}$, we have 
  \begin{equation}\nonumber
    \begin{split}
      {}&\norm{\sum_{s=1}^{t-1}w_{t-1,s}(\mu(X_s^\T\theta_t) - \mu(X_s^\T \theta_s) )X_s}_{H_t^{-1}(\theta_t)}\\ \leq{}& \frac{1}{\sqrt{c_\mu}}\norm{\sum_{s=1}^{t-1}w_{t-1,s}(\mu(X_s^\T\theta_t) - \mu(X_s^\T \theta_s) )X_s}_{V_{t-1}^{-1}}.
    \end{split}
    \end{equation}
    Then use Lemma~\ref{lemma:GLB-C_t-bound} and we complete the proof.
\end{proof}

\begin{proof}[{Proof of Lemma~\ref{lemma:SCB-F_t-bound}}]
We define $\etat_s \define \frac{\sqrt{w_{t-1,s}}\eta_s}{\abs{m}}$, $\Xt_s \define \sqrt{w_{t-1,s}}X_s$ and notice that $\forall t\in[T] , s\in [t-1], \abs{w_{t-1,s}}\leq 1$, then $\etat_s$ is bounded by $1$ with variance $\tilde{\sigma}_s$, then we have 
  \begin{align*}
    F_t &= \norm{\sum_{s=1}^{t-1}w_{t-1,s}\eta_sX_s-\lambda c_\mu \theta_t}_{H_t^{-1}(\theta_t)}\\&\leq \norm{\sum_{s=1}^{t-1}w_{t-1,s}\eta_sX_s}_{H_t^{-1}(\theta_t)}+ \sqrt{\lambda c_\mu} S\\&= \norm{\sum_{s=1}^{t-1}\etat_s\Xt_s}_{\Ht_t^{-1}(\theta_t)}+\sqrt{\lambda c_\mu} S
  \end{align*}
  where  $\Ht_t(\theta) \define \frac{\lambda c_\mu}{m^2} I_d + \sum_{s=1}^{t-1}\frac{\dmu(X_s^\T \theta)}{m^2}\Xt_s\Xt_s^\T$. By the model assumption~\eqref{eq:SCB-model-assume}, we know that $\tilde{\sigma}_t^2 = \frac{\E[\etat_t^2|\F_{t-1}]}{m^2} \leq \E[\eta_t^2|\F_t] = \frac{\Var[r_t|\F_t]}{m^2} = \frac{\dmu(X_t\theta_t)}{m^2}$, then from the Self-normalized concentration inequality for self-concordant bandits~\cite[Theorem 1]{ICML'20:logistic-bandits}, restated in~\pref{thm:self-normalized-weight-SCB}, we can get the bound for the first term $\norm{\sum_{s=1}^{t-1}\etat_s\Xt_s}_{\Ht_t^{-1}(\theta_t)}$ as follows,
  \begin{equation}\nonumber
    \begin{split}
      {}&\norm{\sum_{s=1}^{t-1}w_{t-1,s}\etat_sX_s}_{H_t^{-1}(\theta_t)} \\\leq{}& \frac{\sqrt{\lambda c_\mu}}{2 m }+\frac{2 m }{\sqrt{\lambda c_\mu}} \log \sbr{\frac{\det(H_t)^{1 / 2}}{\delta (\lambda c_\mu )^{d / 2}}}+\frac{2 m }{\sqrt{\lambda c_\mu}} d \log (2),
    \end{split}
  \end{equation}
  then we have 
  \begin{equation}\nonumber
    \begin{split}
      {}&\norm{ \sum_{s=1}^{t-1}\gamma^{-s}\eta_sX_s}_{\Ht_t^{-1}(\theta_t)} \leq \frac{\sqrt{\lambda c_\mu}}{2 m }+\frac{2 m }{\sqrt{\lambda c_\mu}}\log\frac{1}{\delta}\\{}&+\frac{d m }{\sqrt{\lambda c_\mu}} \log \sbr{1+\frac{ L^2k_\mu\sum_{s=1}^{t-1}w_{t-1,s}}{\lambda c_\mu d}}+\frac{2 m }{\sqrt{\lambda c_\mu}} d \log (2).
    \end{split}
  \end{equation}
  Therefore, we get the upper bound for $F_t$ term.
\end{proof}

\subsection{Proof of~\pref{thm:SCB-regret}}
\label{sec:SCB-regret-proof}

\begin{proof}
  Let $X_t^* \triangleq \argmax_{\x\in \X}\mu(\x^\T \theta_t)$. Due to Lemma~\ref{lemma:GLB-estimation-error} and the fact that $X_t^*,X_t\in \X$, each of the following holds with probability at least $1-\delta$,
  \begin{equation}\nonumber
    \begin{split}
      \forall t\in [T], {}&\mu(X_t^{*\T}\theta_t)\\ \leq{}&  \mu(X_t^{*\T}\thetat_t) +2\sqrt{1+2S}\frac{k_\mu}{\sqrt{c_\mu}}\Bigg(\betat_{t-1}\|X_t^*\|_{V_{t-1}^{-1}}\\{}&+ L^2\frac{k_\mu}{\sqrt{\lambda c_\mu}}\sqrt{d}\sum_{p=1}^{t-1}\sqrt{\sum_{s=1}^{p}w_{t-1,s}} \norm{\theta_p - \theta_{p+1}}_2\Bigg),\\
      \forall t\in [T], {}&\mu(X_t^{\T}\theta_t)\\ \geq{}&  \mu(X_t^{\T}\thetat_t) -2\sqrt{1+2S}\frac{k_\mu}{\sqrt{c_\mu}}\Bigg(\betat_{t-1}\|X_t\|_{V_{t-1}^{-1}}\\{}&+ L^2\frac{k_\mu}{\sqrt{\lambda c_\mu}}\sqrt{d}\sum_{p=1}^{t-1}\sqrt{\sum_{s=1}^{p}w_{t-1,s}} \norm{\theta_p - \theta_{p+1}}_2\Bigg).
    \end{split}
  \end{equation}
  By the union bound, the following holds with probability at least $1-2\delta$: $\forall t\in [T]$
  \begin{equation}\nonumber
  \begin{split}
      {}&\mu(X_t^{*\T}\theta_t) - \mu(X_t^{\T}\theta_t) \leq \mu(X_t^{*\T}\thetat_t) -\mu(X_t^{\T}\thetat_t)\\ {}&+ 2\sqrt{1+2S}\bigg(\frac{2L^2k_\mu^2}{c_\mu}\sqrt{\frac{d}{\lambda}}\sum_{p=1}^{t-1}\sqrt{\sum_{s=1}^{p}w_{t-1,s}} \norm{\theta_p - \theta_{p+1}}_2 \\
      {}&+ \frac{k_\mu}{\sqrt{c_\mu}}\sbr{\betat_{t-1}\|X_t^*\|_{V_{t-1}^{-1}}+\betat_{t-1}\|X_t\|_{V_{t-1}^{-1}} }\bigg)\\
      \leq{}& \frac{4\sqrt{1+2S}L^2k_\mu^2}{c_\mu}\sqrt{\frac{d}{\lambda}}\sum_{p=1}^{t-1}\sqrt{\sum_{s=1}^{p}w_{t-1,s}} \norm{\theta_p - \theta_{p+1}}_2\\{}& + \frac{4\sqrt{1+2S}k_\mu}{\sqrt{c_\mu}}\betat_{t-1}\|X_t\|_{V_{t-1}^{-1}},
  \end{split}
  \end{equation}
  where the last step comes from the arm selection criterion~\eqref{eq:SCB-select-criteria} such that 
  \begin{align*}
    {}&\mu(X_t^{*\T} \thetat_t)+ 2\sqrt{1+2S}\frac{k_\mu}{\sqrt{c_\mu}}\betat_{t-1}\|X_t^*\|_{V_{t-1}^{-1}} \\\leq{}& \mu(X_t^{\T} \thetat_t)+ 2\sqrt{1+2S}\frac{k_\mu}{\sqrt{c_\mu}}\betat_{t-1}\|X_t\|_{V_{t-1}^{-1}}.
  \end{align*}
  Hence, the following dynamic regret bound holds with probability at least $1-2\delta$ and can be divided into two parts,
  \begin{equation}\nonumber
  \begin{split}
    &\DReg_T =  \sum_{t=1}^T \mu(X_t^{*\T} \theta_t) - \mu(X_t^\T \theta_t)\\
    &\leq \underbrace{\frac{4\sqrt{1+2S}L^2k_\mu^2}{c_\mu}\sqrt{\frac{d}{\lambda}}\sum_{t=1}^T\sum_{p=1}^{t-1}\sqrt{\sum_{s=1}^{p}w_{t-1,s}} \norm{\theta_p - \theta_{p+1}}_2}_{\bias}\\{}&\quad + \underbrace{\frac{4\sqrt{1+2S}k_\mu}{\sqrt{c_\mu}}\betat_T\sum_{t=1}^T\|X_t\|_{V_{t-1}^{-1}}}_{\variance}.
  \end{split}
  \end{equation}
  where $\betat_t = \frac{d m }{\sqrt{\lambda c_\mu}} \log \sbr{1+\frac{ L^2k_\mu(1-\gamma^{t})}{\lambda c_\mu d(1-\gamma)}}+\frac{2 m }{\sqrt{\lambda c_\mu}}\log\frac{1}{\delta}+\frac{\sqrt{\lambda c_\mu}}{2 m }+\frac{2 m }{\sqrt{\lambda c_\mu}} d \log (2)+\sqrt{\lambda c_\mu} S$ is the confidence radius.

  Now we derive the upper bound for these two parts.

  \parag{Bias Part.} Similar to the proof of inequality~\eqref{eq:LB-regret-bias-bound}, we have 
  \begin{align*}
    {}&\frac{4\sqrt{1+2S}L^2k_\mu^2}{c_\mu}\sqrt{\frac{d}{\lambda}}\sum_{t=1}^T \sum_{p=1}^{t-1}   \sqrt{\sum_{s=1}^{p}w_{t-1,s}}\norm{\theta_p -\theta_{p+1}}_2 \\\leq{}& \frac{8\sqrt{1+2S}L^2k_\mu^2}{c_\mu}\sqrt{\frac{d}{\lambda}}\frac{1}{(1-\gamma)^{\sfrac{3}{2}}}P_T.
  \end{align*}

  \parag{Variance Part.}
  First use the Cauchy-Schwarz inequality, we know that
  \begin{equation}\nonumber
  \begin{split}
    {}&\frac{4\sqrt{1+2S}k_\mu}{\sqrt{c_\mu}}\betat_T\sum_{t=1}^T\|X_t\|_{V_{t-1}^{-1}}\\
      \leq {}& \frac{4\sqrt{1+2S}k_\mu}{\sqrt{c_\mu}}\betat_T\sqrt{T}\sqrt{\sum_{t=1}^T\|X_t\|_{V_{t-1}^{-1}}^2}.
  \end{split}
  \end{equation}
For term $\sqrt{\sum_{t=1}^T\|X_t\|_{V_{t-1}^{-1}}^2}$, we can use the Lemma~\ref{lemma:potential-lemma} to bound it, Let $C^{\text{SCB}}_T \define \frac{4\sqrt{1+2S}k_\mu}{\sqrt{c_\mu}}\betat_T\sqrt{2\max\{1,L^2/\lambda\}dT}$ then 
  \begin{equation}\nonumber
  \begin{split}
    {}&\frac{4\sqrt{1+2S}k_\mu}{\sqrt{c_\mu}}\betat_T\sqrt{T}\sqrt{\sum_{t=1}^T\|X_t\|_{V_{t-1}^{-1}}^2} \\\leq{}& C^{\text{SCB}}_T\sqrt{ T\log\frac{1}{\gamma}+\log\sbr{1+ \frac{L^2}{\lambda d(1-\gamma)}}}.
  \end{split}
  \end{equation}
  Combining the upper bound for the bias and variance parts, and letting $\delta = 1/(2T)$, we have the following regret bound with probability at least $1-1/T$,
  \begin{equation}\nonumber
    \begin{split}
      \DReg_T \leq{}& \frac{8\sqrt{1+2S}L^2k_\mu^2}{c_\mu}\sqrt{\frac{d}{\lambda}}\frac{1}{(1-\gamma)^{\sfrac{3}{2}}}P_T\\{}&+ C^{\text{SCB}}_T\sqrt{ T\log\frac{1}{\gamma}+\log\sbr{1+ \frac{L^2}{\lambda d(1-\gamma)}}}.
  \end{split}
  \end{equation}
  where $\betat_t = \frac{d m }{\sqrt{\lambda c_\mu}} \log \sbr{1+\frac{ L^2k_\mu(1-\gamma^{t})}{\lambda c_\mu d(1-\gamma)}}+\frac{2 m }{\sqrt{\lambda c_\mu}}\log\sbr{2T}+\frac{\sqrt{\lambda c_\mu}}{2 m }+\frac{2 m }{\sqrt{\lambda c_\mu}} d \log (2)+\sqrt{\lambda c_\mu} S$. Since there is a $T \sqrt{\log (1/\gamma)}$ term in the regret bound, which means that we cannot let $\gamma$ close to $0$, so we set $\gamma \geq 1/T$, then we have $\log(1/\gamma) \leq C (1-\gamma)$, where $C = \log T/(1-1/T)$. Then, ignoring logarithmic factors in time horizon $T$, and let $\lambda = d\log(T)/c_\mu$, we finally obtain that,
  \begin{equation}\nonumber
  \begin{split}
      \DReg_T\leq{}& \Ot\sbr{\frac{k_\mu^2}{\sqrt{c_\mu}}\frac{1}{(1-\gamma)^{\sfrac{3}{2}}}P_T + \frac{k_\mu}{\sqrt{c_\mu}}d(1-\gamma)^{\sfrac{1}{2}}T}.
  \end{split}
  \end{equation}
  When $P_T< d/(k_\mu T)$ (which corresponds a small amount of non-stationarity), we simply set $\gamma = 1-1/T$ and achieve an $\Ot(k_\mu c_\mu^{-\sfrac{1}{2}}d\sqrt{T})$ regret bound.  Besides, when coming to the non-degenerated case of $P_T\geq d/(k_\mu T)$, We set the discounted factor optimally as $1-\gamma = \sqrt{k_\mu P_T/(dT)}$ and attain an $\Ot(k_\mu^{\sfrac{5}{4}}c_\mu^{-\sfrac{1}{2}}d^{\sfrac{3}{4}}P_T^{\sfrac{1}{4}}T^{\sfrac{3}{4}})$ regret bound, which completes the proof.
\end{proof}
\section{Piecewise-Stationary SCB}
\label{sec:SCB-PW}
In this section, we study SCB under piecewise-stationary environment and our work is a direct improvement over~\cite{AISTATS'21:SCB-forgetting}. Next, we will first propose our \SCBweightourspw algorithm, and then, present the analysis of the confidence set. Finally, we give the proof of the dynamic regret upper bound.

\subsection{SCB-PW-WeightUCB Algorithm}
\label{sec:SCB-PW-algorithm}
Inspired by \cite{AISTATS'21:optimal-logistic-bandits}, we make a direct improvement over \cite{AISTATS'21:SCB-forgetting}. Just like~\cite{AISTATS'21:SCB-forgetting}, for $D\geq 1$, define $\mathcal{T}(D) = \{1\leq t\leq T, \text{~such that~} \theta_s = \theta_t \text{~for~} t-D\leq s\leq t-1\}$. $t\in\mathcal{T}(D)$ when $t$ is at least $D$ steps away from the previous closest changing point. But the difference is that~\cite{AISTATS'21:SCB-forgetting} considers $D$ as an analysis parameter, and we treat $D$ as a tunable algorithm parameter. Notice that, the $D$ here is \emph{not} a virtual window size, but the algorithm's estimate of how durable the environment is stationary.

\parag{Estimator.} At iteration $t$, we adopt the same maximum likelihood estimator~\eqref{eq:GLB-estimator} with $w_{t,s} = \gamma^{t-s}$ as in the drifting case. 

\parag{Confidence Set.} We further construct confidence set for the real $\theta_t$. For $\delta \in (0,1)$, we define,
\begin{equation*}
  \label{eq:SCB-PW-confidence-set}
  \begin{split}
    \C_t(\delta) \define \bbr{\theta \in \Theta \givenn \|g_t(\theta) - g_t(\thetah_t)\|_{H_t^{-1}(\theta)}\leq \rho_t},
  \end{split}
\end{equation*}
where $\rho_t = \frac{2 L^2S k_\mu}{\sqrt{\lambda c_\mu}} \frac{\gamma^D}{1-\gamma} + \frac{Lm}{\sqrt{\lambda c_\mu}}\frac{\gamma^D}{1-\gamma} + \betabr_t$ and $\betabr_t = \frac{d m }{\sqrt{\lambda c_\mu}}\log \sbr{1 + \frac{ L^2k_\mu(1-\gamma^{2D})}{\lambda c_\mu d(1-\gamma)}}+\frac{\sqrt{\lambda c_\mu}}{2 m }+\frac{2 m }{\sqrt{\lambda c_\mu}}\log\frac{1}{\delta} +\frac{2 m }{\sqrt{\lambda c_\mu}} d \log (2) + \sqrt{\lambda c_\mu}S$.
\begin{myLemma}
  \label{lemma:SCB-PW-confidence-set}
  For any $\delta \in (0,1)$, with probability at least $1-\delta$, we have $\forall t\in \mathcal{T}(D), \theta_t \in \C_t(\delta)$.
  \begin{equation*}
    \begin{split}
      \C_t(\delta) =\Bigg\{\theta \in \Theta \given {}&\|g_t(\theta) - g_t(\thetah_t)\|_{H_t^{-1}(\theta)}\\\leq{}&  \frac{2 L^2S k_\mu}{\sqrt{\lambda c_\mu}} \frac{\gamma^D}{1-\gamma} + \frac{Lm}{\sqrt{\lambda c_\mu}}\frac{\gamma^D}{1-\gamma} + \betabr_t \Bigg\},
    \end{split}
  \end{equation*}
  where $\betabr_t = \frac{d m }{\sqrt{\lambda c_\mu}}\log \sbr{1 + \frac{ L^2k_\mu(1-\gamma^{2D})}{\lambda c_\mu d(1-\gamma)}}+\frac{\sqrt{\lambda c_\mu}}{2 m }+\frac{2 m }{\sqrt{\lambda c_\mu}}\log\frac{1}{\delta} +\frac{2 m }{\sqrt{\lambda c_\mu}} d \log (2) + \sqrt{\lambda c_\mu}S$.
\end{myLemma}
The proof of Lemma~\ref{lemma:SCB-PW-confidence-set} is presented in Appendix~\ref{sec:SCB-PW-estimation-error-proof}.

\parag{Selection Criteria.}
Algorithms discussed earlier for drifting cases are using bonus-based selection criteria. But here we use a parameter-based selection criterion as follows,
\begin{equation}
\label{eq:SCB-PW-select-criteria}
\begin{split}
  (X_t,\thetat_t) = \argmax_{\x\in\X, \theta\in \C_t(\delta)}\mu(\x^\T\theta).
\end{split}
\end{equation}
The main difference between parameter-based and bonus-based selection criteria is discussed in Section 3.2 of~\cite{AISTATS'21:optimal-logistic-bandits}. The overall algorithm is summarized in Algorithm~\ref{alg:SCB-PW-WeightUCB}.

\begin{algorithm}[!t]
  \caption{SCB-PW-WeightUCB}
  \label{alg:SCB-PW-WeightUCB}
\begin{algorithmic}[1]
\REQUIRE time horizon $T$, discounted factor $\gamma$, confidence $\delta$, regularizer $\lambda$, inverse link function $\mu$, parameters $S$, $L$ and $m$, changing confidence $D$\\
\STATE Set $\thetah_0 = \mathbf{0}$ and compute $k_\mu$ and $c_\mu$
\FOR{$t = 1,2,3,...,T$}
  \STATE Compute $(X_t,\thetat_t) = \argmax_{\x\in\X, \theta\in \C_t(\delta)}\mu(\x^\T\theta)$ 
  \STATE Select $X_t$ and receive the reward $r_t$
  \STATE Compute $\thetah_{t+1}$ according to~\eqref{eq:GLB-estimator}
\ENDFOR
\end{algorithmic}
\end{algorithm}

\subsection{Proof of Lemma~\ref{lemma:SCB-PW-confidence-set}}
\label{sec:SCB-PW-estimation-error-proof}

\begin{proof}
  Based on the model assumption~\eqref{eq:SCB-model-assume}, the function $g_t$~\eqref{eq:GLB-gt} and the $g_t(\thetah_t)$~\eqref{eq:SCB-gt-thetah}, we have,
  \begin{align*}
    {}&g_t(\theta_t)-g_t(\thetah_t)\\= {}& \lambda c_\mu\theta_t +\sum_{s=1}^{t-1}\gamma^{t-s-1}\mu(X_s^\T\theta_t)X_s-\sum_{s=1}^{t-1}\gamma^{t-s-1}r_sX_s\\
    = {}&\lambda c_\mu\theta_t +\sum_{s=1}^{t-1}\gamma^{t-s-1}\mu(X_s^\T\theta_t)X_s\\{}&-\sum_{s=1}^{t-1}\gamma^{t-s-1}(\mu(X_s^\T \theta_s) + \eta_s)X_s\\
    = {}&\sum_{s=1}^{t-1}\gamma^{t-s-1}(\mu(X_s^\T\theta_t) - \mu(X_s^\T \theta_s) )X_s + \lambda c_\mu \theta_t \\{}&-\sum_{s=1}^{t-1}\gamma^{t-s-1}\eta_sX_s.
  \end{align*}
  Then,
  \begin{align*}
    {}&\|g_t(\theta_t) - g_t(\thetah_t)\|_{H_t^{-1}(\theta_t)} \\ ={}& \Bigg\|\sum_{s=1}^{t-1}\gamma^{t-s-1}(\mu(X_s^\T\theta_t) - \mu(X_s^\T \theta_s) )X_s \\{}&\hspace{5em} + \lambda c_\mu \theta_t-\sum_{s=1}^{t-1}\gamma^{t-s-1}\eta_sX_s\Bigg\|_{H_t^{-1}(\theta_t)} \\
    \leq {}& \norm{\sum_{s=1}^{t-1}\gamma^{t-s-1}(\mu(X_s^\T\theta_t) - \mu(X_s^\T \theta_s) )X_s}_{H_t^{-1}(\theta_t)} \\
    &+ \norm{\lambda c_\mu \theta_t -\sum_{s=1}^{t-1}\gamma^{t-s-1}\eta_sX_s}_{H_t^{-1}(\theta_t)}\\
    \leq {}& \underbrace{\norm{\sum_{s=1}^{t-1}\gamma^{t-s-1}(\mu(X_s^\T\theta_t) - \mu(X_s^\T \theta_s) )X_s}_{H_t^{-1}(\theta_t)}}_{\term{a}}\\ &+ \underbrace{\norm{\sum_{s=1}^{t-D-1}\gamma^{t-s-1}\eta_sX_s}_{H_t^{-1}(\theta_t)}}_{\term{b}}\\
    &+ \underbrace{\norm{\sum_{s=t-D}^{t-1}\gamma^{t-s-1}\eta_sX_s - \lambda c_\mu \theta_t}_{H_t^{-1}(\theta_t)}}_{\term{c}}.
  \end{align*}
  \parag{Term~(a).}
  Since $t \in \mathcal{T}(D)$, we have
  \begin{align*}
    {}&\norm{\sum_{s=1}^{t-1}\gamma^{t-s-1}(\mu(X_s^\T\theta_t) - \mu(X_s^\T \theta_s) )X_s}_{H_t^{-1}(\theta_t)} \\ ={}& \norm{\sum_{s=1}^{t-D-1}\gamma^{t-s-1}(\mu(X_s^\T\theta_t) - \mu(X_s^\T \theta_s) )X_s}_{H_t^{-1}(\theta_t)}\\
    \leq {}& \norm{\sum_{s=1}^{t-D-1}\gamma^{t-s-1}k_\mu X_s^\T(\theta_t-\theta_s)X_s}_{H_t^{-1}(\theta_t)}\\
    \leq{}& \sum_{s=1}^{t-D-1}\gamma^{t-s-1}k_\mu \|X_s\|_2\|(\theta_t-\theta_s)\|_2\norm{X_s}_{H_t^{-1}(\theta_t)}\\
    \leq{}& \frac{2 L^2S k_\mu}{\sqrt{\lambda c_\mu}} \frac{\gamma^D}{1-\gamma}. 
  \end{align*}

  \parag{Term~(b).}
  \begin{equation}\nonumber
    \begin{split}
      \norm{\sum_{s=1}^{t-D-1}\gamma^{t-s-1}\eta_sX_s}_{H_t^{-1}(\theta_t)} \leq{}& \sum_{s=1}^{t-D-1}\gamma^{t-s-1} m \norm{X_s}_{H_t^{-1}(\theta_t)}\\\leq{}& \frac{Lm}{\sqrt{\lambda c_\mu}}\sum_{s=1}^{t-D-1}\gamma^{t-s-1}\\\leq{}& \frac{Lm}{\sqrt{\lambda c_\mu}}\frac{\gamma^D}{1-\gamma}. \\
    \end{split}
  \end{equation}

  \parag{Term~(c).}We define $\etat_s \define \frac{\sqrt{\gamma^{t-s-1}}\eta_s}{\abs{m}}$, $\Xt_s \define \sqrt{\gamma^{t-s-1}}X_s$ and notice that $\forall t\in[T] , s\in [t-1], \abs{\gamma^{t-s-1}}\leq 1$, then $\etat_s$ is bounded by $1$ with variance $\tilde{\sigma}_s$. Let $\Ht_t(\theta) \define \frac{\lambda c_\mu}{m^2} I_d + \sum_{s=1}^{t-1}\frac{\dmu(X_s^\T \theta)}{m^2}\Xt_s\Xt_s^\T$ and $\Ht_{t-D:t}(\theta) = \lambda c_\mu I_d + \sum_{s=t-D}^{t-1}\frac{\dmu(X_s^\T \theta)}{m^2}\Xt_s\Xt_s^\T$,
  \begin{align*}
      {}&\norm{\sum_{s=t-D}^{t-1}\gamma^{t-s-1}\eta_sX_s - \lambda c_\mu \theta_t}_{H_t^{-1}(\theta_t)}\\\leq{}& \norm{\sum_{s=t-D}^{t-1}\gamma^{t-s-1}\eta_sX_s}_{H_t^{-1}(\theta_t)} + \sqrt{\lambda c_\mu} S \\\leq{}& \norm{\sum_{s=t-D}^{t-1}\etat_s\Xt_s}_{\Ht_{t-D:t}^{-1}(\theta_t)} + \sqrt{\lambda c_\mu} S,
  \end{align*}
  where $\Ht_{t}(\theta) \succeq \Ht_{t-D:t}(\theta)$. Next, we need to bound the term $\|\sum_{s=t-D}^{t-1}\etat_s\Xt_s\|_{\Ht_{t-D:t}^{-1}(\theta_t)}$ using self-normalization bound~\cite[Theorem 1]{ICML'20:logistic-bandits}, restated in Theorem~\ref{thm:self-normalized-weight-SCB}, similar to the proof of Lemma~\ref{lemma:SCB-F_t-bound}, we have
  \begin{align*}
      &\norm{ \sum_{s=t-D}^{t-1}\gamma^{-s}\eta_sX_s}_{\Ht_{t-D:t}^{-1}(\theta_t)}\\
       \leq& \frac{\sqrt{\lambda c_\mu}}{2 m }+\frac{2 m }{\sqrt{\lambda c_\mu}} \log \sbr{\frac{\det\sbr{\Ht_{t-D:t}}^{1 / 2}}{\delta (\lambda c_\mu)^{d / 2}}}+\frac{2 m }{\sqrt{\lambda c_\mu}} d \log (2)\\
       \leq& \frac{\sqrt{\lambda c_\mu}}{2 m }+\frac{2 m }{\sqrt{\lambda c_\mu}}\log\frac{1}{\delta} +\frac{d m }{\sqrt{\lambda c_\mu}}\log \sbr{1 + \frac{ L^2k_\mu(1-\gamma^{2D})}{\lambda c_\mu d(1-\gamma)}}\\&+\frac{2 m }{\sqrt{\lambda c_\mu}} d \log (2).
    \end{align*} 
  Let $\betabr_t \define \frac{d m }{\sqrt{\lambda c_\mu}}\log \sbr{1 + \frac{ L^2k_\mu(1-\gamma^{2D})}{\lambda c_\mu d(1-\gamma)}}+\frac{2 m }{\sqrt{\lambda c_\mu}}\log\frac{1}{\delta}+\frac{\sqrt{\lambda c_\mu}}{2 m } +\frac{2 m }{\sqrt{\lambda c_\mu}} d \log (2) + \sqrt{\lambda c_\mu}S$
  , finally we have,
  \begin{equation}\nonumber
    \begin{split}
      {}&\|g_t(\theta_t) - g_t(\thetah_t)\|_{H_t^{-1}(\theta_t)}\\ \leq{}& \frac{2 L^2S k_\mu}{\sqrt{\lambda c_\mu}} \frac{\gamma^D}{1-\gamma} + \frac{Lm}{\sqrt{\lambda c_\mu}}\frac{\gamma^D}{1-\gamma} + \betabr_t,
    \end{split}
  \end{equation}
  which completes the proof. 
\end{proof}

\subsection{Proof of~\pref{thm:SCB-PW-regret}}
\label{sec:SCB-PW-regret-proof}

\begin{proof}
  Let $R_t = \mu(X_t^{*\T} \theta_t) - \mu(X_t^\T \theta_t)$
  \begin{equation}\nonumber
    \begin{split}
        \DReg_T = \sum_{t=1}^T R_t ={}& \sum_{t\notin \mathcal{T}(D)} R_t + \sum_{t\in \mathcal{T}(D)} R_t \\={}& \Gamma_T D + \sum_{t\in \mathcal{T}(D)} R_t.\\
    \end{split}
  \end{equation}
  For $t\in \mathcal{T}(D)$, by selection criterion~\eqref{eq:SCB-PW-select-criteria}, 
  \begin{equation}\nonumber
  \begin{split}
    &R_t = \mu(X_t^{*\T}\theta_t) - \mu(X_t^{\T}\theta_t) \\
    \leq {}&\mu(X_t^{\T}\thetat_t) - \mu(X_t^{\T}\thetah_t) + \mu(X_t^{\T}\thetah_t) -\mu(X_t^{\T}\theta_t) \\
    \leq {}&\alpha(X_t,\thetat_t,\thetah_t)\abs{X_t^\T\sbr{\thetat_t - \thetah_t}} + \alpha(X_t,\theta_t,\thetah_t)\abs{X_t^\T\sbr{\theta_t - \thetah_t}}\\
    \leq {}&\sqrt{1+2S}\\&\cdot\bigg(\alpha(X_t,\thetat_t,\thetah_t)\norm{X_t}_{G_t^{-1}(\thetat_t,\thetah_t)} \norm{g_t(\thetat_t) - g_t(\thetah_t)}_{H_t^{-1}(\thetat_t)} \\
    & +\alpha(X_t,\theta_t,\thetah_t)\norm{X_t}_{G_t^{-1}(\theta_t,\thetah_t)} \norm{g_t(\theta_t) - g_t(\thetah_t)}_{H_t^{-1}(\theta_t)}\bigg).
  \end{split}
  \end{equation}
  where $\alpha(\x, \theta_1, \theta_2) = \int_{0}^1 \dmu(v\x^\T\theta_2+(1-v)x^\T\theta_1)\diff{v}$, and the last second inequality comes from the mean value theorem $
  \mu(\x^\T\theta_1) - \mu(\x^\T \theta_2) = \alpha(\x, \theta_1, \theta_2)(\x^\T\theta_1 - \x^\T \theta_2)$. Since that $\thetat_t \in \C_t(\delta)$ and with probability at least $1-\delta$, $\forall t\in [T], \theta_t \in \C_t(\delta)$, and by union bound, the following dynamic regret bound hold with probability at least $1-\delta$,
  \begin{equation}\nonumber
    \begin{split}
      \sum_{t\in \mathcal{T}(D)} R_t 
      \leq {}&\sum_{t\in \mathcal{T}(D)}\sqrt{1+2S} \big(\alpha(X_t,\thetat_t,\thetah_t)\norm{X_t}_{G_t^{-1}(\thetat_t,\thetah_t)}\rho_t\\{}& + \alpha(X_t,\theta_t,\thetah_t)\norm{X_t}_{G_t^{-1}(\theta_t,\thetah_t)} \rho_t\big)\\
      \leq {}&\sqrt{1+2S} \rho_T  \Bigg(\sum_{t\in \mathcal{T}(D)}\alpha(X_t,\thetat_t,\thetah_t)\norm{X_t}_{G_t^{-1}(\thetat_t,\thetah_t)}\\{}&+ \sum_{t\in \mathcal{T}(D)}\alpha(X_t,\theta_t,\thetah_t)\norm{X_t}_{G_t^{-1}(\theta_t,\thetah_t)} \Bigg).\\
    \end{split}
  \end{equation}
    Now we try to derive the upper bound for term $\sum_{t\in \mathcal{T}(D)}\alpha(X_t,\thetat_t,\thetah_t)\norm{X_t}_{G_t^{-1}(\thetat_t,\thetah_t)}$. 
    
  Based on the definition of $g_t$~\eqref{eq:GLB-gt}, we have
  \begin{equation*}
    \begin{split}
        {}&g_t(\theta_1) - g_t(\theta_2)\\ ={}&\lambda c_\mu (\theta_1 - \theta_2) + \sum_{s=1}^{t-1}\gamma^{t-s-1}(\mu(X_s^\T \theta_1)-\mu(X_s^\T \theta_2))X_s\\
        ={}& \lambda c_\mu (\theta_1 - \theta_2) + \sum_{s=1}^{t-1}\gamma^{t-s-1}\alpha(X_s, \theta_1, \theta_2) X_s^\T X_s (\theta_1-\theta_2)\\
        ={}& \sbr{\lambda c_\mu + \sum_{s=1}^{t-1}\gamma^{t-s-1}\alpha(X_s, \theta_1, \theta_2) X_s^\T X_s}(\theta_1-\theta_2).
    \end{split}
  \end{equation*}
  Then based on the definition of $G_t$~\eqref{eq:SCB-gt-mvt}, we know $G_t(\theta_1, \theta_2) = \lambda c_\mu + \sum_{s=1}^{t-1}\gamma^{t-s-1}\alpha(X_s, \theta_1, \theta_2) X_s^\T X_s.$ which means $G_t(\thetat_t,\thetah_t) = \lambda c_\mu I_d + \sum_{s=1}^{t-1}\gamma^{t-s-1}\alpha(X_s,\thetat_t,\thetah_t)X_sX_s^\T$, if we let $\Xt_s = \sqrt{\alpha(X_s,\thetat_t,\thetah_t)}X_s$, then
    \begin{align*}\nonumber
        {}&\sum_{t\in \mathcal{T}(D)} \alpha(X_t,\thetat_t,\thetah_t)\norm{X_t}_{G_t^{-1}(\thetat_t,\thetah_t)} \\\leq{}& \sqrt{\sum_{t=1}^T\alpha(X_t,\thetat_t,\thetah_t)}\sqrt{\sum_{t=1}^T \alpha(X_t,\thetat_t,\thetah_t) \norm{X_t}_{G_t^{-1}(\thetat_t,\thetah_t)}^2 } \\
        \leq{}& \sqrt{k_\mu T}\sqrt{\sum_{t=1}^T \norm{\Xt_t}_{G_t^{-1}(\thetat_t,\thetah_t)}^2 }.
      \end{align*}

  Then for the term $\sqrt{\sum_{t=1}^T \|\Xt_t\|_{G_t^{-1}(\thetat_t,\thetah_t)}^2 }$, we can directly use the Lemma~\ref{lemma:potential-lemma} to bound it, let $C_T^{\text{PWSCB}} \define \sqrt{2k_\mu\max\{1,L^2k_\mu/(\lambda c_\mu)\}dT}$ we have
  \begin{equation}\nonumber
  \begin{split}
    {}&\sqrt{k_\mu T}\sqrt{\sum_{t=1}^T \norm{\Xt_t}_{G_t^{-1}(\thetat_t,\thetah_t)}^2 }\\\leq {}& C_T^{\text{PWSCB}}\sqrt{ T\log\frac{1}{\gamma}+\log\sbr{1+ \frac{L^2k_\mu}{\lambda c_\mu d(1-\gamma)}}}.
  \end{split}
  \end{equation}
  We can bound term $\sum_{t\in \mathcal{T}(D)}\alpha(X_t,\theta_t,\thetah_t)\norm{X_t}_{G_t^{-1}(\theta_t,\thetah_t)} $ in the same way and get, 
  \begin{equation}\nonumber
    \begin{split}
      {}&\sum_{t\in \mathcal{T}(D)}\alpha(X_t,\theta_t,\thetah_t)\norm{X_t}_{G_t^{-1}(\theta_t,\thetah_t)} \\\leq {}& C_T^{\text{PWSCB}}\sqrt{ T\log\frac{1}{\gamma}+\log\sbr{1+ \frac{L^2k_\mu}{\lambda c_\mu d(1-\gamma)}}}.
    \end{split}
    \end{equation}

  Combine these two bounds and let $\delta = 1/T$, we have the following regret bound with probability at least $1-1/T$,
  \begin{equation}\nonumber
    \begin{split}
      &\DReg_T \leq \Gamma_T D\\&{} + 2\sqrt{1+2S} \rho_T C_T^{\text{PWSCB}}\sqrt{ T\log\frac{1}{\gamma}+\log\sbr{1+ \frac{L^2k_\mu}{\lambda c_\mu d(1-\gamma)}}},
    \end{split}
  \end{equation}
  where $\rho_t = \frac{2 L^2S k_\mu}{\sqrt{\lambda c_\mu}} \frac{\gamma^D}{1-\gamma} + \frac{Lm}{\sqrt{\lambda c_\mu}}\frac{\gamma^D}{1-\gamma} + \betabr_t$ and $\betabr_t = \frac{d m }{\sqrt{\lambda c_\mu}}\log \sbr{1 + \frac{ L^2k_\mu(1-\gamma^{2D})}{\lambda c_\mu d(1-\gamma)}}+\frac{\sqrt{\lambda c_\mu}}{2 m }+\frac{2 m }{\sqrt{\lambda c_\mu}}\log\sbr{T} +\frac{2 m }{\sqrt{\lambda c_\mu}} d \log (2) + \sqrt{\lambda c_\mu}S$.  Since there is a $T \sqrt{\log (1/\gamma)}$ term in the regret bound, which means that we cannot let $\gamma$ close to $0$, so we set $\gamma \geq 1/2$, then we have $\log(1/\gamma) \leq 2\log(2) (1-\gamma)$. Then, we set $D = \log(T)/\log(1/\gamma)$, noticing that $0< 1/\gamma -1<1$ and using $\log(1+x)\geq x/2$ for $0<x<1$, we have 
  \begin{equation}\nonumber
    \log\frac{1}{\gamma} = \log(1+1/\gamma -1) \geq \frac{1-\gamma}{2\gamma}.
\end{equation}
  Therefore, we have $D\leq \frac{2\gamma\log(T)}{1-\gamma}$. Then, ignoring logarithmic factors in time horizon $T$, and let $\lambda = d\log(T)/c_\mu$, we finally obtain that,
  \begin{equation}\nonumber
  \begin{split}
    {}&\DReg_T\\\leq{}& \Ot\sbr{\frac{1}{1-\gamma}\Gamma_T + \sbr{\frac{1}{\sqrt{d}}\frac{1}{1-\gamma}\frac{1}{T}+\sqrt{d}}\sqrt{d(1-\gamma)}T}\\
      \leq{}& \Ot\sbr{\frac{1}{1-\gamma}\Gamma_T + \frac{1}{\sqrt{1-\gamma}} + d\sqrt{(1-\gamma)}T}.
  \end{split}
  \end{equation}
  When $\Gamma_T< d/\sqrt{T}$ (which corresponds a small amount of non-stationarity), we simply set $\gamma = 1-1/T$ and achieve an $\Ot(d\sqrt{T})$ regret bound.  Besides, when coming to the non-degenerate case of $\Gamma_T > d/\sqrt{T}$, We set the discounted factor optimally as $1-\gamma = \sbr{\Gamma_T/(dT)}^{\sfrac{2}{3}}$ and attain an $\Ot(d^{\sfrac{2}{3}}\Gamma_T^{\sfrac{1}{3}}T^{\sfrac{2}{3}})$ regret, which completes the proof.
\end{proof}
\section{Analysis of \LMDPweightours}
\label{app:LMDP}
\subsection{Proof of Lemma~\ref{lemma:LMDP-r-ucb}}
\label{app:LMDP-lemma-r-ucb}
\begin{proof}
    Fix $h\in[H]$, based on the reward model assumption~\eqref{eq:LMDP-model} and the estimator~\eqref{eq:LMDP-estimate-r}, the estimation error of reward estimation can be decomposed as
    \begin{align*}
        {}&\th_{h}^k -\theta_{h}^k \\ ={}& \sbr{\Lambda_{h}^{k-1}}^{-1}\sbr{\sum_{j=1}^{k-1}w_{k-1,j}r_h^j(s_h^j,a_h^j)\phi(s_h^j,a_h^j)} - \theta_{h}^k\\
        ={}& \sbr{\Lambda_{h}^{k-1}}^{-1}\sbr{\sum_{j=1}^{k-1}w_{k-1,j}\phi(s_h^j,a_h^j)^\T\theta_h^j\phi(s_h^j,a_h^j)} - \theta_{h}^k\\
        ={}& \sbr{\Lambda_{h}^{k-1}}^{-1}\sbr{\sum_{j=1}^{k-1}w_{k-1,j}\phi(s_h^j,a_h^j)\phi(s_h^j,a_h^j)^\T\theta_h^j}\\
        &- \sbr{\Lambda_{h}^{k-1}}^{-1}\sbr{\lamt I_d + \sum_{j=1}^{k-1}w_{k-1,j}\phi(s_h^j,a_h^j)\phi(s_h^j,a_h^j)^\T}\theta_{h}^k\\
        ={}& \underbrace{\sbr{\Lambda_{h}^{k-1}}^{-1}\sbr{\sum_{j=1}^{k-1}w_{k-1,j}\phi(s_h^j,a_h^j)\phi(s_h^j,a_h^j)^\T\sbr{\theta_h^j-\theta_h^k}}}_{\bias}\\&\qquad\qquad\qquad - \underbrace{\sbr{\Lambda_{h}^{k-1}}^{-1}\lamt \theta_{h}^k}_{\variance}.
\end{align*}
Then, by the Cauchy-Schwarz inequality, we know that for any $s\in\S, a\in \A$, 
\begin{equation}\label{eq:LMDP-decompose-r}
    \begin{split}
        \abs{\phi(s,a)^\T\sbr{\th_{h}^k -\theta_{h}^k}} &= \norm{\phi(s,a)}_{\sbr{\Lambda_{h}^{k-1}}^{-1}}\sbr{A_h^k+B_h^k},
\end{split}
\end{equation}  
where
\begin{equation*}
    \begin{split}
        A_h^k &= \norm{\sum_{j=1}^{k-1}w_{k-1,j}\phi(s_h^j,a_h^j)\phi(s_h^j,a_h^j)^\T\sbr{\theta_h^j-\theta_h^k}}_{\sbr{\Lambda_{h}^{k-1}}^{-1}},\\ B_h^k &= \norm{\lamt \theta_{h}^k}_{\sbr{\Lambda_{h}^{k-1}}^{-1}}.
    \end{split}
\end{equation*}
The above two terms can be bounded separately, 

\parag{Term $A_h^k$.~} The first step is to extract the variations of the parameter $\theta_h^k$ as follows,
    \begin{align*}
            {}&\norm{\sum_{j=1}^{k-1}w_{k-1,j}\phi(s_h^j,a_h^j)\phi(s_h^j,a_h^j)^\T\sbr{\theta_h^j-\theta_h^k}}_{\sbr{\Lambda_{h}^{k-1}}^{-1}} \\ ={}& \norm{\sum_{j=1}^{k-1}w_{k-1,j}\phi(s_h^j,a_h^j)\phi(s_h^j,a_h^j)^\T\sum_{p=j}^{k-1}\sbr{\theta_h^p-\theta_h^{p+1}}}_{\sbr{\Lambda_{h}^{k-1}}^{-1}}\\
            ={}&\norm{\sum_{p=1}^{k-1}\sum_{j=1}^{p}w_{k-1,j}\phi(s_h^j,a_h^j)\phi(s_h^j,a_h^j)^\T\sbr{\theta_h^p-\theta_h^{p+1}}}_{\sbr{\Lambda_{h}^{k-1}}^{-1}}\\
            \leq{}&\sum_{p=1}^{k-1}\norm{\sum_{j=1}^{p}w_{k-1,j}\phi(s_h^j,a_h^j)\norm{\phi(s_h^j,a_h^j)}_2\norm{\theta_h^p-\theta_h^{p+1}}_2}_{\sbr{\Lambda_{h}^{k-1}}^{-1}}\\
            \leq{}&L_{\phi}\sum_{p=1}^{k-1}\sum_{j=1}^{p}w_{k-1,j}\norm{\phi(s_h^j,a_h^j)}_{\sbr{\Lambda_{h}^{k-1}}^{-1}}\norm{\theta_h^p-\theta_h^{p+1}}_2,
    \end{align*}
    and term $\sum_{j=1}^{p}w_{k-1,j}\norm{\phi(s_h^j,a_h^j)}_{\sbr{\Lambda_{h}^{k-1}}^{-1}}$ can be able to further derive an expression about weight $w_{k-1,j}$ as follows,
    \begin{equation*}
        \begin{split}
            {}&\sum_{j=1}^{p}w_{k-1,j}\norm{\phi(s_h^j,a_h^j)}_{\sbr{\Lambda_{h}^{k-1}}^{-1}} \\ \leq{}& \sqrt{\sum_{j=1}^{p}w_{k-1,j}}\sqrt{\sum_{j=1}^{p}w_{k-1,j}\norm{\phi(s_h^j,a_h^j)}_{\sbr{\Lambda_{h}^{k-1}}^{-1}}^2}\\
            \leq{}& \sqrt{d}\sqrt{\sum_{j=1}^{p}w_{k-1,j}}.
        \end{split}
    \end{equation*}
    In above, the first step holds by the Cauchy-Schwarz inequality. Besides, the last step follows the Lemma~\ref{lemma:d} by letting $X_j = \sqrt{w_{k-1,j}}\phi(s_h^j,a_h^j)$ and $U_{k-1} = \Lambda_{h}^{k-1}$, which means for Term 1 we have
    \begin{align}
            {}&\norm{\sum_{j=1}^{k-1}w_{k-1,j}\phi(s_h^j,a_h^j)\phi(s_h^j,a_h^j)^\T\sbr{\theta_h^j-\theta_h^k}}_{\sbr{\Lambda_{h}^{k-1}}^{-1}}\nonumber\\ \leq{}& L_{\phi}\sqrt{d}\sum_{p=1}^{k-1}\sqrt{\sum_{j=1}^{p}w_{k-1,j}}\norm{\theta_h^p-\theta_h^{p+1}}_2,\label{eq:LMDP-term1-r}
    \end{align}

    \parag{Term $B_h^k$.~} 
    \begin{equation}\label{eq:LMDP-term2-r}
        \begin{split}
            \norm{\lamt \theta_{h}^k}_{\sbr{\Lambda_{h}^{k-1}}^{-1}} \leq \frac{\lamt}{\sqrt{\lambda_{\min}(\Lambda_{h}^{k-1})}}\norm{\theta_{h}^k}_2\leq \sqrt{\lamt}S_\theta.
        \end{split}
    \end{equation}
    Plug Eq~\eqref{eq:LMDP-term1-r} and Eq~\eqref{eq:LMDP-term2-r} into Eq~\eqref{eq:LMDP-decompose-r} and we have 
    \begin{equation*}
        \begin{split}
            {}&\norm{\th_{h}^k -\theta_{h}^k}_{\Lambda_{h}^{k-1}} \\\leq{}& L_{\phi}\sqrt{d}\sum_{p=1}^{k-1}\sqrt{\sum_{j=1}^{p}w_{k-1,j}}\norm{\theta_h^p-\theta_h^{p+1}}_2 + \sqrt{\lamt}S_\theta,
        \end{split}
    \end{equation*}
    further we have 
    \begin{equation*}
        \begin{split}
            {}&\abs{\phi(s,a)^\T\sbr{\th_{h}^k -\theta_{h}^k}} \\\leq{}& \norm{\phi(s,a)}_{\sbr{\Lambda_{h}^{k-1}}^{-1}}\\{}&\cdot\Bigg(L_{\phi}\sqrt{d}\sum_{p=1}^{k-1}\sqrt{\sum_{j=1}^{p}w_{k-1,j}}\norm{\theta_h^p-\theta_h^{p+1}}_2 + \sqrt{\lamt}S_\theta\Bigg)\\
            \leq{}& L_{\phi}^2\sqrt{\frac{d}{\lamt}}\sum_{p=1}^{k-1}\sqrt{\sum_{j=1}^{p}w_{k-1,j}}\norm{\theta_h^p-\theta_h^{p+1}}_2\\{}&\qquad\qquad\qquad\qquad\qquad\qquad + \beta_\theta\norm{\phi(s,a)}_{\sbr{\Lambda_{h}^{k-1}}^{-1}},
        \end{split}
    \end{equation*}    
    where $\norm{\phi(s,a)}_{\sbr{\Lambda_{h}^{k-1}}^{-1}}\leq \norm{\phi(s,a)}_{2}/\sqrt{\lambda_\theta}$ and $\beta_\theta \define \sqrt{\lamt}S_\theta$, which completes the proof.
\end{proof}

\subsection{Proof of Lemma~\ref{lemma:LMDP-v-ucb}}
\label{app:LMDP-lemma-v-ucb}
\begin{proof}
Fix $h\in[H]$, based on the model assumption~\eqref{eq:LMDP-model} and estimator~\eqref{eq:LMDP-estimate-v}, we have 
    \begin{align}
        {}&\wh_{h}^k = \sbr{\Sigma_{h}^{k-1}}^{-1}\sbr{\sum_{j=1}^{k-1}\alpha_{k-1,j}V_{h+1}^j(s_{h+1}^j)\psi_{h+1}^j\sbr{s_h^j, a_h^j}}\nonumber\\
        {}&=\sbr{\Sigma_{h}^{k-1}}^{-1}\Bigg(\sum_{j=1}^{k-1}\alpha_{k-1,j}\sbr{\psi_{h+1}^j\sbr{s_h^j, a_h^j}^\T\w_h^j +\eta_{h+1}^j}\nonumber\\{}&\hspace{13em}\cdot\psi_{h+1}^j\sbr{s_h^j, a_h^j}\Bigg),\label{eq:LMDP-wh}
    \end{align}
    where we define the noise as $\eta_{h+1}^j \define V_{h+1}^j\sbr{s_{h+1}^j} - \mbr{\mathbb{P}_h^j V_{h+1}^j}\sbr{s_h^j, a_h^j}$, and we have 
    \begin{equation}\label{eq:LMDP-wh*}
        \begin{split}
            &\w_h^k = \sbr{\Sigma_{h}^{k-1}}^{-1} \Bigg(\lamw I_d\\{}& + \sum_{j=1}^{k-1}\alpha_{k-1,j}\psi_{h+1}^j\sbr{s_h^j, a_h^j}\psi_{h+1}^j\sbr{s_h^j, a_h^j}^\T\Bigg) \w_h^k,
        \end{split}
    \end{equation}
    combine Eq~\eqref{eq:LMDP-wh} and Eq~\eqref{eq:LMDP-wh*} and we have the estimation error 
    \begin{equation*}
        \begin{split}
            {}&\wh_{h}^k - \w_h^k \\={}& \underbrace{
                \substack{
                \displaystyle
                \sbr{\Sigma_{h}^{k-1}}^{-1}
                \Bigg(
                \sum_{j=1}^{k-1}
                \alpha_{k-1,j}
                \psi_{h+1}^j\sbr{s_h^j, a_h^j}
                \psi_{h+1}^j\sbr{s_h^j, a_h^j}^\T \\\displaystyle
                \hspace{13em}
                \cdot \sbr{\w_h^j - \w_h^k}
                \Bigg)
                }
                }_{\bias}\\
            {}&+\underbrace{\sbr{\Sigma_{h}^{k-1}}^{-1}\sbr{\sum_{j=1}^{k-1}\alpha_{k-1,j}\eta_{h+1}^j \psi_{h+1}^j\sbr{s_h^j, a_h^j}-\lamw \w_h^k}}_{\variance}.
        \end{split}
    \end{equation*}
    Then, by the Cauchy-Schwarz inequality, we know that for any $s\in\S, a\in \A$, 
    \begin{equation}\label{eq:LMDP-decomposition-v}
        \begin{split}
            {}&\abs{\psi_{h+1}^k\sbr{s, a}^\T \sbr{\wh_{h}^k - \w_h^k}} \\\leq{}& \norm{\psi_{h+1}^k\sbr{s, a}}_{\sbr{\Sigma_{h}^{k-1}}^{-1}}\sbr{C_h^k+D_h^k},
        \end{split}
    \end{equation}  
    where 
    \begin{equation*}
        \begin{split}
            C_h^k &= \Bigg\|\sum_{j=1}^{k-1}\alpha_{k-1,j}\\{}&\hspace{1em}\cdot\psi_{h+1}^j\sbr{s_h^j, a_h^j}\psi_{h+1}^j\sbr{s_h^j, a_h^j}^\T\sbr{\w_h^j-\w_h^k}\Bigg\|_{\sbr{\Sigma_{h}^{k-1}}^{-1}}\\
            D_h^k &= \norm{\sum_{j=1}^{k-1}\alpha_{k-1,j}\eta_{h+1}^j \psi_{h+1}^j\sbr{s_h^j, a_h^j}-\lamw \w_h^k}_{\sbr{\Sigma_{h}^{k-1}}^{-1}}.
        \end{split}
    \end{equation*}
    The above two terms can be bounded separately, as summarized in the following two lemmas,
    \begin{myLemma}\label{lemma:LMDP-C}
        For any $k\in[K]$, we have
        \begin{equation*}
            \begin{split}
                C_h^k\leq HL_{\psi}\sqrt{d}\sum_{p=1}^{k-1}\sqrt{\sum_{j=1}^{p}\alpha_{k-1,j}}\norm{\w_h^p-\w_h^{p+1}}_2.
            \end{split}
        \end{equation*}
    \end{myLemma}
    \begin{myLemma}\label{lemma:LMDP-D}
        If $\forall k,j\in[K], \alpha_{k-1,j}\leq 1$, for any $\delta \in(0,1)$, with probability at least $1-\delta$, the following holds for all $k\in[K]$,
        \begin{equation*}
            \begin{split}
                D_h^k  \leq{}& H \sqrt{\frac{1}{2}\log \frac{1}{\delta}+\frac{d}{4}\log\left(1+\frac{H^2 L_\psi^2\sum_{j=1}^{k-1}\alpha_{k-1,j}}{\lamw d}\right)}\\{}&\hspace{13em}+\sqrt{\lamw}S_\w.
            \end{split}
        \end{equation*}
    \end{myLemma}
    Based on the inequality~\eqref{eq:LMDP-decomposition-v}, Lemma~\ref{lemma:LMDP-C}, Lemma~\ref{lemma:LMDP-D}, and $\norm{\psi_{h+1}^k\sbr{s, a}}_{\sbr{\Sigma_{h}^{k-1}}^{-1}}\leq \norm{\psi_{h+1}^k\sbr{s, a}}_{2}/\sqrt{\lambda_\w}\leq HL_{\psi}/\sqrt{\lambda_\w}$, with probability at least $1-\delta$, the following holds for all $k\in[K]$,
    \begin{equation*}
        \begin{split}
            {}&\abs{\psi_{h+1}^k\sbr{s, a}^\T \sbr{\wh_{h}^k - \w_h^k}} \\\leq{}& \Gamma_{h,\w}^{k-1} + \beta_{\w}^{k-1}\norm{\psi_{h+1}^k\sbr{s, a}}_{\sbr{\Sigma_{h}^{k-1}}^{-1}},
        \end{split}
    \end{equation*}
    where 
    \begin{equation*}
        \begin{split}
            \Gamma_{h,\w}^{k-1}\define{}&  H^2L_{\psi}^2\sqrt{\frac{d}{\lamw}}\sum_{p=1}^{k-1}\sqrt{\sum_{j=1}^{p}\alpha_{k-1,j}}\norm{\w_h^p-\w_h^{p+1}}_2\\
            \beta_{\w}^{k-1} \define{}& H \sqrt{\frac{1}{2}\log \frac{1}{\delta}+\frac{d}{4}\log\left(1+\frac{H^2 L_\psi^2\sum_{j=1}^{k-1}\alpha_{k-1,j}}{\lamw d}\right)}\\
            {}&\qquad\qquad\qquad\qquad\qquad\qquad\qquad +\sqrt{\lamw}S_\w,
        \end{split}
    \end{equation*}
    which completes the proof.
\end{proof}

\begin{proof}[Proof of Lemma~\ref{lemma:LMDP-C}]
    The first step is to extract the variations of the parameter $\w_h^k$ as follows,
    \begin{align*}
            {}&\Bigg\|\sum_{j=1}^{k-1}\alpha_{k-1,j}\psi_{h+1}^j\sbr{s_h^j, a_h^j}\\{}&\hspace{6em}\cdot\psi_{h+1}^j\sbr{s_h^j, a_h^j}^\T\sbr{\w_h^j-\w_h^k}\Bigg\|_{\sbr{\Sigma_{h}^{k-1}}^{-1}}\\
            ={}& \Bigg\|\sum_{j=1}^{k-1}\alpha_{k-1,j}\psi_{h+1}^j\sbr{s_h^j, a_h^j}\\{}&\hspace{4em}\cdot\psi_{h+1}^j\sbr{s_h^j, a_h^j}^\T\sum_{p=j}^{k-1}\sbr{\w_h^p-\w_h^{p+1}}\Bigg\|_{\sbr{\Sigma_{h}^{k-1}}^{-1}}\\
            ={}& \Bigg\|\sum_{p=1}^{k-1}\sum_{j=1}^{p}\alpha_{k-1,j}\psi_{h+1}^j\sbr{s_h^j, a_h^j}\\{}&\hspace{6em}\cdot\psi_{h+1}^j\sbr{s_h^j, a_h^j}^\T\sbr{\w_h^p-\w_h^{p+1}}\Bigg\|_{\sbr{\Sigma_{h}^{k-1}}^{-1}}\\
            \leq{}& \sum_{p=1}^{k-1}\Bigg\|\sum_{j=1}^{p}\alpha_{k-1,j}\psi_{h+1}^j\sbr{s_h^j, a_h^j}\\{}&\hspace{6em}\cdot\psi_{h+1}^j\sbr{s_h^j, a_h^j}^\T\sbr{\w_h^p-\w_h^{p+1}}\Bigg\|_{\sbr{\Sigma_{h}^{k-1}}^{-1}}\\
            \leq{}& HL_{\psi}\sum_{p=1}^{k-1}\sum_{j=1}^{p}\alpha_{k-1,j}\norm{\psi_{h+1}^j\sbr{s_h^j, a_h^j}}_{\sbr{\Sigma_{h}^{k-1}}^{-1}}\\{}&\hspace{15em}\cdot\norm{\w_h^p-\w_h^{p+1}}_2,
    \end{align*}
    and term $\sum_{j=1}^{p}\alpha_{k-1,j}\norm{\psi_{h+1}^j\sbr{s_h^j, a_h^j}}_{\sbr{\Sigma_{h}^{k-1}}^{-1}}$ can be able to further derive an expression about weight $\alpha_{k-1,j}$ as follows,
    \begin{equation*}
        \begin{split}
            {}&\sum_{j=1}^{p}\alpha_{k-1,j}\norm{\psi_{h+1}^j\sbr{s_h^j, a_h^j}}_{\sbr{\Sigma_{h}^{k-1}}^{-1}} \\\leq{}& \sqrt{\sum_{j=1}^{p}\alpha_{k-1,j}}\sqrt{\sum_{j=1}^{p}\alpha_{k-1,j}\norm{\psi_{h+1}^j\sbr{s_h^j, a_h^j}}_{\sbr{\Sigma_{h}^{k-1}}^{-1}}^2}\\
            \leq{}& \sqrt{d}\sqrt{\sum_{j=1}^{p}\alpha_{k-1,j}}.
        \end{split}
    \end{equation*}
    In above, the second last step holds by the Cauchy-Schwarz inequality. Besides, the last step follows the Lemma~\ref{lemma:d} by letting $X_j = \sqrt{\alpha_{k-1,j}}\psi_{h+1}^j\sbr{s_h^j, a_h^j}$ and $U_{k-1} = \Sigma_{h}^{k-1}$. Hence we complete the proof.
\end{proof}

\begin{proof}[Proof of Lemma~\ref{lemma:LMDP-D}]
    \begin{equation*}
        \begin{split}
            {}&\norm{\sum_{j=1}^{k-1}\alpha_{k-1,j}\eta_{h+1}^j \psi_{h+1}^j\sbr{s_h^j, a_h^j}-\lamw \w_h^k}_{\sbr{\Sigma_{h}^{k-1}}^{-1}} \\\leq{}& \norm{\sum_{j=1}^{k-1}\alpha_{k-1,j}\eta_{h+1}^j \psi_{h+1}^j\sbr{s_h^j, a_h^j}}_{\sbr{\Sigma_{h}^{k-1}}^{-1}}\\{}&\hspace{13em}+ \norm{\lamw \w_h^k}_{\sbr{\Sigma_{h}^{k-1}}^{-1}}\\
            \leq{}& \norm{\sum_{j=1}^{k-1}\alpha_{k-1,j}\eta_{h+1}^j \psi_{h+1}^j\sbr{s_h^j, a_h^j}}_{\sbr{\Sigma_{h}^{k-1}}^{-1}}+ \sqrt{\lamw}S_\w.
        \end{split}
    \end{equation*}
     Let $\etat_{h+1}^j \define \sqrt{\alpha_{k-1,j}}\eta_{h+1}^j$ and $X_j \define \sqrt{\alpha_{k-1,j}}\psi_{h+1}^j\sbr{s_h^j, a_h^j}$, then we have 
    notice that since the reward $r\in [0,1]$, and $\alpha_{k-1,j} \leq 1$, the noise $\etat_{h+1}^j$ is bounded by:
    \begin{equation*}
        \begin{split}
            \etat_{h+1}^j = \sqrt{\alpha_{k-1,j}}\sbr{V_{h+1}^j\sbr{s_{h+1}^j} - \mbr{\mathbb{P}_h V_{h+1}^j}\sbr{s_h^j, a_h^j}} \leq H,
        \end{split}
    \end{equation*}
    based on Lemma~\ref{lemma:Hoeffding}, we find that the noise $\etat_{h+1}^j$ is $\frac{H}{2}$-sub-Gaussian. Then, by Theorem~\ref{thm:snc-AY}, we have with probability at least $1-\delta$, the following holds for all $k\in[K]$.
    \begin{equation*}
        \begin{split}
            {}&\norm{\sum_{j=1}^{k-1}\etat_{h+1}^j X_j}_{\sbr{\Sigma_{h}^{k-1}}^{-1}} \\\leq{}& \sqrt{\frac{H^2}{2}\log\left(\frac{\det(\Sigma_{h}^{k-1})^{\frac{1}{2}}\det(\Sigma_{0,h})^{-\frac{1}{2}}}{\delta}\right)}
        \end{split}
    \end{equation*}
    where
    \begin{equation*}
        \begin{split}
            \det(\Sigma_{h}^{k-1}) \leq{}& \sbr{\frac{\trace(\Sigma_{h}^{k-1})}{d}}^d\\ ={}& \sbr{\frac{d\lamw + \sum_{j=1}^{k-1}\norm{\alpha_{k-1,j}\psi_{h+1}^j\sbr{s_h^j, a_h^j}}_{2}^2}{d}}^d\\ ={}& \sbr{\frac{d\lambda_\w+H^2 L_\psi^2\sum_{j=1}^{k-1}\alpha_{k-1,j}}{d}}^d\\
            \det(\Sigma_{0,h}) \leq{}& \lamw^d,
        \end{split}
    \end{equation*}
    so we have 
    \begin{equation*}
        \begin{split}
            {}&\norm{\sum_{j=1}^{k-1}\alpha_{k-1,j}\eta_{h+1}^j \psi_{h+1}^j\sbr{s_h^j, a_h^j}}_{\sbr{\Sigma_{h}^{k-1}}^{-1}}\\\leq{}& H \sqrt{\frac{1}{2}\log \frac{1}{\delta}+\frac{d}{4}\log\left(1+\frac{H^2 L_\psi^2\sum_{j=1}^{k-1}\alpha_{k-1,j}}{\lamw d}\right)}.
        \end{split}
    \end{equation*}
    which completes the proof.
\end{proof}

\subsection{Proof of Theorem~\ref{thm:LMDP-regret-bound}}
\label{app:proof-LMDP-regret}
\begin{proof}
    To prove the theorem, we first introduce the following lemma
    \begin{myLemma}
        \label{lemma:LMDP-model-prediction-error}
        We define the model prediction error as 
        \begin{align}\label{eq:LMDP-model-prediction-error}
            E_h^k(s,a) = r_h^k(s,a) + \P_h^kV_{h+1}^k(s,a) - Q_h^k(s,a),
        \end{align}
        then with probability at least $1-2\delta$, the following holds for all $k\in[K]$, $h\in[H]$ and $\forall s \in\S,a\in\A$,
        \begin{align*}
            {}&-2\beta_\theta\norm{\phi(s,a)}_{\sbr{\Lambda_{h}^{k-1}}^{-1}} - 2\beta_{\w}^{k-1}\norm{\psi_{h+1}^k\sbr{s, a}}_{\sbr{\Sigma_{h}^{k-1}}^{-1}}\\{}& - \Gamma_{h,\theta}^{k-1}-\Gamma_{h,\w}^{k-1} \leq E_h^k(s,a) \leq \Gamma_{h,\theta}^{k-1}+\Gamma_{h,\w}^{k-1}.
        \end{align*}
    \end{myLemma}
    We can further connect the dynamic regret to the model prediction error, by the following Lemma.
    \begin{myLemma}
        \label{lemma:LMDP-regret-decomposition}
        For the policies $\{\pi_h^k\}_{h\in[H],k\in[K]}$ with $a_h^k=\argmax_{a\in\A}Q_{h}^k(s_h^k,a)$, and the optimal policies and $\delta \in (0,1)$, we have the following decompostion holds with probability at least $1-2\delta$,
        \begin{equation*}
            \begin{split}
                \DReg_T \leq{}& \sum_{k=1}^K\sum_{h=1}^H\sbr{\E_{\pi_{*,h}^k}\mbr{E_h^k(s_h^k,a_h^k)}-E_h^k(s_h^k,a_h^k)}\\{}&\qquad\qquad \qquad \qquad \qquad   + 4H\sqrt{2T\log(1/\delta)}.
        \end{split}
        \end{equation*}
    \end{myLemma}
    Based on Lemma~\ref{lemma:LMDP-regret-decomposition} and notice that $\forall s\in\S, a\in\A$, $\abs{E_h^k(s,a)}\leq 2H$, we have
    \begin{equation*}
        \begin{split}
            {}&\E_{\pi_{*,h}^k}\mbr{E_h^k(s_h^k,a_h^k)}-E_h^k(s_h^k,a_h^k)\\
            \leq{}& \min\big\{4H,2\Gamma_{h,\theta}^{k-1} + 2\Gamma_{h,\w}^{k-1} + 2\beta_\theta\norm{\phi(s_h^k,a_h^k)}_{\sbr{\Lambda_{h}^{k-1}}^{-1}}\\{}&+ 2\beta_{\w}^{k-1}\norm{\psi_{h+1}^k\sbr{s_h^k,a_h^k}}_{\sbr{\Sigma_{h}^{k-1}}^{-1}}\big\}\\
            \leq{}& 2\Gamma_{h,\theta}^{k-1} + 2\Gamma_{h,\w}^{k-1} +\min\bbr{4H, 2\beta_\theta\norm{\phi(s_h^k,a_h^k)}_{\sbr{\Lambda_{h}^{k-1}}^{-1}}}\\{}&+\min\bbr{4H, 2\beta_{\w}^{k-1}\norm{\psi_{h+1}^k\sbr{s_h^k,a_h^k}}_{\sbr{\Sigma_{h}^{k-1}}^{-1}}}\\
            \leq{}& 2\Gamma_{h,\theta}^{k-1} + 2\Gamma_{h,\w}^{k-1} +4H\beta_\theta\min\bbr{1, \norm{\phi(s_h^k,a_h^k)}_{\sbr{\Lambda_{h}^{k-1}}^{-1}}}\\{}&+4\beta_{\w}^{k-1}\min\bbr{1, \norm{\psi_{h+1}^k\sbr{s_h^k,a_h^k}}_{\sbr{\Sigma_{h}^{k-1}}^{-1}}},
    \end{split}
    \end{equation*}
    the last inequality comes from that $\beta_{\theta}>1, \beta_{\w}^{k-1}\geq H$.
    so we have with probability at least $1-4\delta$,
    \begin{equation*}
        \begin{split}
            {}&\DReg_T\leq \underbrace{4H\beta_\theta\sum_{k=1}^K\sum_{h=1}^{H}\min\bbr{1, \norm{\phi(s_h^k,a_h^k)}_{\sbr{\Lambda_{h}^{k-1}}^{-1}}}}_{\variance \mathtt{1}}\\{}& +\underbrace{4\sum_{k=1}^K\sum_{h=1}^{H}\beta_{\w}^{k-1}\min\bbr{1, \norm{\psi_{h+1}^k\sbr{s_h^k,a_h^k}}_{\sbr{\Sigma_{h}^{k-1}}^{-1}}}}_{\variance \mathtt{2}}\\
            {}&+ \underbrace{2\sum_{k=1}^K\sum_{h=1}^{H}\Gamma_{h,\theta}^{k-1}+2\sum_{k=1}^K\sum_{h=1}^{H}\Gamma_{h,\w}^{k-1}}_{\bias} + 4H\sqrt{2T\log(1/\delta)}.
    \end{split}
    \end{equation*}

\parag{Bias.} Now we set $w_{k, j}=\gamma^{k-j}, \gamma\in(0,1)$,
\begin{align}
        {}&2\sum_{k=1}^K\sum_{h=1}^{H}\Gamma_{h,\theta}^{k-1}\nonumber \\ ={}&  2 L_{\phi}^2\sqrt{\frac{d}{\lamt}}\sum_{k=1}^K\sum_{h=1}^{H}\sum_{p=1}^{k-1}\sqrt{\sum_{j=1}^{p}w_{k-1,j}}\norm{\theta_h^p-\theta_h^{p+1}}_2\nonumber\\
        ={}&  2 L_{\phi}^2\sqrt{\frac{d}{\lamt}}\sum_{p=1}^{K-1} \sum_{h=1}^{H}\sum_{k=p+1}^K \sqrt{\sum_{j=1}^{p}w_{k-1,j}}\norm{\theta_h^p-\theta_h^{p+1}}_2\nonumber\\
        ={}&  2 L_{\phi}^2\sqrt{\frac{d}{\lamt}}\sum_{p=1}^{K-1} \sum_{h=1}^{H}\sum_{k=p+1}^K\gamma^{\frac{k-1}{2}} \sqrt{\sum_{j=1}^{p}\gamma^{-j}}\norm{\theta_h^p-\theta_h^{p+1}}_2\nonumber\\
        ={}&  2 L_{\phi}^2\sqrt{\frac{d}{\lamt}}\sum_{p=1}^{K-1} \sum_{h=1}^{H}\frac{\gamma^\frac{p}{2}-\gamma^\frac{K}{2}}{1-\gamma^{\frac{1}{2}}} \sqrt{\frac{\gamma^{-p}-1}{1-\gamma}}\norm{\theta_h^p-\theta_h^{p+1}}_2\nonumber\\
        &\leq  4 L_{\phi}^2\sqrt{\frac{d}{\lamt}}\frac{1}{(1-\gamma)^{3 / 2}} \sum_{p=1}^{K-1} \sum_{h=1}^{H}\norm{\theta_h^p-\theta_h^{p+1}}_2,\label{eq:LMDP-regret-bias-theta}
\end{align}
then we set $\alpha_{k, j}=\gamma^{k-j}, \gamma\in(0,1)$ and have
\begin{align}
        {}&2\sum_{k=1}^K\sum_{h=1}^{H}\Gamma_{h,\w}^{k-1}\nonumber \\ ={}& 2H^2L_{\psi}^2\sqrt{\frac{d}{\lamw}}\sum_{k=1}^K\sum_{h=1}^{H}\sum_{p=1}^{k-1}\sqrt{\sum_{j=1}^{p}\alpha_{k-1,j}}\norm{\w_h^p-\w_h^{p+1}}_2\nonumber\\
        ={}& 2H^2L_{\psi}^2\sqrt{\frac{d}{\lamw}}\sum_{p=1}^{K-1} \sum_{h=1}^{H}\sum_{k=p+1}^K \sqrt{\sum_{j=1}^{p}\alpha_{k-1,j}}\norm{\w_h^p-\w_h^{p+1}}_2\nonumber\\
        ={}& 2H^2L_{\psi}^2\sqrt{\frac{d}{\lamw}}\sum_{p=1}^{K-1} \sum_{h=1}^{H}\sum_{k=p+1}^K \gamma^{\frac{k-1}{2}} \sqrt{\sum_{j=1}^{p}\gamma^{-j}}\norm{\w_h^p-\w_h^{p+1}}_2\nonumber\\
        ={}& 2H^2L_{\psi}^2\sqrt{\frac{d}{\lamw}}\sum_{p=1}^{K-1} \sum_{h=1}^{H}\frac{\gamma^{\frac{p}{2}}-\gamma^{\frac{K}{2}}}{1-\gamma^{\frac{1}{2}}} \sqrt{\frac{\gamma^{-p}-1}{1-\gamma}}\norm{\w_h^p-\w_h^{p+1}}_2\nonumber\\
        \leq{}& 4H^2L_{\psi}^2\sqrt{\frac{d}{\lamw}}\frac{1}{(1-\gamma)^{3 / 2}}\sum_{p=1}^{K-1} \sum_{h=1}^{H}\norm{\w_h^p-\w_h^{p+1}}_2\label{eq:LMDP-regret-bias-w},
\end{align}

\parag{Variance.~}
For variance part 1, we have 
\begin{equation*}
    \begin{split}
        {}&4H\beta_\theta\sum_{k=1}^K\sum_{h=1}^{H}\min\bbr{1, \norm{\phi(s_h^k,a_h^k)}_{\sbr{\Lambda_{h}^{k-1}}^{-1}}} \\ \leq{}& 4H^2\beta_\theta\sqrt{K}\sqrt{\sum_{k=1}^K\min\bbr{1,\norm{\phi(s_h^k,a_h^k)}_{\sbr{\Lambda_{h}^{k-1}}^{-1}}^2}}.
\end{split}
\end{equation*}
Based on the Lemma~\ref{lemma:potential-lemma} (Potential Lemma), and let $X_k = \phi(s_h^k,a_h^k)$, $U_k = \Lambda_{h}^{k-1}$, we know that $\forall h\in[H]$, we have 
\begin{equation*}
    \begin{split}
        {}&\sum_{k=1}^K\min\bbr{1,\norm{\phi(s_h^k,a_h^k)}_{\sbr{\Lambda_{h}^{k-1}}^{-1}}^2}\\ \leq{}& 2d\sbr{K \log\frac{1}{\gamma}+\log\sbr{1+\frac{L_\phi^2 }{\lamt d(1-\gamma)}}}, 
\end{split}
\end{equation*}
so we have 
\begin{equation*}
    \begin{split}
        {}&4H\beta_\theta\sum_{k=1}^K\sum_{h=1}^{H}\min\bbr{1, \norm{\phi(s_h^k,a_h^k)}_{\sbr{\Lambda_{h}^{k-1}}^{-1}}} \\ \leq{}& 4H^2\beta_\theta\sqrt{K}\sqrt{2d\sbr{K \log\frac{1}{\gamma}+\log\sbr{1+\frac{L_\phi^2 }{\lamt d(1-\gamma)}}}}.
\end{split}
\end{equation*}
For variance part 2, we have
\begin{align}
        {}&4\sum_{k=1}^K\sum_{h=1}^{H}\beta_{\w}^{k-1}\min\bbr{1, \norm{\psi_{h+1}^k\sbr{s_h^k,a_h^k}}_{\sbr{\Sigma_{h}^{k-1}}^{-1}}} \nonumber\\ \leq{}& 4\beta_{\w}^{K}\sum_{k=1}^K\sum_{h=1}^{H} \min\bbr{1,\norm{\psi_{h+1}^k\sbr{s_h^k, a_h^k}}_{\sbr{\Sigma_{h}^{k-1}}^{-1}}}\label{eq:LMDP-potential-theta}\\
        \leq{}&  4H\beta_{\w}^{K}\sqrt{K}\sqrt{\sum_{k=1}^K\min\bbr{1,\norm{\psi_{h+1}^k\sbr{s_h^k, a_h^k}}_{\sbr{\Sigma_{h}^{k-1}}^{-1}}^2}}\nonumber
\end{align}
Based on potential lemma, we know that $\forall h\in[H]$, we have 
\begin{equation*}
    \begin{split}
       {}&\sum_{k=1}^K \min\bbr{1,\norm{\psi_{h+1}^k\sbr{s_h^k, a_h^k}}_{\sbr{\Sigma_{h}^{k-1}}^{-1}}^2} \\ \leq{}& 2d\sbr{K \log\frac{1}{\gamma}+\log\sbr{1+\frac{H^2L_\psi^2 }{\lamw d(1-\gamma)}}}, 
\end{split}
\end{equation*}
so we have 
\begin{equation*}
    \begin{split}
        {}&4\sum_{k=1}^K\sum_{h=1}^{H}\beta_{\w}^{k-1}\min\bbr{1, \norm{\psi_{h+1}^k\sbr{s_h^k,a_h^k}}_{\sbr{\Sigma_{h}^{k-1}}^{-1}}} \\ \leq{}& 4H\beta_{\w}^{K}\sqrt{K}\sqrt{2d\sbr{K \log\frac{1}{\gamma}+\log\sbr{1+\frac{H^2L_\psi^2 }{\lamw d(1-\gamma)}}}}.
\end{split}
\end{equation*}
Since there is a term $HK \sqrt{\log (1 / \gamma)}$ in the regret bound, we cannot let $\gamma$ close to 0 , so we set $\gamma \geq 1 / K$ and have $\log (1 / \gamma) \leq C(1-\gamma)$, where $C=\log K /(1-1 / K)$. We set $\lambda_\theta = d$, and $\lambda_\w = H^2 d$. Combining the upper bounds of the bias and variance parts and with confidence level $\delta=1 /(4 T)$, by union bound we have the following dynamic regret bound with probability at least $1-1 / T$,
\begin{equation*}
    \begin{split}
        {}&\DReg_T\\\leq{}& \O\Big(\frac{1}{(1-\gamma)^{3 / 2}} P_T^\theta + H\frac{1}{(1-\gamma)^{3 / 2}}P_T^\w + HdHK\sqrt{1-\gamma}\\{}&+ H^{3/2}d\sqrt{HK}\Big)\\
        \leq{}& \O\sbr{Hd\sbr{\frac{1}{(1-\gamma)^{3 / 2}} \Delta + HK\sqrt{1-\gamma}}+ H^{3/2}d\sqrt{HK}}
\end{split}
\end{equation*}
Furthermore, by setting the discounted factor optimally as $\gamma = 1-\max\bbr{1/K,\sqrt{\Delta/T}}$, we have
\begin{equation}
    \DReg_T \leq \begin{cases}\Ot\sbr{Hd \Delta^{1/4}T^{3/4}} & \text { when } \Delta \geq H / K, \\ \Ot\sbr{dH^{3/2}\sqrt{T}}& \text { when } \Delta<H / K .\end{cases}
    \end{equation}
\end{proof}
\begin{proof}[Proof of Lemma~\ref{lemma:LMDP-model-prediction-error}]
    We first consider the upper bound of $E_h^k$, based on the definition of $Q_h^k$~\eqref{eq:LMDP-optimistic-Q} and model assumption~\eqref{eq:LMDP-model} and Eq.~\eqref{eq:LMDP-linear-transtion}, we have $\forall a\in\A,s\in\S$, 
    \begin{equation*}
        \begin{split}
            {}& r_h^k(s, a)+\mbr{\P_h^k V_{h+1}^{k}}(s,a) - Q_h^{k}(s,a)\\
            ={}& r_h^k(s, a)+\mbr{\P_h^k V_{h+1}^{k}}(s,a) - \phi(s,a)^\T\th_{h}^k - \psi_{h+1}^k\sbr{s, a}^\top \wh_{h}^k\\{}&-\beta_\theta\norm{\phi(s,a)}_{\sbr{\Lambda_{h}^{k-1}}^{-1}}-\beta_{\w}^{k-1}\norm{\psi_{h+1}^k\sbr{s, a}}_{\sbr{\Sigma_{h}^{k-1}}^{-1}}\\
            ={}&\phi(s,a)^\T\sbr{\theta_{h}^k - \th_{h}^k} + \psi_{h+1}^k\sbr{s, a}^\top \sbr{\w_h^k-\wh_{h}^k}\\{}&-\beta_\theta\norm{\phi(s,a)}_{\sbr{\Lambda_{h}^{k-1}}^{-1}} - \beta_{\w}^{k-1}\norm{\psi_{h+1}^k\sbr{s, a}}_{\sbr{\Sigma_{h}^{k-1}}^{-1}}\\
            \leq{}&\Gamma_{h,\theta}^{k-1}+\Gamma_{h,\w}^{k-1},
        \end{split}
    \end{equation*}
    where the last inequality comes from Lemma~\ref{lemma:LMDP-r-ucb} and Lemma~\ref{lemma:LMDP-v-ucb}. Similarly, we can get the lower bound of $E_h^k$, $\forall a\in\A,s\in\S$, 
    \begin{equation*}
        \begin{split}
            {}& Q_h^{k}(s,a) - r_h^k(s, a)-\mbr{\P_h^k V_{h+1}^{k}}(s,a)\\
            ={}&  \phi(s,a)^\T\th_{h}^k+\beta_\theta\norm{\phi(s,a)}_{\sbr{\Lambda_{h}^{k-1}}^{-1}} + \psi_{h+1}^k\sbr{s, a}^\top \wh_{h}^k\\{}&+\beta_{\w}^{k-1}\norm{\psi_{h+1}^k\sbr{s, a}}_{\sbr{\Sigma_{h}^{k-1}}^{-1}} - r_h^k(s, a)-\mbr{\P_h^k V_{h+1}^{k}}(s,a) \\
            ={}&\phi(s,a)^\T\sbr{\th_{h}^k - \theta_{h}^k} + \psi_{h+1}^k\sbr{s, a}^\top \sbr{\wh_{h}^k - \w_h^k}\\{}&+\beta_\theta\norm{\phi(s,a)}_{\sbr{\Lambda_{h}^{k-1}}^{-1}} + \beta_{\w}^{k-1}\norm{\psi_{h+1}^k\sbr{s, a}}_{\sbr{\Sigma_{h}^{k-1}}^{-1}}\\
            \leq{}&\Gamma_{h,\theta}^{k-1}+\Gamma_{h,\w}^{k-1} + 2\beta_\theta\norm{\phi(s,a)}_{\sbr{\Lambda_{h}^{k-1}}^{-1}}\\{}& + 2\beta_{\w}^{k-1}\norm{\psi_{h+1}^k\sbr{s, a}}_{\sbr{\Sigma_{h}^{k-1}}^{-1}},
        \end{split}
    \end{equation*}
    thus completes the proof.
\end{proof}

\begin{proof}[Proof of Lemma~\ref{lemma:LMDP-regret-decomposition}]
    We first decompose the one step dynamic regret,
    \begin{equation*}
        \begin{split}
            {}&V_1^{k,\pi_*^k}\sbr{s_1^k}-V_1^{k,\pi^k}\sbr{s_1^k} \\ ={}& \underbrace{V_1^{k,\pi_*^k}\sbr{s_1^k}- V_1^{k}\sbr{s_1^k}}_{\term{1}}+\underbrace{V_1^{k}\sbr{s_1^k} -V_1^{k,\pi^k}\sbr{s_1^k}}_{\term{2}}.
        \end{split}
    \end{equation*}
    
\parag{\term{1}.~} We first have $\forall s\in \S$, 
\begin{equation*}
    \begin{split}
        {}& V_h^{k,\pi_*^k}\sbr{s}- V_h^{k}\sbr{s} \\
        ={}& \E_{a\sim \pi_{*,h}^k(s)}\mbr{Q_h^{k,\pi_*^k}(s,a)} - \E_{a\sim \pi_h^k(s)}\mbr{Q_h^{k}(s,a)}\\
        ={}& \E_{a\sim \pi_{*,h}^k(s)}\mbr{Q_h^{k,\pi_*^k}(s,a)}-\E_{a\sim \pi_{*,h}^k(s)}\mbr{Q_h^{k}(s,a)}\\{}&+\E_{a\sim \pi_{*,h}^k(s)}\mbr{Q_h^{k}(s,a)} - \E_{a\sim \pi_h^k(s)}\mbr{Q_h^{k}(s,a)}\\
        \leq{}& \E_{a\sim \pi_{*,h}^k(s)}\mbr{Q_h^{k,\pi_*^k}(s,a)-Q_h^{k}(s,a)},
        \end{split}
\end{equation*}
where the last inequality comes from $\pi_h^k(s)=\argmax_{a\in\A}Q_{h}^k(s,a)$. Then we have
\begin{equation*}
    \begin{split}
        {}& Q_h^{k,\pi_*^k}(s,a)-Q_h^{k}(s,a)\\
        ={}& r_h^k(s, a)+\mbr{\P_h^k V_{h+1}^{k,\pi_*^k}}(s,a) - r_h^k(s, a)-\mbr{\P_h^k V_{h+1}^{k}}(s,a)\\{}&+r_h^k(s, a)+\mbr{\P_h^k V_{h+1}^{k}}(s,a) - Q_h^{k}(s,a)\\
        ={}& \mbr{\P_h^k \sbr{V_{h+1}^{k,\pi_*^k}-V_{h+1}^{k}}}(s,a)+E_h^k(s,a),
    \end{split}
\end{equation*}
where the last equality comes from the definition of model prediction error. For notational simplicity, we define the operators $\Jb_h^k f(s) = \inner{f(x,\cdot)}{\pi_{*,h}^k(\cdot\given s)}$, then we have
\begin{equation*}
    \begin{split}
        {}&V_h^{k,\pi_*^k}\sbr{s}- V_h^{k}\sbr{s}\\ \leq{}& \Jb_h^k\P_h^k \sbr{V_{h+1}^{k,\pi_*^k}-V_{h+1}^{k}}(s,a)+\Jb_h^kE_h^k(s,a),
    \end{split}
\end{equation*}
recursively expanding the above inequality and we have 
\begin{equation*}
    \begin{split}
        {}&V_1^{k,\pi_*^k}\sbr{s_1^k}- V_1^{k}\sbr{s_1^k} \\ \leq{}& \sbr{\prod_{h=1}^H \Jb_h^k\P_h^k} \sbr{V_{H+1}^{k,\pi_*^k}(s_{H+1}^k)-V_{H+1}^{k}(s_{H+1})}\\{}& + \sum_{h=1}^H\sbr{\prod_{j=1}^H \Jb_j^k\P_j^k}\Jb_h^kE_h^k(s_h^k,a_h^k)\\
        \leq{}& \sum_{h=1}^H\E_{\pi_{*,h}^k}\mbr{E_h^k(s_h^k,a_h^k)},
    \end{split}
\end{equation*}
where the last inequality comes from that $\forall \pi, V_{H+1}^{k,\pi}(\cdot) = 0,V_{H+1}^{k}(\cdot) = 0$. Then we have 
\begin{equation}\label{eq:LMDP-regret-term1}
    \begin{split}
        {}&\sum_{k=1}^K V_1^{k,\pi_*^k}\sbr{s_h^k}- V_1^{k}\sbr{s_h^k} \\\leq{}& \sum_{k=1}^K\sum_{h=1}^H\E_{\pi_{*,h}^k}\mbr{E_h^k(s_h^k,a_h^k)}.
    \end{split}
\end{equation}

\parag{\term{2}.~}  Based on Eq~\eqref{eq:LMDP-model-prediction-error}, we have $\forall s\in\S, a\in\A$,
\begin{equation*}
    \begin{split}
        E_h^k(s,a) ={}& r_h^k(s, a)+\mbr{\P_h^k V_{h+1}^{k}}(s,a) - Q_h^{k}(s,a)\\
        ={}& r_h^k(s, a)+\mbr{\P_h^k V_{h+1}^{k}}(s,a) -Q_h^{k,\pi^k}(s,a)\\{}&+Q_h^{k,\pi^k}(s,a)- Q_h^{k}(s,a)\\
        ={}&\sbr{\mbr{\P_h^k V_{h+1}^{k}}(s,a) -\mbr{\P_h^k V_{h+1}^{k,\pi^k}}(s,a)}\\{}&+Q_h^{k,\pi^k}(s,a)- Q_h^{k}(s,a).
    \end{split}
\end{equation*}
By applying this equality, we further have 
\begin{equation*}
    \begin{split}
        {}& V_{h}^k\sbr{s_{h}^k}- V_h^{k,\pi^k}\sbr{s_{h}^k}\\
        ={}&  \E_{a\sim \pi_h^k(s_h^k)}\mbr{Q_h^{k}(s_h^k,a)}- \E_{a\sim \pi_h^k(s_h^k)}\mbr{Q_h^{k,\pi^k}(s_h^k,a)}\\{}&\hspace{15em}+E_h^k-E_h^k\\
        ={}& \E_{a\sim \pi_h^k(s_h^k)}\mbr{Q_h^{k}(s_h^k,a)}- \E_{a\sim \pi_h^k(s_h^k)}\mbr{Q_h^{k,\pi^k}(s_h^k,a)}\\
         {}&+ \sbr{\mbr{\P_h^k V_{h+1}^{k}}\sbr{s_h^k,a_h^k} -\mbr{\P_h^k V_{h+1}^{k,\pi^k}}\sbr{s_h^k,a_h^k}}\\{}&+Q_h^{k,\pi_h^k}\sbr{s_h^k,a_h^k}- Q_h^{k}\sbr{s_h^k,a_h^k} - E_h^k\sbr{s_h^k,a_h^k}.
\end{split}
\end{equation*}
we define 
\begin{equation*}
   \begin{split}
       \M_{h,V}^k ={}& \sbr{\mbr{\P_h^k V_{h+1}^{k}}\sbr{s_h^k,a_h^k} -\mbr{\P_h^k V_{h+1}^{k,\pi_h^k}}\sbr{s_h^k,a_h^k}}\\{}& - \sbr{V_{h+1}^k\sbr{s_{h+1}^k}- V_{h+1}^{k,\pi^k}\sbr{s_{h+1}^k}}\\
       \M_{h,Q}^k ={}& \E_{a\sim \pi_h^k(s)}\mbr{Q_h^{k}(s_h^k,a) - Q_h^{k,\pi_h^k}(s_h^k,a)}\\{}& - \sbr{Q_h^{k}\sbr{s_h^k,a_h^k} - Q_h^{k,\pi_h^k}\sbr{s_h^k,a_h^k}},
\end{split}
\end{equation*}
then we have 
\begin{equation*}
    \begin{split}
        {}&V_{h}^k\sbr{s_{h}^k}- V_h^{\pi^k}\sbr{s_{h}^k} -\sbr{V_{h+1}^k\sbr{s_{h+1}^k}- V_{h+1}^{\pi^k}\sbr{s_{h+1}^k}}\\={}& \M_{h,V}^k +\M_{h,Q}^k -E_h^k\sbr{s_h^k,a_h^k}.
\end{split}
\end{equation*}
Summing up for $k\in[K]$ and $h\in[H]$, since $V_{H+1}^k =0$, $V_{H+1}^{k,\pi} = 0$, we have
\begin{equation*}
\begin{split}
    {}&\sum_{k=1}^KV_{1}^k\sbr{s_{1}^k}- V_1^{\pi^k}\sbr{s_{1}^k}\\ \leq{}& \sum_{k=1}^K\sum_{h=1}^{H}\sbr{-E_h^k\sbr{s_h^k,a_h^k}}+\sum_{k=1}^K\sum_{h=1}^{H}\M_{h,V}^k + \sum_{k=1}^K\sum_{h=1}^{H}\M_{h,Q}^k.
\end{split}
\end{equation*}
Since $\M_{h,V}^k$ and $\M_{h,Q}^k$ are martingale difference which bounded by $2H$, then based on Lemma~\ref{lemma:Azuma}, we have the following holds each with probability at least $1-\delta$,
    \begin{equation*}
        \begin{split}
            \sum_{k=1}^K\sum_{h=1}^{H}\M_{h,V}^k &\leq 2H\sqrt{2T\log(1/\delta)},\\ \sum_{k=1}^K\sum_{h=1}^{H}\M_{h,Q}^k &\leq 2H\sqrt{2T\log(1/\delta)},
    \end{split}
    \end{equation*}
    where $T=KH$. Then the following holds with probability at least $1-2\delta$,
    \begin{equation}\label{eq:LMDP-regret-term2}
        \begin{split}
            {}&\sum_{k=1}^K \sbr{V_{1}^k\sbr{s_{1}^k}- V_1^{\pi^k}\sbr{s_{1}^k}}\\ \leq{}& \sum_{k=1}^K\sum_{h=1}^{H}\sbr{-E_h^k\sbr{s_h^k,a_h^k}} + 4H\sqrt{2T\log(1/\delta)}.
    \end{split}
    \end{equation}
    Combining Eq~\eqref{eq:LMDP-regret-term1} and Eq~\eqref{eq:LMDP-regret-term2} and we have 
    \begin{equation*}
        \begin{split}
            {}&V_1^{k,\pi_*^k}\sbr{s_1^k}-V_1^{k,\pi^k}\sbr{s_1^k}\\ ={}& \sum_{k=1}^K\sum_{h=1}^H\sbr{\E_{\pi_{*,h}^k}\mbr{E_h^k(s_h^k,a_h^k)} - E_h^k\sbr{s_h^k,a_h^k}}\\{}&\hspace{12em}+ 4H\sqrt{2T\log(1/\delta)}.
        \end{split}
    \end{equation*}
\end{proof}
\section{Analysis of \MNLweightours}
\label{app:MNL}
\subsection{Proof of Lemma~\ref{lemma:MNL-v-UCB}}
\label{app:MNL-lemma-v-ucb}
\begin{proof}
    Based on Lemma~2 of~\cite{NIPS'24:MNL_MDP_efficient}, we know that 
    \begin{equation*}
        \begin{split}
            {}&\abs{\mbr{\Pt_h^k V}(s,a)-\mbr{\P_h^k V}(s,a)}\\ ={}& \Bigg|\sum_{s^\prime\in\S_h^k}p_h^k(\psi(s^\prime\given s,a)^\T \wt_h^k)V(s^\prime)\\{}&\qquad\qquad\qquad\qquad-\sum_{s^\prime \in\S_h^k}p_h^k(\psi(s^\prime\given s,a)^\T \w_h^k)V(s^\prime)\Bigg|\\
             \leq{}& H \max_{s^\prime \in \S_h^k}\abs{\psi(s^\prime\given s,a)^\T \sbr{\wt_{h}^k - \w_h^k}}.
        \end{split}
    \end{equation*}
    We first construct the following Lemma.
\begin{myLemma}
    \label{lemma:MNL-p-ucb}
    For any $\x \in \X$, and $\delta \in (0,1)$, $\forall k,j\in [K], \alpha_{k,j}\leq 1$, with probability at least $1-\delta$, the following holds for all $k \in [K], h\in[H]$
    \begin{align*}
        {}&\abs{\psi(s^\prime\given s,a)^\T \sbr{\wt_{h}^k - \w_h^k}}\\ \leq{}& \frac{1}{\kappa}\sbr{\Gamma_{\w}^{k-1} + \betab_{\w}^{k-1}\norm{\psi(s^\prime\given s,a)}_{\sbr{\Sigmab_{h}^{k-1}}^{-1}}}.
    \end{align*}
    $\Gamma_{\w}^{k-1}\define  L_{\psi}^2\sqrt{\frac{d}{\lamw}}\sum_{p=1}^{k-1}\sqrt{\sum_{j=1}^{p}\alpha_{k-1,j}}\norm{\w_h^p-\w_h^{p+1}}_2$, and $\betab_{\w}^{k}$ is the radius of confidence region set by
    \begin{align*}
        \betab_{\w}^{k} \define{}& \sqrt{\frac{1}{2}\log \frac{1}{\delta}+\frac{d}{4}\log\left(1+\frac{U L_\psi^2\sum_{j=1}^{k}\alpha_{k,j}}{\lamw  d}\right)}\\{}&\hspace{15em}+\sqrt{\lamw}\kappa S_\w.
    \end{align*}
    \end{myLemma}
    \noindent Then we have 
    \begin{equation*}
        \begin{split}
            {}&\abs{\mbr{\Pt_h^k V}(s,a)-\mbr{\P_h^k V}(s,a)}\\ \leq{}& \frac{H}{\kappa}\sbr{\Gamma_{\w}^{k-1} + \betab_{\w}^{k-1}\max_{s^\prime \in \S_h^k}\norm{\psi(s^\prime\given s,a)}_{\sbr{\Sigmab_{h}^{k-1}}^{-1}}},
        \end{split}
    \end{equation*}
    where $\Gamma_{\w}^{k-1}\define  L_{\psi}^2\sqrt{\frac{d}{\lamw}}\sum_{p=1}^{k-1}\sqrt{\sum_{j=1}^{p}\alpha_{k-1,j}}\norm{\w_h^p-\w_h^{p+1}}_2$, $\betab_{\w}^{k}$ is the radius of confidence region set by
    \begin{align*}
        \betab_{\w}^{k} \define{}& \sqrt{\frac{1}{2}\log \frac{1}{\delta}+\frac{d}{4}\log\left(1+\frac{U L_\psi^2\sum_{j=1}^{k}\alpha_{k,j}}{\lamw  d}\right)}\\{}&\hspace{15em}+\sqrt{\lamw}\kappa S_\w.
    \end{align*}
\end{proof}

\begin{proof}[Proof of Lemma~\ref{lemma:MNL-p-ucb}]
    Fix $h\in[H]$, based on the estimator~\eqref{eq:MNL-estimate-p}, we have 
    \begin{equation*}
        \begin{split}
            g_h^k(\wh_{h}^k) &= \lamw \kappa\wh_{h}^k+\sum_{j=1}^{k-1}\alpha_{k-1,j} \sum_{s^\prime \in \S_{h}^j} p_h^j(\psib_h^j(s^\prime)^\T \wh_{h}^k) \psib_h^j(s^\prime) \\
            &= \sum_{j=1}^{k-1}\alpha_{k-1,j} \sum_{s^\prime \in \S_{h}^j}y_{h}^{j}(s^\prime) \psib_h^j(s^\prime)\\
            &= \sum_{j=1}^{k-1}\alpha_{k-1,j} \sum_{s^\prime \in \S_{h}^j}\sbr{p_h^j(\psib_h^j(s^\prime)^\T \w_{h}^j) +\eta_h^j(s^\prime)} \psib_h^j(s^\prime),
        \end{split}
    \end{equation*}
    where  we define $\eta_h^j(s^\prime) \define y_{h}^{j}(s^\prime) - p_h^j(\psib_h^j(s^\prime)^\T \w_{h}^j)$. Then by the mean value theorem, we know that 
    \begin{equation}\label{eq:MNL-MLE}
        \begin{split}
            g_h^k(\w_1) - g_h^k(\w_2) = G_h^k(\w_1-\w_2), 
        \end{split}
    \end{equation}
    where $ G_h^k(\w_1-\w_2)\define \int_0^1 \nabla g_h^k\left(s \w_2+(1-s) \w_1\right) \mathrm{d} s \in \R^{d \times d}$. Notice that for any $\w\in\W$, the gradient of $g_h^k$ is 
    \begin{equation*}
        \begin{split}
            &\nabla g_h^k(\w) = \lamw \kappa I_d\\{}& + \sum_{j=1}^{k-1}\alpha_{k-1,j} \Bigg(\sum_{s^\prime \in \S_{h}^j} p_h^j(\psib_h^j(s^\prime)^\T \w) \psib_h^j(s^\prime)\psib_h^j(s^\prime)^\T\\{}& - \sum_{s^\prime \in \S_{h}^j} \sum_{s^{\prime\prime} \in \S_{h}^j}p_h^j(\psib_h^j(s^\prime)^\T \w) p_h^j(\psib_h^j(s^{\prime\prime})^\T \w) \psib_h^j(s^\prime)\psib_h^j(s^{\prime\prime})^\T\Bigg)
        \end{split}
    \end{equation*}
    Based on Lemma~5 of~\cite{NIPS'24:MNL_MDP_efficient}, we know that 
    \begin{equation*}
        \begin{split}
            &\nabla g_h^k(\w)\\ &\succeq \lamw \kappa I_d + \kappa\sum_{j=1}^{k-1}\alpha_{k-1,j} \sum_{s^\prime \in \S_{h}^j} \psib_h^j(s^\prime)\psib_h^j(s^\prime)^\T = \kappa \Sigmab_h^{k-1},
        \end{split}
    \end{equation*}
    which clearly implies $\forall \w_1,\w_2\in\W, G_h^k(\w_1,\w_2)\succeq \kappa \Sigmab_h^{k-1}$. By Assumption~\ref{ass:MNL-kappa}, the mean value theorem~\eqref{eq:MNL-MLE} on $g_h^k$ and the projection~\eqref{eq:MNL-projection}, we have 
    \begin{equation*}
        \begin{split}
            {}&\abs{\psi(s^\prime\given s,a)^\T \sbr{\wt_{h}^k - \w_h^k}} \\ ={}& \abs{\psi(s^\prime\given s,a)^\T \sbr{{G_{h}^k}(\wt_{h}^k, \w_h^k)}^{-1}\sbr{g_h^k(\wt_{h}^k) - g_h^k(\w_h^k)}}\\
            \leq{}& \norm{\psi(s^\prime\given s,a)}_{\sbr{{G_{h}^k}(\wt_{h}^k, \w_h^k)}^{-1}}\\{}&\quad\cdot\norm{g_h^k(\wt_{h}^k) - g_h^k(\w_h^k)}_{\sbr{{G_{h}^k}(\wt_{h}^k, \w_h^k)}^{-1}}\\
            \leq{}& \frac{1}{\kappa}\norm{\psi(s^\prime\given s,a)}_{\sbr{\Sigmab_h^{k-1}}^{-1}}\norm{g_h^k(\wt_{h}^k) - g_h^k(\w_h^k)}_{\sbr{\Sigmab_h^{k-1}}^{-1}}\\
            \leq{}& \frac{1}{\kappa}\norm{\psi(s^\prime\given s,a)}_{\sbr{\Sigmab_h^{k-1}}^{-1}}\norm{g_h^k(\wh_{h}^k) - g_h^k(\w_h^k)}_{\sbr{\Sigmab_h^{k-1}}^{-1}},
        \end{split}
    \end{equation*}  
    we further have 
    \begin{align*}
            {}&g_h^k(\wh_{h}^k) - g_h^k(\w_h^k)\\
            ={}& \sum_{j=1}^{k-1}\alpha_{k-1,j} \sum_{s^\prime \in \S_{h}^j}\sbr{p_h^j(\psib_h^j(s^\prime)^\T \w_{h}^j) +\eta_h^j(s^\prime)} \psib_h^j(s^\prime)\\{}&-\lamw \kappa\w_{h}^k-\sum_{j=1}^{k-1}\alpha_{k-1,j} \sum_{s^\prime \in \S_{h}^j}p_h^j(\psib_h^j(s^\prime)^\T \w_{h}^k) \psib_h^j(s^\prime)\\
            ={}& \underbrace{\sum_{j=1}^{k-1}\alpha_{k-1,j} \sum_{s^\prime \in \S_{h}^j}\sbr{p_h^j(\psib_h^j(s^\prime)^\T \w_{h}^j) -p_h^j(\psib_h^j(s^\prime)^\T \w_{h}^k)} \psib_h^j(s^\prime)}_{\bias}\\{}&\qquad\qquad-\underbrace{\lamw \kappa\w_{h}^k+\sum_{j=1}^{k-1}\alpha_{k-1,j} \sum_{s^\prime \in \S_{h}^j} \eta_h^j(s^\prime) \psib_h^j(s^\prime)}_{\variance}.
        \end{align*} 
    Then, by the Cauchy-Schwarz inequality, we know that for any $s\in\S, a\in \A$, 
    \begin{equation}\label{eq:MNL-decomposation}
        \begin{split}
            {}&\abs{\psi(s^\prime\given s,a)^\T \sbr{\wh_{h}^k - \w_h^k}}\\ \leq{}& \frac{1}{\kappa}\norm{\psi(s^\prime\given s,a)}_{\sbr{\Sigmab_h^{k-1}}^{-1}}\sbr{E_h^k+F_h^k},
        \end{split}
    \end{equation}  
    where 
    \begin{equation*}
        \begin{split}
            E_h^k ={}& \Bigg\|\sum_{j=1}^{k-1}\alpha_{k-1,j} \sum_{s^\prime \in \S_{h}^j}\bigg(p_h^j(\psib_h^j(s^\prime)^\T \w_{h}^j)\\{}&\qquad\qquad\qquad -p_h^j(\psib_h^j(s^\prime)^\T \w_{h}^k)\bigg) \psib_h^j(s^\prime)\Bigg\|_{\sbr{\Sigmab_{h}^{k-1}}^{-1}}\\
            F_h^k ={}& \norm{-\lamw \kappa\w_{h}^k+\sum_{j=1}^{k-1}\alpha_{k-1,j} \sum_{s^\prime \in \S_{h}^j} \eta_h^j(s^\prime) \psib_h^j(s^\prime)}_{\sbr{\Sigmab_{h}^{k-1}}^{-1}}.
        \end{split}
    \end{equation*}
    The above two terms can be bounded separately, as summarized in the following two lemmas,
    \begin{myLemma}\label{lemma:MNL-E}
        For any $k\in[K]$, we have
        \begin{equation*}
            \begin{split}
                E_h^k\leq L_{\psi}\sqrt{d}\sum_{p=1}^{k-1}\sqrt{\sum_{j=1}^{p}\alpha_{k-1,j}}\norm{\w_h^p-\w_h^{p+1}}_2.
            \end{split}
        \end{equation*}
    \end{myLemma}
    \begin{myLemma}\label{lemma:MNL-F}
        If $\forall k\in[K],\forall j\in[k-1], \alpha_{k-1,j}\leq 1$, for any $\delta \in(0,1)$, with probability at least $1-\delta$, the following holds for all $k\in[K]$,
        \begin{equation*}
            \begin{split}
                F_h^k  \leq{}& \sqrt{\frac{1}{2}\log \frac{1}{\delta}+\frac{d}{4}\log\left(1+\frac{U L_\psi^2\sum_{j=1}^{k-1}\alpha_{k-1,j}}{\lamw  d}\right)}\\{}&\qquad\qquad\qquad\qquad\qquad\qquad\qquad +\sqrt{\lamw}\kappa S_\w.
            \end{split}
        \end{equation*}
    \end{myLemma}
    Based on the inequality~\eqref{eq:MNL-decomposation}, Lemma~\ref{lemma:MNL-E}, Lemma~\ref{lemma:MNL-F}, and $\norm{\psi(s^\prime\given s,a)}_{\sbr{\Sigmab_{h}^{k-1}}^{-1}}\leq \norm{\psi(s^\prime\given s,a)}_{2}/\sqrt{\lambda_\w}\leq L_{\psi}/\sqrt{\lambda_\w}$, with probability at least $1-\delta$, the following holds for all $k\in[K]$,
    \begin{equation*}
        \begin{split}
            {}&\abs{\psi(s^\prime\given s,a)^\T \sbr{\wt_{h}^k - \w_h^k}} \\\leq {}& \frac{1}{\kappa}\sbr{\Gamma_{\w}^{k-1} + \betab_{\w}^{k-1}\norm{\psi(s^\prime\given s,a)}_{\sbr{\Sigmab_{h}^{k-1}}^{-1}}},
        \end{split}
    \end{equation*}
    where 
    \begin{equation*}
        \begin{split}
            \Gamma_{\w}^{k-1}&\define  L_{\psi}^2\sqrt{\frac{d}{\lamw}}\sum_{p=1}^{k-1}\sqrt{\sum_{j=1}^{p}\alpha_{k-1,j}}\norm{\w_h^p-\w_h^{p+1}}_2\\
            \betab_{\w}^{k-1} &\define \sqrt{\frac{1}{2}\log \frac{1}{\delta}+\frac{d}{4}\log\left(1+\frac{U L_\psi^2\sum_{j=1}^{k-1}\alpha_{k-1,j}}{\lamw  d}\right)}\\{}&\qquad\qquad\qquad\qquad\qquad\qquad\qquad +\sqrt{\lamw}\kappa S_\w,
        \end{split}
    \end{equation*}
    which completes the proof.
\end{proof}
\begin{proof}[Proof of Lemma~\ref{lemma:MNL-E}]
    The first step is to extract the variations of the parameter $\w_h^k$ as follows,
    \begin{align*}
            {}&\Bigg\|\sum_{j=1}^{k-1}\alpha_{k-1,j} \sum_{s^\prime \in \S_{h}^j}\bigg(p_h^j(\psib_h^j(s^\prime)^\T \w_{h}^j) \\
            {}&\qquad\qquad\qquad\qquad -p_h^j(\psib_h^j(s^\prime)^\T \w_{h}^k)\bigg) \psib_h^j(s^\prime)\Bigg\|_{\sbr{\Sigmab_{h}^{k-1}}^{-1}}\\
            ={}&\Bigg\|\sum_{j=1}^{k-1}\alpha_{k-1,j} \sum_{s^\prime \in \S_{h}^j}\sum_{p=j}^{k-1}\bigg(p_h^j(\psib_h^j(s^\prime)^\T \w_{h}^p) \\
            {}&\qquad\qquad\qquad\qquad -p_h^j(\psib_h^j(s^\prime)^\T \w_{h}^{p+1})\bigg) \psib_h^j(s^\prime)\Bigg\|_{\sbr{\Sigmab_{h}^{k-1}}^{-1}}\\
            ={}&\Bigg\|\sum_{p=1}^{k-1} \sum_{j=1}^{p}\alpha_{k-1,j} \sum_{s^\prime \in \S_{h}^j}\bigg(p_h^j(\psib_h^j(s^\prime)^\T \w_{h}^p) \\
            {}&\qquad\qquad\qquad\qquad -p_h^j(\psib_h^j(s^\prime)^\T \w_{h}^{p+1})\bigg) \psib_h^j(s^\prime)\Bigg\|_{\sbr{\Sigmab_{h}^{k-1}}^{-1}}.
    \end{align*}
    The gradient of $p_h^j(\psib_h^j(s^\prime)^\T \w)$ is given by 
    \begin{equation*}
        \begin{split}
            {}&\nabla p_h^j(\psib_h^j(s^\prime)^\T \w) = p_h^j(\psib_h^j(s^\prime)^\T \w)\psib_h^j(s^\prime)^\T \\{}&\qquad\qquad  - p_h^j(\psib_h^j(s^\prime)^\T \w) \sum_{s^{\prime\prime}\in\S_h^j} p_h^j(\psib_h^j(s^{\prime\prime})^\T \w)\psib_h^j(s^{\prime\prime})^\T,
        \end{split}
    \end{equation*}
    by the mean value theorem, there exist $\wb_h^p = \nu \w_{h}^p + (1-\nu) \w_{h}^{p+1}$ for some $\nu\in[0,1]$, such that
    \begin{align*}
            &p_h^j(\psib_h^j(s^\prime)^\T \w_{h}^p) -p_h^j(\psib_h^j(s^\prime)^\T \w_{h}^{p+1})\\={}&p_h^j(\psib_h^j(s^\prime)^\T \wb_{h}^p) \Bigg(\psib_h^j(s^\prime)^\T- \sum_{s^{\prime\prime}\in\S_h^j} p_h^j(\psib_h^j(s^{\prime\prime})^\T \wb_{h}^p)\psib_h^j(s^{\prime\prime})^\T\Bigg)\\& \cdot\sbr{\w_{h}^p - \w_{h}^{p+1}}\\
            \leq{}&L_\psi p_h^j(\psib_h^j(s^\prime)^\T \wb_{h}^p)\Bigg(1- \sum_{s^{\prime\prime}\in\S_h^j} p_h^j(\psib_h^j(s^{\prime\prime})^\T \wb_{h}^p)\Bigg)\\&\cdot\norm{\w_{h}^p - \w_{h}^{p+1}}_2\\
            \leq{}&L_\psi\nabla p_h^j(\psib_h^j(s^\prime)^\T \wb_{h}^p) \norm{\w_{h}^p - \w_{h}^{p+1}}_2,
    \end{align*}
    where we define $\nabla p_h^j(\psib_h^j(s^\prime)^\T \wb_{h}^p) = p_h^j(\psib_h^j(s^\prime)^\T \wb_{h}^p)  - p_h^j(\psib_h^j(s^\prime)^\T \wb_{h}^p) \sum_{s^{\prime\prime}\in\S_h^j} p_h^j(\psib_h^j(s^{\prime\prime})^\T \wb_{h}^p)$. Then we have
    \begin{align*}
            {}&\Bigg\|\sum_{p=1}^{k-1} \sum_{j=1}^{p}\alpha_{k-1,j} \sum_{s^\prime \in \S_{h}^j}\bigg(p_h^j(\psib_h^j(s^\prime)^\T \w_{h}^p)\\{}&\quad -p_h^j(\psib_h^j(s^\prime)^\T \w_{h}^{p+1})\bigg) \psib_h^j(s^\prime)\Bigg\|_{\sbr{\Sigmab_{h}^{k-1}}^{-1}}\\
            \leq{}&L_\psi\sum_{p=1}^{k-1}\norm{ \sum_{j=1}^{p}\alpha_{k-1,j} \sum_{s^\prime \in \S_{h}^j}\nabla p_h^j(\psib_h^j(s^\prime)^\T \wb_{h}^p) \psib_h^j(s^\prime)}_{\sbr{\Sigmab_{h}^{k-1}}^{-1}}\\{}& \quad \cdot\norm{\w_{h}^p - \w_{h}^{p+1}}_2\\
            \leq{}&L_\psi\sum_{p=1}^{k-1} \sum_{j=1}^{p}\alpha_{k-1,j} \abs{\sum_{s^\prime \in \S_{h}^j}\nabla p_h^j(\psib_h^j(s^\prime)^\T \wb_{h}^p)}\\{}&\quad \cdot\norm{\max_{s^\prime \in \S_h^j}\psib_h^j(s^\prime)}_{\sbr{\Sigmab_{h}^{k-1}}^{-1}}\norm{\w_{h}^p - \w_{h}^{p+1}}_2\\
            \leq{}&L_\psi\sum_{p=1}^{k-1} \sum_{j=1}^{p}\alpha_{k-1,j} \norm{\max_{s^\prime \in \S_h^j}\psib_h^j(s^\prime)}_{\sbr{\Sigmab_{h}^{k-1}}^{-1}}\norm{\w_{h}^p - \w_{h}^{p+1}}_2,
    \end{align*}
    where the last inequality comes from the fact that $\abs{\sum_{s^\prime \in \S_{h}^j}\nabla p_h^j(\psib_h^j(s^\prime)^\T \wb_{h}^p)}\leq 1$. Further, for the term $\sum_{j=1}^{p}\alpha_{k-1,j} \norm{\max_{s^\prime \in \S_h^j}\psib_h^j(s^\prime)}_{\sbr{\Sigmab_{h}^{k-1}}^{-1}}$ can be able to further derive an expression about weight $\alpha_{k-1,j}$ as follows,
    \begin{align*}
            {}&\sum_{j=1}^{p}\alpha_{k-1,j} \norm{\max_{s^\prime \in \S_h^j}\psib_h^j(s^\prime)}_{\sbr{\Sigmab_{h}^{k-1}}^{-1}}\\ \leq{}& \sqrt{\sum_{j=1}^{p}\alpha_{k-1,j}}\sqrt{\sum_{j=1}^{p}\alpha_{k-1,j}\norm{\max_{s^\prime \in \S_h^j}\psib_h^j(s^\prime)}_{\sbr{\Sigmab_{h}^{k-1}}^{-1}}^2}\\
            \leq{}& \sqrt{\sum_{j=1}^{p}\alpha_{k-1,j}}\sqrt{\sum_{j=1}^{p}\sum_{s^\prime\in\S_h^j}\alpha_{k-1,j}\norm{\psib_h^j(s^\prime)}_{\sbr{\Sigmab_{h}^{k-1}}^{-1}}^2}\\
            \leq{}& \sqrt{d}\sqrt{\sum_{j=1}^{p}\alpha_{k-1,j}}.
    \end{align*}
    In above, the second last step holds by the Cauchy-Schwarz inequality. Besides, the last step follows that
    \begin{align*}
        {} &\sum_{j=1}^{p}\sum_{s^\prime\in\S_h^j}\alpha_{k-1,j}\norm{\psib_h^j(s^\prime)}_{\sbr{\Sigmab_{h}^{k-1}}^{-1}}^2 \\
         {} &\sum_{j=1}^{p}\sum_{s^\prime\in\S_h^j} \mathrm{Tr}(\alpha_{k-1,j}\psib_h^j(s^\prime)^\T \sbr{\Sigmab_{h}^{k-1}}^{-1} \psib_h^j(s^\prime))\\
       = {}& \mathrm{Tr}\sbr{\sbr{\Sigmab_{h}^{k-1}}^{-1} \sum_{j=1}^{p}\sum_{s^\prime\in\S_h^j}\alpha_{k-1,j} \psib_h^j(s^\prime)\psib_h^j(s^\prime)^\T} \\
       \leq {} & \mathrm{Tr}\sbr{\sbr{\Sigmab_{h}^{k-1}}^{-1} \sum_{j=1}^{k-1}\sum_{s^\prime\in\S_h^j}\alpha_{k-1,j} \psib_h^j(s^\prime)\psib_h^j(s^\prime)^\T}\\{}& + \mathrm{Tr}\sbr{U_{t-1}^{-1} \lambda \sum_{i=1}^{d} \mathbf{e}_i \mathbf{e}_i^{\T}} \\
       = {} & \mathrm{Tr}(I_d) = d.
     \end{align*} 
    Hence we complete the proof.
\end{proof}

\begin{proof}[Proof of Lemma~\ref{lemma:MNL-F}]
    \begin{align*}
            {}&\norm{-\lamw \kappa\w_{h}^k+\sum_{j=1}^{k-1}\alpha_{k-1,j} \sum_{s^\prime \in \S_{h}^j} \eta_h^j(s^\prime) \psib_h^j(s^\prime)}_{\sbr{\Sigmab_{h}^{k-1}}^{-1}}\\ \leq{}& \norm{\sum_{j=1}^{k-1}\alpha_{k-1,j} \sum_{s^\prime \in \S_{h}^j} \eta_h^j(s^\prime) \psib_h^j(s^\prime)}_{\sbr{\Sigmab_{h}^{k-1}}^{-1}}+ \norm{\lamw \kappa\w_h^k}_{\sbr{\Sigmab_{h}^{k-1}}^{-1}}\\
            \leq{}& \norm{\sum_{j=1}^{k-1}\alpha_{k-1,j} \sum_{s^\prime \in \S_{h}^j} \eta_h^j(s^\prime) \psib_h^j(s^\prime)}_{\sbr{\Sigmab_{h}^{k-1}}^{-1}}+ \sqrt{\lamw}\kappa S_\w.
    \end{align*}
     we define $\etat_j = \sqrt{\alpha_{k-1,j}}\eta_j$ and $X_j = \sqrt{\alpha_{k-1,j}}\psib_h^j(s^\prime)$, then we have 
    notice that since the reward $r\in [0,1]$, and $\alpha_{k-1,j} \leq 1$, the noise $\etat_j$ is bounded by:
    \begin{equation*}
        \begin{split}
            \etat_j = \sqrt{\alpha_{k-1,j}}\sbr{y_{h}^{j}(s^\prime) - p_h^j(\psib_h^j(s^\prime)^\T \w_{h}^j)} \leq 1,
        \end{split}
    \end{equation*}
    based on Lemma~\ref{lemma:Hoeffding}, we find that the noise $\etat_j$ is $\frac{1}{2}$-sub-Gaussian. Then, by Theorem~\ref{thm:snc-AY}, we have with probability at least $1-\delta$, the following holds for all $k\in[K]$.
    \begin{equation*}
        \begin{split}
            \norm{\sum_{j=1}^{k-1}\etat_j X_j}_{\sbr{\Sigmab_{h}^{k-1}}^{-1}} \leq \sqrt{\frac{1}{2}\log\left(\frac{\det(\Sigmab_{h}^{k-1})^{\frac{1}{2}}\det(\Sigmab_{h}^0)^{-\frac{1}{2}}}{\delta}\right)}
        \end{split}
    \end{equation*}
    where
    \begin{equation*}
        \begin{split}
            \det(\Sigmab_{h}^{k-1}) &\leq \sbr{\frac{\trace(\Sigmab_{h}^{k-1})}{d}}^d\\ &= \sbr{\frac{d\lamw + \sum_{j=1}^{k-1}\sum_{s^\prime \in \S_{h}^j}\alpha_{k-1,j}\norm{\psib_h^j(s^\prime)}_{2}^2}{d}}^d\\ &=\sbr{\frac{d\lamw+U L_\psi^2\sum_{j=1}^{k-1}\alpha_{k-1,j}}{d}}^d\\
            \det(\Sigmab_{h}^0) &\leq (\lamw)^d,
        \end{split}
    \end{equation*}
    so we have 
    \begin{equation*}
        \begin{split}
            {}&\norm{\sum_{j=1}^{k-1}\alpha_{k-1,j}\eta_j \psib_{h+1}^j\sbr{s_h^j, a_h^j}}_{\sbr{\Sigmab_{h}^{k-1}}^{-1}}\\\leq{}& \sqrt{\frac{1}{2}\log \frac{1}{\delta}+\frac{d}{4}\log\left(1+\frac{U L_\psi^2\sum_{j=1}^{k-1}\alpha_{k-1,j}}{\lamw  d}\right)}.
        \end{split}
    \end{equation*}
    which completes the proof.
\end{proof}

\subsection{Proof of Theorem~\ref{thm:MNL-regret-bound}}
\label{app:MNL-thm-regret-bound}
\begin{proof}
To prove the theorem, we first introduce the following lemma
    \begin{myLemma}
        \label{lemma:MNL-model-prediction-error}
We define the model prediction error as 
\begin{align}\label{eq:MNL-model-prediction-error}
    E_h^k(s,a) = r_h^k(s,a) + \P_h^k\Vb_{h+1}^k(s,a) - \Qb_h^k(s,a),
\end{align}
then with probability at least $1-2\delta$, the following holds for all $k\in[K]$, $h\in[H]$ and $\forall s \in\S,a\in\A$,
\begin{align*}
    {}&-\Gamma_{h,\theta}^{k-1}-\frac{H}{\kappa}\Gamma_{h,\w}^{k-1} - 2\beta_\theta\norm{\phi(s,a)}_{\sbr{\Lambda_{h}^{k-1}}^{-1}}\\{}& - 2\frac{H}{\kappa}\betab_{\w}^{k-1}\max_{s^\prime \in \S_h^k}\norm{\psi(s^\prime\given s,a)}_{\sbr{\Sigmab_{h}^{k-1}}^{-1}}\leq E_h^k(s,a) \\{}&\leq \Gamma_{h,\theta}^{k-1}+\frac{H}{\kappa}\Gamma_{h,\w}^{k-1}.
\end{align*}
\end{myLemma}
And notice that $\forall s\in\S, a\in\A$, $\abs{E_h^k(s,a)}\leq 2H$, we have
    \begin{equation*}
        \begin{split}
            {}&\E_{\pi_{*,h}^k}\mbr{E_h^k(s_h^k,a_h^k)}-E_h^k(s_h^k,a_h^k)\\
            \leq{}& \min\Bigg\{4H,2\Gamma_{h,\theta}^{k-1}+2\frac{H}{\kappa}\Gamma_{h,\w}^{k-1} + 2\beta_\theta\norm{\phi(s_h^k,a_h^k)}_{\sbr{\Lambda_{h}^{k-1}}^{-1}}\\{}&\quad + 2\frac{H}{\kappa}\betab_{\w}^{k-1}\max_{s^\prime \in \S_h^k}\norm{\psi(s^\prime\given s_h^k,a_h^k)}_{\sbr{\Sigmab_{h}^{k-1}}^{-1}}\Bigg\}\\
            \leq{}& 2\Gamma_{h,\theta}^{k-1}+2\frac{H}{\kappa}\Gamma_{h,\w}^{k-1} +\min\bbr{4H, 2\beta_\theta\norm{\phi(s_h^k,a_h^k)}_{\sbr{\Lambda_{h}^{k-1}}^{-1}}}\\{}&\quad +\min\bbr{4H, 2\frac{H}{\kappa}\betab_{\w}^{k-1}\max_{s^\prime \in \S_h^k}\norm{\psi(s^\prime\given s_h^k,a_h^k)}_{\sbr{\Sigmab_{h}^{k-1}}^{-1}}}\\
            \leq{}&  2\Gamma_{h,\theta}^{k-1}+2\frac{H}{\kappa}\Gamma_{h,\w}^{k-1} +4H\beta_\theta\min\bbr{1, \norm{\phi(s_h^k,a_h^k)}_{\sbr{\Lambda_{h}^{k-1}}^{-1}}}\\{}&\quad +4\frac{H}{\kappa}\betab_{\w}^{k-1}\min\bbr{1, \max_{s^\prime \in \S_h^k}\norm{\psi(s^\prime\given s_h^k,a_h^k)}_{\sbr{\Sigmab_{h}^{k-1}}^{-1}}},
    \end{split}
    \end{equation*}
    By Lemma~\ref{lemma:LMDP-regret-decomposition}, we can further connect the dynamic regret to the model prediction error, we have with probability at least $1-4\delta$,
    \begin{equation*}
        \begin{split}
            {}&\DReg_T\leq \underbrace{4H\beta_\theta\sum_{k=1}^K\sum_{h=1}^{H}\min\bbr{1, \norm{\phi(s_h^k,a_h^k)}_{\sbr{\Lambda_{h}^{k-1}}^{-1}}} }_{\variance \mathtt{1}}\\{}&+\underbrace{4\frac{H}{\kappa}\sum_{k=1}^K\sum_{h=1}^{H}\betab_{\w}^{k-1}\min\bbr{1,\max_{s^\prime \in \S_h^k}\norm{\psi(s^\prime\given s_h^k,a_h^k)}_{\sbr{\Sigmab_{h}^{k-1}}^{-1}}}}_{\variance \mathtt{2}}\\
            {}& + \underbrace{2\sum_{k=1}^K\sum_{h=1}^{H}\Gamma_{h,\theta}^{k-1}+2\frac{H}{\kappa}\sum_{k=1}^K\sum_{h=1}^{H}\Gamma_{h,\w}^{k-1}}_{\bias} + 4H\sqrt{2T\log(1/\delta)}.
    \end{split}
    \end{equation*}
    \parag{Bias.} Now we set $w_{k, j}=\gamma^{k-j}, \gamma\in(0,1)$, same as Eq~\eqref{eq:LMDP-regret-bias-theta}, we have 
    \begin{align*}
            2\sum_{k=1}^K\sum_{h=1}^{H}\Gamma_{h,\theta}^{k-1} \leq 4 L_{\phi}^2\sqrt{\frac{d}{\lamt}}\frac{1}{(1-\gamma)^{3 / 2}} \sum_{p=1}^{K-1} \sum_{h=1}^{H}\norm{\theta_h^p-\theta_h^{p+1}}_2,
    \end{align*}
    then we set $\alpha_{k, j}=\gamma^{k-j}, \gamma\in(0,1)$, similar to Eq~\eqref{eq:LMDP-regret-bias-w}, we have
    \begin{equation*}
        \begin{split}
            {}&2\frac{H}{\kappa}\sum_{k=1}^K\sum_{h=1}^{H}\Gamma_{h,\w}^{k-1}\\ \leq{}& 4\frac{H}{\kappa}L_{\psi}^2\sqrt{\frac{d}{\lamw}}\frac{1}{(1-\gamma)^{3 / 2}}\sum_{p=1}^{K-1} \sum_{h=1}^{H}\norm{\w_h^p-\w_h^{p+1}}_2,
    \end{split}
    \end{equation*}
    \parag{Variance.~}
    Same as Eq~\eqref{eq:LMDP-potential-theta}, we have 
\begin{equation*}
    \begin{split}
        {}&4H\beta_\theta\sum_{k=1}^K\sum_{h=1}^{H}\min\bbr{1, \norm{\phi(s_h^k,a_h^k)}_{\sbr{\Lambda_{h}^{k-1}}^{-1}}} \\ \leq{}& 4H\beta_\theta\sqrt{KH}\sqrt{H2d\sbr{K \log\frac{1}{\gamma}+\log\sbr{1+\frac{L_\phi^2 }{\lamt d(1-\gamma)}}}}.
\end{split}
\end{equation*}
For the second term, 
\begin{equation*}
    \begin{split}
        {}&4\frac{H}{\kappa}\sum_{k=1}^K\sum_{h=1}^{H}\betab_{\w}^{k-1}\min\bbr{1,\max_{s^\prime \in \S_h^k}\norm{\psi(s^\prime\given s_h^k,a_h^k)}_{\sbr{\Sigmab_{h}^{k-1}}^{-1}}}\\ \leq{}& 4\frac{H}{\kappa} \betab_{\w}^{K}\sum_{k=1}^K\sum_{h=1}^{H} \min\bbr{1,\max_{s^\prime \in \S_h^k}\norm{\psi(s^\prime\given s_h^k,a_h^k)}_{\sbr{\Sigmab_{h}^{k-1}}^{-1}}}\\
        \leq{}&  4\frac{H^2}{\kappa} \betab_{\w}^{K}\sqrt{K}\sqrt{\sum_{k=1}^K\min\bbr{1,\max_{s^\prime \in \S_h^k}\norm{\psi(s^\prime\given s_h^k,a_h^k)}_{\sbr{\Sigmab_{h}^{k-1}}^{-1}}^2}}\\
        \leq{}&  4\frac{H^2}{\kappa} \betab_{\w}^{K}\sqrt{K}\sqrt{\sum_{k=1}^K\min\bbr{1,\sum_{s^\prime \in \S_h^k}\norm{\psi(s^\prime\given s_h^k,a_h^k)}_{\sbr{\Sigmab_{h}^{k-1}}^{-1}}^2}}.
\end{split}
\end{equation*}
Based on the Lemma~\ref{lemma:potential-lemma} (Potential Lemma), we know that $\forall h\in[H]$, we have 
\begin{equation*}
    \begin{split}
        {}&\sum_{k=1}^K \min\bbr{1,\max_{s^\prime \in \S_h^k}\norm{\psi(s^\prime\given s_h^k,a_h^k)}_{\sbr{\Sigmab_{h}^{k-1}}^{-1}}^2} \\ \leq{}& 2d\sbr{K \log\frac{1}{\gamma}+\log\sbr{1+\frac{UL_\psi^2 }{\lamw d(1-\gamma)}}}, 
\end{split}
\end{equation*}
so we have
\begin{equation*}
    \begin{split}
        {}&4\frac{H}{\kappa}\sum_{k=1}^K\sum_{h=1}^{H}\betab_{\w}^{k-1}\min\bbr{1,\max_{s^\prime \in \S_h^k}\norm{\psi(s^\prime\given s_h^k,a_h^k)}_{\sbr{\Sigmab_{h}^{k-1}}^{-1}}}\\ \leq{}& 4\frac{H^2}{\kappa}\betab_{\w}^{K}\sqrt{K}\sqrt{2d\sbr{K \log\frac{1}{\gamma}+\log\sbr{1+\frac{UL_\psi^2 }{\lamw d(1-\gamma)}}}}.
\end{split}
\end{equation*}
Since there is a term $HK \sqrt{\log (1 / \gamma)}$ in the regret bound, we cannot let $\gamma$ close to 0 , so we set $\gamma \geq 1 / K$ and have $\log (1 / \gamma) \leq C(1-\gamma)$, where $C=\log K /(1-1 / K)$. We set $\lambda_\theta = d$, and $\lambda_\w = d$. Combining the upper bounds of the bias and variance parts and with confidence level $\delta=1 /(4 T)$, by union bound we have the following dynamic regret bound with probability at least $1-1 / T$,
\begin{equation*}
    \begin{split}
        &\DReg_T\leq \O\Bigg(\frac{1}{(1-\gamma)^{3 / 2}} P_T^\theta + \frac{H}{\kappa}\frac{1}{(1-\gamma)^{3 / 2}}P_T^\w\\{}&\hspace{3em} + dH^2K\sqrt{1-\gamma}+\frac{dH^2K}{\kappa}\sqrt{1-\gamma} + H^{3/2}d\sqrt{HK}\Bigg)\\
        &\leq{} \O\sbr{\frac{Hd}{\kappa}\sbr{\frac{1}{(1-\gamma)^{3 / 2}} \Delta + HK\sqrt{1-\gamma}}+ H^{3/2}d\sqrt{HK}}.
\end{split}
\end{equation*}
Furthermore, by setting the discounted factor optimally as $\gamma = 1-\max\bbr{1/K,\sqrt{\Delta/T}}$, we have
\begin{equation*}
    \DReg_T \leq \begin{cases}\Ot\sbr{\kappa^{-1}Hd \Delta^{1/4}T^{3/4}} & \text { when } \Delta \geq H / K, \\ \Ot\sbr{\kappa^{-1}dH^{3/2}\sqrt{T}}& \text { when } \Delta<H / K .\end{cases}
    \end{equation*}
\end{proof}

\begin{proof}[Proof of Lemma~\ref{lemma:MNL-model-prediction-error}]
    We first consider the upper bound of $E_h^k$, based on the definition of $\Qb_h^k$~\eqref{eq:LMDP-optimistic-Q} and model assumption~\eqref{eq:LMDP-model} and Eq.~\eqref{eq:LMDP-linear-transtion}, we have $\forall a\in\A,s\in\S$, 
    \begin{equation*}
        \begin{split}
            {}& r_h^k(s, a)+\mbr{\P_h^k \Vb_{h+1}^{k}}(s,a) - \Qb_h^{k}(s,a)\\
            ={}& r_h^k(s, a)+\mbr{\P_h^k \Vb_{h+1}^{k}}(s,a) - \phi(s,a)^\T\th_{h}^k\\{}&\quad-\beta_\theta\norm{\phi(s,a)}_{\sbr{\Lambda_{h}^{k-1}}^{-1}} - [\Pt_h^k\Vb_{h+1}^k](s,a)\\{}&\quad-\frac{H}{\kappa}\betab_{\w}^{k-1}\max_{s^\prime \in \S_h^k}\norm{\psi(s^\prime\given s,a)}_{\sbr{\Sigmab_{h}^{k-1}}^{-1}}\\
            ={}&\phi(s,a)^\T\sbr{\theta_{h}^k - \th_{h}^k} -\beta_\theta\norm{\phi(s,a)}_{\sbr{\Lambda_{h}^{k-1}}^{-1}} \\{}&\quad+ \sbr{\mbr{\P_h^k \Vb_{h+1}^k}(s,a) - \mbr{\Pt_h^k \Vb_{h+1}^k}(s,a)}\\{}&\quad-\frac{H}{\kappa}\betab_{\w}^{k-1}\max_{s^\prime \in \S_h^k}\norm{\psi(s^\prime\given s,a)}_{\sbr{\Sigmab_{h}^{k-1}}^{-1}}\\
            \leq{}&\Gamma_{h,\theta}^{k-1}+\frac{H}{\kappa}\Gamma_{h,\w}^{k-1},
        \end{split}
    \end{equation*}
    where the last inequality comes from Lemma~\ref{lemma:LMDP-r-ucb} and Lemma~\ref{lemma:MNL-v-UCB}. Similarly, we can get the lower bound of $E_h^k$, $\forall a\in\A,s\in\S$, 
    \begin{equation*}
        \begin{split}
            {}& \Qb_h^{k}(s,a) - r_h^k(s, a)-\mbr{\P_h^k \Vb_{h+1}^{k}}(s,a)\\
            ={}&   \phi(s,a)^\T\th_{h}^k+\beta_\theta\norm{\phi(s,a)}_{\sbr{\Lambda_{h}^{k-1}}^{-1}} + [\Pt_h^k\Vb_{h+1}^k](s,a)\\{}&\quad+\frac{H}{\kappa}\betab_{\w}^{k-1}\max_{s^\prime \in \S_h^k}\norm{\psi(s^\prime\given s,a)}_{\sbr{\Sigmab_{h}^{k-1}}^{-1}}\\{}&\quad - r_h^k(s, a)-\mbr{\P_h^k \Vb_{h+1}^{k}}(s,a) \\
            ={}&\phi(s,a)^\T\sbr{\th_{h}^k - \theta_{h}^k} + \sbr{\mbr{\Pt_h^k \Vb_{h+1}^k}(s,a)-\mbr{\P_h^k \Vb_{h+1}^k}(s,a)}\\{}&\quad+\beta_\theta\norm{\phi(s,a)}_{\sbr{\Lambda_{h}^{k-1}}^{-1}}\\{}&\quad + \frac{H}{\kappa}\betab_{\w}^{k-1}\max_{s^\prime \in \S_h^k}\norm{\psi(s^\prime\given s,a)}_{\sbr{\Sigmab_{h}^{k-1}}^{-1}}\\
            \leq{}&\Gamma_{h,\theta}^{k-1}+\frac{H}{\kappa}\Gamma_{h,\w}^{k-1} + 2\beta_\theta\norm{\phi(s,a)}_{\sbr{\Lambda_{h}^{k-1}}^{-1}}\\{}&\quad+ 2\frac{H}{\kappa}\betab_{\w}^{k-1}\max_{s^\prime \in \S_h^k}\norm{\psi(s^\prime\given s,a)}_{\sbr{\Sigmab_{h}^{k-1}}^{-1}},
        \end{split}
    \end{equation*}
    thus completes the proof.
\end{proof}
\fi
\section{Technical Lemmas}
\label{sec:thechnical-lemmas}
In this section, we provide several useful lemmas, mainly about concentrations\iffullversion, and some derivatives of self-concordant property\fi.

\subsection{Concentration inequalities}

\begin{myLemma}[Hoeffding's Lemma]
  \label{lemma:Hoeffding}
  Let $Z$ be a real random variable such that $Z\in[a,b]$ almost surely. Then
  \begin{equation*}
      \forall \lambda \in \R \quad \E\mbr{e^{\lambda (Z-\E[Z])}} \leq \exp \left(\frac{\lambda^2 (b-a)^2}{8}\right),
  \end{equation*}
  or variable $(Z-\E[Z])$ is $\frac{(b-a)}{2}$-sub-Gaussian.
\end{myLemma}

\begin{myLemma}[Azuma-Hoeffding inequality]
  \label{lemma:Azuma}
  Let $M>0$ be a constant. Let $\bbr{x_i}_{i=1}^n$ be a martingale difference sequence with respect to a filtration $\bbr{\mathcal{G}_i}_i$ $(\E\mbr{x_i \mid \mathcal{G}_i}=0$ a.s. and $x_i$ is $\mathcal{G}_{i+1}$-measurable$)$ such that for all $i \in[n],\left|x_i\right| \leq M$ holds almost surely. Then, for any $0<\delta<1$, with probability at least $1-\delta$, we have
  $$
  \sum_{i=1}^n x_i \leq M \sqrt{2 n \log (1 / \delta)}.
  $$
\end{myLemma}

\begin{myThm}[Self-normalized concentration inequality for linear bandits~{\cite[Theorem 1]{NIPS'11:AY-linear-bandits}}]\label{thm:snc-AY}
  Let $\{F_t\}_{t=0}^\infty$ be a filtration. Let $\{\eta_t\}_{t=0}^\infty$ be a real-valued stochastic process such that $\eta_t$ is $F_t$-measurable and $\eta_t$ is conditionally $R$-sub-Gaussian for some $R\geq 0$ i.e., $\forall \lambda \in \R$, $\E\mbr{e^{\lambda \eta_t} \mid F_{t-1}} \leq \exp (\frac{\lambda^2 R^2}{2})$. Let $\{X_t\}_{t=1}^\infty$ be an $\R^d$-valued stochastic process such that $X_t$ is $F_{t-1}$-measurable. Assume that V is a $d\times d$ positive definite matrix. For any  $t\geq 0$, define
 $$V_t = V_0 + \sum_{s=1}^t X_sX_s^\T, \quad\quad S_t = \sum_{s=1}^t\eta_sX_s.$$
 Then, for any $\delta > 0$, with probability at least $1-\delta$, for all $t\geq 0$,
 $$\|S_t\|_{V_{t}^{-1}} \leq \sqrt{2R^2\log\left(\frac{\det(V_{t})^{\frac{1}{2}}\det(V_0)^{-\frac{1}{2}}}{\delta}\right)}.$$
\end{myThm}

\begin{myThm}[Self-normalized concentration inequality for self-concordant bandits~{\cite[Theorem 1]{ICML'20:logistic-bandits}}]
  \label{thm:self-normalized-weight-SCB}
 Let $\bbr{\F_t}_{t=0}^\infty$ be a filtration. Let $\bbr{\eta_t}_{t=0}^\infty$ be a martingale difference sequence such that $\eta_{t}$ is $\F_{t}$ measurable. Let $\bbr{X_t}_{t=0}^\infty$ be a stochastic process on $\R^d$ such that $X_t$ is $\F_{t-1}$ measurable and $\norm{X_t}_2 \leq 1$.  Furthermore, assume that conditionally on $\F_t$ we have $\abs{\eta_{t}} \leq 1$ a.s., and denote $\sigma_t^2=\mathbb{E}\mbr{\eta_{t}^2 \given \F_{t-1}}$. For any $t \geq 0$, define 
 $$H_t=\sum_{s=1}^{t} \sigma_s^2 X_s X_s^{\top}+\lambda I_d,\qquad S_t=\sum_{s=1}^{t}\eta_{s} X_s,$$ 
 with $\lambda >0$. Then, for any $\delta >0$, with probability at least $1-\delta$, for all $t\geq 0$,
\begin{equation}\nonumber
\norm{S_t}_{H_t^{-1}} \leq \frac{\sqrt{\lambda}}{2}+\frac{2}{\sqrt{\lambda}} \log \sbr{\frac{\operatorname{det}\sbr{H_t}^{1 / 2}}{\delta \lambda^{d / 2}}}+\frac{2}{\sqrt{\lambda}} d \log (2).
\end{equation}
\end{myThm}

\begin{myLemma}
  \label{lemma:d}
  Suppose $U_0 = \lambda I_d$, $U_t = U_{t-1} + A_tA_t^{\T}$, and $A_t\in\R^d$, then
  \begin{equation}\label{eq:d1}
      \begin{split}
          \forall p \in[t-1], \quad \sum_{s=1}^{p}\norm{A_s}_{U_{t-1}^{-1}}^2 \leq d.
      \end{split}
  \end{equation}
\end{myLemma}
\begin{proof}[Proof of Lemma~\ref{lemma:d}]
  \begin{align*}
    & \sum_{s=1}^{p} \norm{A_s}_{U_{t-1}^{-1}}^2 = \sum_{s=1}^{p} \mathrm{Tr}(A_s^\T U_{t-1}^{-1} A_s) 
     =  \mathrm{Tr}\big(U_{t-1}^{-1} \sum_{s=1}^{p} A_s A_s^\T\big) \\
     & \leq \mathrm{Tr}\big(U_{t-1}^{-1} \sum_{s=1}^{p} A_s A_s^{\T}\big) + \mathrm{Tr}\big(U_{t-1}^{-1} \sum_{s=p+1}^{t-1} A_s A_s^{\T}\big)\\ 
     {} & \qquad + \mathrm{Tr}\big(U_{t-1}^{-1} \lambda \sum_{i=1}^{d} \mathbf{e}_i \mathbf{e}_i^{\T}\big) = \mathrm{Tr}(I_d) = d.
   \end{align*} 
\end{proof}

\begin{myLemma}[Determinant inequality]
  \label{lemma:det-inequality} We let the $V_{t} = \sum_{s=1}^{t}w_{t,s} X_s X_s^\T + \lambda I_d, V_0 = \lambda I_d$. Assume $\|\x\|_2\leq L$ and we have,
  \begin{equation}\nonumber
  \begin{split}
      \det(V_{t})  \leq \sbr{\lambda + \frac{L^2\sum_{s=1}^t w_{t,s}}{d}}^d.
  \end{split} 
  \end{equation}
\end{myLemma}

\begin{proof}
  Now we have $V_{t} = \sum_{s=1}^{t}w_{t,s} X_s X_s^\T + \lambda I_d$, take the trace on both sides, and get the upper bound of $\mathrm{Tr}(V_{t})$
  \begin{equation}
  \begin{split}
  \label{eq:bound-trace-LB}
      \mathrm{Tr}(V_{t}) ={}&  \mathrm{Tr}( \lambda I_d) + \sum_{s=1}^{t}w_{t,s}\mathrm{Tr}\sbr{ X_s X_s^\T}\\
       ={}& \lambda d + \sum_{s=1}^{t}w_{t,s}\|X_s\|_2^2 \leq \lambda  d + L^2\sum_{s=1}^{t}w_{t,s}.
  \end{split} 
  \end{equation}
  Base on the definition of determinant and the upper bound of $\mathrm{Tr}(V_{t})$~\eqref{eq:bound-trace-LB}, we can get the upper bound for $\det(V_{t})$,
  \begin{equation}\nonumber
  \begin{split}
      \det(V_{t}) =  \prod_{i =1}^d \lambda_i \leq{}& \sbr{\frac{\sum_{i=1}^d \lambda_i}{d}}^d = \sbr{\frac{\mathrm{Tr}(V_{t})}{d}}^d \\\leq{}& \sbr{\lambda + \frac{L^2\sum_{s=1}^t w_{t,s}}{d}}^d.
  \end{split} 
  \end{equation}
\end{proof}

\iffullversion
\subsection{Self-Concordant Properties}
Based on the generalized self-concordant property of the (inverse) link function $\mu(\cdot)$, we have the following lemma, which will be later used to derive Lemma~\ref{lemma:SCB-G-H}.
\begin{myLemma}[{Lemma 9 of \cite{ICML'20:logistic-bandits}}]
  \label{lemma:self-concordant-properties}
  For any $z_1, z_2\in \R$, we have the following inequality:
  \begin{equation*}
    \begin{split}
      \dmu(z_1)\frac{1-\exp(-|z_1-z_2|)}{|z_1-z_2|} &\leq \int_{0}^1 \dmu(z_1+v(z_2-z_1))\diff{v} \\&\leq \dmu(z_1)\frac{\exp(|z_1-z_2|)-1}{|z_1-z_2|}.
    \end{split}
  \end{equation*}
  Furthermore, $\int_{0}^1 \dmu(z_1+v(z_2-z_1))\diff{v} \geq \dmu(z_1)(1+|z_1-z_2|)^{-1}$.
\end{myLemma}
The following lemma provides a weighted version of Lemma 10 of~{\cite{ICML'20:logistic-bandits}} which can be easily proven.
\begin{myLemma}
  \label{lemma:SCB-G-H}
  With $G_t$ defined in~\eqref{eq:SCB-gt-mvt} and $H_t$ defined in~\eqref{eq:SCB-H}, the following inequalities hold
  \begin{equation*}
    \begin{split}
      \forall \theta_1,\theta_2 \in \Theta,\quad G_t(\theta_1,\theta_2) &\geq (1+2S)^{-1}H_t(\theta_1),\\\quad G_t(\theta_1,\theta_2) &\geq (1+2S)^{-1}H_t(\theta_2).
    \end{split}
  \end{equation*}
\end{myLemma}

\blue{

\begin{myLemma}[Lemma 7 of \cite{arXiv'21:restartucb}]
  \label{lemma:bounded-noise}
  Denote by $L_i$ the absolute value of cumulative rewards for episode $i$, i.e., $L_i \triangleq \abs{\sum_{t=(i-1) \Delta+1}^{i \Delta} r_t\left(X_t\right)}$, then

  $$
  \operatorname{Pr}\left[\forall i \in[\lceil T / \Delta\rceil], L_i \leq L S \Delta+2 R \sqrt{\Delta \ln \frac{T}{\sqrt{\Delta}}}\right] \geq 1-\frac{2}{T}.
  $$
\end{myLemma}
}
\fi

\iffullversion
\blue{

\section{Bandits over Bandits}
\label{sec:BOB}
\subsection{BOB Algorithm}
We divide the $T$ rounds into equal-length episodes of size $\Delta$, such that $\Delta=\lceil d \sqrt{T}\rceil$. In each episode, we run \LBweightours with different discounted factors $\gamma$. Specifically, the $\gamma$ comes from the candidate set $\W$,
\begin{equation}\label{eq:candidate-set}
    \begin{split}
        \W = \bbr{\gamma_i = 1 -  d^{-\frac{1}{2}} 2^{1-i}  \given i \in [N] },
    \end{split}
\end{equation}
where $N=\lceil \log_2(T/\sqrt{d})\rceil+1$ is the number of candidate values, and recall that $S$ is the upper bound on the norm of the underlying regression parameters.

To adaptively select the optimal $\gamma$, we model the selection procedure as an adversarial bandit problem. In this formulation, each step corresponds to one episode of length $\Delta$, such that has total $\lceil T/\Delta \rceil$ rounds. At each step, a bandit algorithm selects a $\gamma$ from the candidate set $\W$, observes the cumulative reward from the \LBweightours with corresponding discounted factor, and uses it as feedback to update its selection strategy.

Let $\gamma_{\min }\left(\gamma_{\max }\right)$ be the minimal (maximal) discounted factor in the candidate set $\W$, then it is evident to verify that
\begin{equation}\nonumber
    \begin{split}
        \gamma_{\min }=1 - \frac{1}{\sqrt{d}},\quad 1>\gamma_{\max}\geq 1 - \frac{1}{T}.
    \end{split}
\end{equation}

\subsection{Regret Analysis}
\begin{myThm}\label{thm:bob}
    \LBweightours without the knowledge of path-length $P_T$, together with Bandits-over-Bandits mechanism satisfies with probability at least $1-3/T$, 
    \begin{equation*}
        \DReg_T = \sum_{t=1}^T \max_{\x\in\X}\x^\T \theta_t - \sum_{t=1}^T X_t^\T \theta_t = \Ot\sbr{d^{\sfrac{3}{4}} P_T^{\sfrac{1}{4}} T^{\sfrac{3}{4}}}.
    \end{equation*}
\end{myThm}
\begin{proof}[Proof of Theorem~\ref{thm:bob}]
We begin by decomposing the dynamic regret. Let $X_t^* \define\argmax_{\x \in \X} \x^\T \theta_t$ and we have 
\begin{equation}\nonumber
    \begin{split}
        \DReg_T = {}&  \sum_{t=1}^T \sbr{X_t^{*\T} \theta_t - X_t^\T \theta_t} \\= {}& \underbrace{\sum_{t=1}^T X_t^{*\T} \theta_t - \sum_{i=1}^{\lceil T / \Delta\rceil} \sum_{t=(i-1) \Delta+1}^{i \Delta} X_t^{\gammabob \T} \theta_t}_{\mathtt{base}\text{-}\mathtt{regret}} \\&+ \underbrace{\sum_{i=1}^{\lceil T / \Delta\rceil} \sum_{t=(i-1) \Delta+1}^{i \Delta} \sbr{X_t^{\gammabob \T} \theta_t - X_t^{\gamma_i \T} \theta_t}}_{\mathtt{meta}\text{-}\mathtt{regret}},
    \end{split}
\end{equation}
where $\gammabob$ is the best discounted factor in the candidate set to approximate the optimal discounted factor  $\gamma^* = 1- \max\{1/T, \sqrt{P_T/(dT)}\}$.

\parag{Base-regret.} Based on \eqref{eq:dregret-lb}, and the union bound, we have with probability at least $1-2N\delta$,
\begin{align}
        \mathtt{base}\text{-}\mathtt{regret} ={}& \sum_{t=1}^T X_t^{*\T} \theta_t - \sum_{i=1}^{\lceil T / \Delta\rceil} \sum_{t=(i-1) \Delta+1}^{i \Delta} X_t^{\gammabob \T} \theta_t\nonumber\\
        \leq{}& \sum_{i=1}^{\lceil T / \Delta\rceil} \Ot\sbr{\frac{1}{(1-\gammabob)^{\sfrac{3}{2}}}P_i + d(1-\gammabob)^{\sfrac{1}{2}}\Delta}\nonumber\\
        \leq{}& \Ot\sbr{\frac{1}{(1-\gammabob)^{\sfrac{3}{2}}}P_T + d(1-\gammabob)^{\sfrac{1}{2}}T}\nonumber\\
        \leq{}& \Ot\sbr{\frac{1}{2^{\sfrac{3}{2}}(1-\gamma^*)^{\sfrac{3}{2}}}P_T + d(1-\gamma^*)^{\sfrac{1}{2}}T}\nonumber\\
        \leq{}& \Ot(d^{\sfrac{3}{4}}P_T^{\sfrac{1}{4}}T^{\sfrac{3}{4}}).\label{eq:base-regret-bob}
    \end{align}
    We know that $\gamma^* = 1- \max\{1/T, \sqrt{P_T/(dT)}\}$ is the optimal discounted factor. Since $\gamma^* \in [\gamma_{\min}, \gamma_{\max}]$, the candidate set $\mathcal{W}$ covers $\gamma^*$. Furthermore, due to the geometric spacing of $\mathcal{W}$, there exists some $\gammabob \in \mathcal{W}$ such that
    \begin{equation}\nonumber
        \begin{split}
            1- \gammabob \leq 1- \gamma^* \leq 2(1-\gammabob).
        \end{split}
        \end{equation}

\parag{Meta-regret.} The analysis of meta-regret follows the proof for window-based algorithm~\cite[Proposition 1]{arXiv'19:window-LB}.

\begin{equation}\nonumber
    \begin{split}
        {}&\mathtt{meta}\text{-}\mathtt{regret} \\={}& \sum_{i=1}^{\lceil T / \Delta\rceil} \sum_{t=(i-1) \Delta+1}^{i \Delta} \sbr{X_t^{\gammabob \T} \theta_t - X_t^{\gamma_i \T} \theta_t}\\
        ={}& \sum_{i=1}^{\lceil T / \Delta\rceil} \sum_{t=(i-1) \Delta+1}^{i \Delta} \sbr{X_t^{\gammabob \T} \theta_t + \eta_t - X_t^{\gamma_i \T} \theta_t - \eta_t}\\
        ={}& \sum_{i=1}^{\lceil T / \Delta\rceil} \sum_{t=(i-1) \Delta+1}^{i \Delta} \sbr{r_t^{\gammabob} - r_t^{\gamma_i}}\\
        ={}& \sum_{i=1}^{\lceil T / \Delta\rceil} \sbr{L_i^{\gammabob} - L_i^{\gamma_i}},
    \end{split}
    \end{equation}
    Based on Lemma~\ref{lemma:bounded-noise}, we know that with probability at least $1-\frac{2}{T}$, we have $\forall i \in[\lceil T/\Delta\rceil], L_i \leq L S \Delta+2 R \sqrt{\Delta \ln \frac{T}{\sqrt{\Delta}}}$ we define $L_{\max} \define L S \Delta+2 R \sqrt{\Delta \ln \frac{T}{\sqrt{\Delta}}}$.  We choose Exp3.IX~\cite{NIPS'15:Neu} as the meta algorithm. Then we have with probability at least $1-\delta$, 
    \begin{equation}\label{eq:meta-regret-bob}
    \begin{split}
        \sum_{i=1}^{\lceil T / \Delta\rceil} \sbr{L_i^{\gammabob} - L_i^{\gamma_i}} &= \Ot\sbr{L_{\max}\sqrt{\frac{T}{\Delta}N}}\\ &= \Ot\sbr{\sqrt{\Delta T N}} = d^{\sfrac{1}{2}} T^{\sfrac{3}{4}}.
    \end{split}
    \end{equation}
    Combining the upper bounds of base-regret~\eqref{eq:base-regret-bob} and meta-regret~\eqref{eq:meta-regret-bob}, by the union bound, and let $\delta = \frac{1}{(2N+1)T}$, we have with probability at least $1-3/T$, 
    \begin{equation}\nonumber
        \begin{split}
            \DReg_T =\Ot\sbr{d^{\sfrac{3}{4}} P_T^{\sfrac{1}{4}} T^{\sfrac{3}{4}}+ d^{\sfrac{1}{2}} T^{\sfrac{3}{4}}}.
        \end{split}
    \end{equation}
    Thus we complete the proof.
\end{proof}}
\fi
}

\section*{Acknowledgments}
This research was supported by NSFC (62361146852, 62206125). Peng Zhao was supported in part by the Xiaomi Foundation.
\bibliographystyle{IEEEtran}
\bibliography{online_learning}

\begin{thebibliography}{10}
\providecommand{\url}[1]{#1}
\csname url@samestyle\endcsname
\providecommand{\newblock}{\relax}
\providecommand{\bibinfo}[2]{#2}
\providecommand{\BIBentrySTDinterwordspacing}{\spaceskip=0pt\relax}
\providecommand{\BIBentryALTinterwordstretchfactor}{4}
\providecommand{\BIBentryALTinterwordspacing}{\spaceskip=\fontdimen2\font plus
\BIBentryALTinterwordstretchfactor\fontdimen3\font minus
  \fontdimen4\font\relax}
\providecommand{\BIBforeignlanguage}[2]{{%
\expandafter\ifx\csname l@#1\endcsname\relax
\typeout{** WARNING: IEEEtran.bst: No hyphenation pattern has been}%
\typeout{** loaded for the language `#1'. Using the pattern for}%
\typeout{** the default language instead.}%
\else
\language=\csname l@#1\endcsname
\fi
#2}}
\providecommand{\BIBdecl}{\relax}
\BIBdecl

\bibitem{AISTATS'19:window-LB}
W.~C. Cheung, D.~Simchi-Levi, and R.~Zhu, ``Learning to optimize under
  non-stationarity,'' in \emph{Proceedings of the 22nd International Conference
  on Artificial Intelligence and Statistics (AISTATS)}, 2019, pp. 1079--1087.

\bibitem{NIPS'19:weighted-LB}
Y.~Russac, C.~Vernade, and O.~Capp\'{e}, ``Weighted linear bandits for
  non-stationary environments,'' in \emph{Advances in Neural Information
  Processing Systems 32 (NeurIPS)}, 2019, pp. 12\,040--12\,049.

\bibitem{AISTATS'20:restart}
P.~Zhao, L.~Zhang, Y.~Jiang, and Z.-H. Zhou, ``A simple approach for
  non-stationary linear bandits,'' in \emph{Proceedings of the 23rd
  International Conference on Artificial Intelligence and Statistics
  (AISTATS)}, 2020, pp. 746--755.

\bibitem{arXiv'20:NS-GLB}
Y.~Russac, O.~Capp{\'e}, and A.~Garivier, ``Algorithms for non-stationary
  generalized linear bandits,'' \emph{ArXiv preprint}, vol. arXiv:2003.10113,
  2020.

\bibitem{UAI'20:kim20a}
B.~Kim and A.~Tewari, ``Randomized exploration for non-stationary stochastic
  linear bandits,'' in \emph{Proceedings of the 36th Conference on Uncertainty
  in Artificial Intelligence (UAI)}, 2020, pp. 71--80.

\bibitem{arXiv'21:faury-driftingGLB}
L.~Faury, Y.~Russac, M.~Abeille, and C.~Calauz{\`e}nes, ``Regret bounds for
  generalized linear bandits under parameter drift,'' \emph{ArXiv preprint},
  vol. arXiv:2103.05750, 2021.

\bibitem{AISTATS'21:SCB-forgetting}
Y.~Russac, L.~Faury, O.~Capp{\'e}, and A.~Garivier, ``Self-concordant analysis
  of generalized linear bandits with forgetting,'' in \emph{Proceedings of the
  24th International Conference on Artificial Intelligence and Statistics
  (AISTATS)}, 2021, pp. 658--666.

\bibitem{COLT'21:black-box}
C.-Y. Wei and H.~Luo, ``Non-stationary reinforcement learning without prior
  knowledge: An optimal black-box approach,'' in \emph{Proceedings of the 34th
  Conference on Learning Theory (COLT)}, 2021, pp. 4300--4354.

\bibitem{AISTATS'22:weighted-GPB}
Y.~Deng, X.~Zhou, B.~Kim, A.~Tewari, A.~Gupta, and N.~Shroff, ``Weighted
  {G}aussian process bandits for non-stationary environments,'' in
  \emph{Proceedings of the 25th International Conference on Artificial
  Intelligence and Statistics (AISTATS)}, 2022, pp. 6909--6932.

\bibitem{arXiv'22:VanRoy}
Y.~Liu, B.~Van~Roy, and K.~Xu, ``Nonstationary bandit learning via predictive
  sampling,'' in \emph{Proceedings of the 26th International Conference on
  Artificial Intelligence and Statistics (AISTATS)}, 2023, pp. 6215--6244.

\bibitem{AISTATS'23:revisiting-weighted}
J.~Wang, P.~Zhao, and Z.-H. Zhou, ``Revisiting weighted strategy for
  non-stationary parametric bandits,'' in \emph{Proceedings of the 26th
  International Conference on Artificial Intelligence and Statistics
  (AISTATS)}, vol. 206, 2023, pp. 7913--7942.

\bibitem{MLJ'21:TSnonstationary}
S.~Tomkins, P.~Liao, P.~V. Klasnja, and S.~A. Murphy, ``Intelligentpooling:
  practical {T}hompson sampling for mhealth,'' \emph{Machine Learning}, vol.
  110, no.~9, pp. 2685--2727, 2021.

\bibitem{AISTATS'21:change-preference}
W.~Huleihel, S.~Pal, and O.~Shayevitz, ``Learning user preferences in
  non-stationary environments,'' in \emph{Proceedings of the 24th International
  Conference on Artificial Intelligence and Statistics (AISTATS)}, 2021, pp.
  1432--1440.

\bibitem{COLT'20:ChiJin}
C.~Jin, Z.~Yang, Z.~Wang, and M.~I. Jordan, ``Provably efficient reinforcement
  learning with linear function approximation,'' in \emph{Proceedings of 33rd
  Conference on Learning Theory (COLT)}, 2020, pp. 2137--2143.

\bibitem{arXiv'20:Touati}
A.~Touati and P.~Vincent, ``Efficient learning in non-stationary linear markov
  decision processes,'' \emph{ArXiv preprint}, vol. arXiv:2010.12870, 2020.

\bibitem{AAAI'23:MNL_MDP}
T.~Hwang and M.~Oh, ``Model-based reinforcement learning with multinomial
  logistic function approximation,'' in \emph{Proceedings of the 37th {AAAI}
  Conference on Artificial Intelligence (AAAI)}, 2023, pp. 7971--7979.

\bibitem{NIPS'24:MNL_MDP_efficient}
L.-F. Li, Y.-J. Zhang, P.~Zhao, and Z.-H. Zhou, ``Provably efficient
  reinforcement learning with multinomial logit function approximation,'' in
  \emph{Advances in Neural Information Processing Systems 37 (NeurIPS)}, 2024,
  pp. 58\,539--58\,573.

\bibitem{arxiv'22:ns_lin_mdp}
H.~Zhong, Z.~Yang, Z.~Wang, and C.~Szepesv{\'{a}}ri, ``Optimistic policy
  optimization is provably efficient in non-stationary mdps,'' \emph{ArXiv
  preprint}, vol. arXiv: 2110.08984, 2022.

\bibitem{NIPS'24:ALM_MDP_dynamic_regret}
L.-F. Li, P.~Zhao, and Z.-H. Zhou, ``Near-optimal dynamic regret for
  adversarial linear mixture mdps,'' in \emph{Advances in Neural Information
  Processing Systems 37 (NeurIPS)}, 2024, pp. 55\,858--55\,883.

\bibitem{NIPS'14:besbes}
O.~Besbes, Y.~Gur, and A.~Zeevi, ``Stochastic multi-armed-bandit problem with
  non-stationary rewards,'' \emph{Advances in Neural Information Processing
  Systems 27 (NeurIPS)}, 2014.

\bibitem{JMLR'24:sword++}
P.~Zhao, Y.~Zhang, L.~Zhang, and Z.-H. Zhou, ``Adaptivity and non-stationarity:
  Problem-dependent dynamic regret for online convex optimization,''
  \emph{Journal of Machine Learning Research}, vol.~25, pp. 98:1--98:52, 2024.

\bibitem{NSR'22:OpenML}
Z.-H. Zhou, ``{Open-environment machine learning},'' \emph{National Science
  Review}, vol.~9, no.~8, 2022.

\bibitem{COLT'10:concept-drift}
K.~Crammer, Y.~Mansour, E.~Even{-}Dar, and J.~W. Vaughan, ``Regret minimization
  with concept drift,'' in \emph{Proceedings of the 23rd Conference on Learning
  Theory (COLT)}, 2010, pp. 168--180.

\bibitem{COLT'13:Chiang}
C.~Chiang, C.~Lee, and C.~Lu, ``Beating bandits in gradually evolving worlds,''
  in \emph{Proceedings of the 26th Annual Conference on Learning Theory
  (COLT)}, 2013, pp. 210--227.

\bibitem{ACM'14:survey-concept-drift}
J.~Gama, I.~Zliobaite, A.~Bifet, M.~Pechenizkiy, and A.~Bouchachia, ``A survey
  on concept drift adaptation,'' \emph{ACM Computing Surveys}, vol.~46, no.~4,
  pp. 44:1--44:37, 2014.

\bibitem{TIT'24:smooth-change}
A.~Ghosh, A.~Sankararaman, K.~Ramchandran, T.~Javidi, and A.~Mazumdar,
  ``Competing bandits in non-stationary matching markets,'' \emph{IEEE
  Transactions on Information Theory}, vol.~70, no.~4, pp. 2831--2850, 2024.

\bibitem{FCSC'25:robust-domain-adaptation}
S.~Li, S.~Zhao, Z.~Cao, S.~Huang, and S.~Chen, ``Robust domain adaptation with
  noisy and shifted label distribution,'' \emph{Frontiers of Computer Science},
  vol.~19, no.~3, p. 193310, 2025.

\bibitem{SADM'12:adaptive-forgetting}
C.~Anagnostopoulos, D.~K. Tasoulis, N.~M. Adams, N.~G. Pavlidis, and D.~J.
  Hand, ``Online linear and quadratic discriminant analysis with adaptive
  forgetting for streaming classification,'' \emph{Statistical Analysis and
  Data Mining}, vol.~5, no.~2, pp. 139--166, 2012.

\bibitem{TKDE'21:DFOP}
P.~Zhao, X.~Wang, S.~Xie, L.~Guo, and Z.-H. Zhou, ``Distribution-free one-pass
  learning,'' \emph{IEEE Transactions on Knowledge and Data Engineering},
  vol.~33, pp. 951--963, 2021.

\bibitem{guo1993performance}
L.~Guo, L.~Ljung, and P.~Priouret, ``Performance analysis of the forgetting
  factor {RLS} algorithm,'' \emph{International Journal of Adaptive Control and
  Signal Processing}, vol.~7, no.~6, pp. 525--537, 1993.

\bibitem{SP'17:FFRLS}
Y.~Chu and C.~M. Mak, ``A variable forgetting factor diffusion recursive least
  squares algorithm for distributed estimation,'' \emph{Signal Processing},
  vol. 140, pp. 219--225, 2017.

\bibitem{NIPS'11:AY-linear-bandits}
Y.~Abbasi-Yadkori, D.~P{\'{a}}l, and C.~Szepesv{\'{a}}ri, ``Improved algorithms
  for linear stochastic bandits,'' in \emph{Advances in Neural Information
  Processing Systems 24 (NIPS)}, 2011, pp. 2312--2320.

\bibitem{ICML'20:logistic-bandits}
L.~Faury, M.~Abeille, C.~Calauz{\`{e}}nes, and O.~Fercoq, ``Improved optimistic
  algorithms for logistic bandits,'' in \emph{Proceedings of the 37th
  International Conference on Machine Learning (ICML)}, 2020, pp. 3052--3060.

\bibitem{AISTATS'21:optimal-logistic-bandits}
M.~Abeille, L.~Faury, and C.~Calauz{\`e}nes, ``Instance-wise minimax-optimal
  algorithms for logistic bandits,'' in \emph{Proceedings of the 24th
  International Conference on Artificial Intelligence and Statistics
  (AISTATS)}, 2021, pp. 3691--3699.

\bibitem{arxiv'21:NSLB_revisit_note}
P.~Zhao and L.~Zhang, ``Non-stationary linear bandits revisited,'' \emph{ArXiv
  preprint}, vol. arXiv:2103.05324, 2021.

\bibitem{arXiv'21:restartucb}
P.~Zhao, L.~Zhang, Y.~Jiang, and Z.-H. Zhou, ``A simple approach for
  non-stationary linear bandits,'' \emph{ArXiv preprint}, vol.
  arXiv:2103.05324, 2021.

\bibitem{colt'22:most_significant_arm}
J.~Suk and S.~Kpotufe, ``Tracking most significant arm switches in bandits,''
  in \emph{Proceedings of the 35th Conference on Learning Theory (COLT)}, 2022,
  pp. 2160--2182.

\bibitem{arxiv'22:new_look}
Y.~Abbasi-Yadkori, A.~Gy{\"o}rgy, and N.~Lazi{\'c}, ``A new look at dynamic
  regret for non-stationary stochastic bandits,'' \emph{Journal of Machine
  Learning Research}, vol.~24, no. 288, pp. 1--37, 2023.

\bibitem{arXiv'23:LB_memory}
G.~Clerici, P.~Laforgue, and N.~Cesa{-}Bianchi, ``Linear bandits with memory:
  from rotting to rising,'' \emph{ArXiv preprint}, vol. arXiv:2302.08345, 2023.

\bibitem{Manage'22:window-LB}
W.~C. Cheung, D.~Simchi{-}Levi, and R.~Zhu, ``Hedging the drift: Learning to
  optimize under nonstationarity,'' \emph{Management Science}, vol.~68, no.~3,
  pp. 1696--1713, 2022.

\bibitem{NIPS10:GLM-infinite}
S.~Filippi, O.~Capp{\'{e}}, A.~Garivier, and C.~Szepesv{\'{a}}ri, ``Parametric
  bandits: The generalized linear case,'' in \emph{Advances in Neural
  Information Processing Systems 23 (NIPS)}, 2010, pp. 586--594.

\bibitem{NIPS'25:GLB_one_pass_update}
Y.-J. Zhang, S.-A. Xu, P.~Zhao, and M.~Sugiyama, ``Generalized linear bandits:
  Almost optimal regret with one-pass update,'' in \emph{Advances in Neural
  Information Processing Systems 38 (NeurIPS)}, 2025, p. to appear.

\bibitem{L4DC'20:lin_mix_mdp}
Z.~Jia, L.~Yang, C.~Szepesv{\'{a}}ri, and M.~Wang, ``Model-based reinforcement
  learning with value-targeted regression,'' in \emph{Proceedings of the 2nd
  Annual Conference on Learning for Dynamics and Control (L4DC)}, 2020, pp.
  666--686.

\bibitem{ICML'20:lin_mix_mdp}
A.~Ayoub, Z.~Jia, C.~Szepesv{\'{a}}ri, M.~Wang, and L.~Yang, ``Model-based
  reinforcement learning with value-targeted regression,'' in \emph{Proceedings
  of the 37th International Conference on Machine Learning (ICML)}, 2020, pp.
  463--474.

\bibitem{ICML18:GLM-finite}
L.~Li, Y.~Lu, and D.~Zhou, ``Provably optimal algorithms for generalized linear
  contextual bandits,'' in \emph{Proceedings of the 34th International
  Conference on Machine Learning (ICML)}, 2017, pp. 2071--2080.

\bibitem{ICML'16:ZhangYJXZ16}
L.~Zhang, T.~Yang, R.~Jin, Y.~Xiao, and Z.-H. Zhou, ``Online stochastic linear
  optimization under one-bit feedback,'' in \emph{Proceedings of the 33rd
  International Conference on Machine Learning (ICML)}, 2016, pp. 392--401.

\bibitem{NIPS'17:online-GLB}
K.-S. Jun, A.~Bhargava, R.~D. Nowak, and R.~Willett, ``Scalable generalized
  linear bandits: Online computation and hashing,'' in \emph{Advances in Neural
  Information Processing Systems 30 (NIPS)}, 2017, pp. 99--109.

\bibitem{COLT'19:TS_LogB}
S.~Dong, T.~Ma, and B.~V. Roy, ``On the performance of thompson sampling on
  logistic bandits,'' in \emph{Proceedings of the 32nd Conference on Learning
  Theory (COLT)}, 2019, pp. 1158--1160.

\bibitem{TIT'12:restless}
C.~Tekin and M.~Liu, ``Online learning of rested and restless bandits,''
  \emph{IEEE Transactions on Information Theory}, vol.~58, no.~8, pp.
  5588--5611, 2012.

\bibitem{TIT'24:restless2}
P.~N. Karthik, V.~Y.~F. Tan, A.~Mukherjee, and A.~Tajer, ``Optimal best arm
  identification with fixed confidence in restless bandits,'' \emph{IEEE
  Transactions on Information Theory}, vol.~70, no.~10, pp. 7349--7384, 2024.

\bibitem{NIPS'19:TV_bound}
D.~Baby and Y.~Wang, ``Online forecasting of total-variation-bounded
  sequences,'' in \emph{Advances in Neural Information Processing Systems 32
  (NeurIPS)}, 2019, pp. 11\,069--11\,079.

\bibitem{arXiv'19:window-LB}
W.~C. Cheung, D.~Simchi-Levi, and R.~Zhu, ``Hedging the drift: Learning to
  optimize under non-stationarity,'' \emph{ArXiv preprint}, vol.
  arXiv:1903.01461, 2019.

\bibitem{NIPS'15:Neu}
G.~Neu, ``Explore no more: Improved high-probability regret bounds for
  non-stochastic bandits,'' in \emph{Advances in Neural Information Processing
  Systems 28 (NIPS)}, 2015, pp. 3168--3176.

\end{thebibliography}

\begin{IEEEbiographynophoto}{Jing Wang}
received the BSc degree from the Nanjing University of Aeronautics and Astronautics in 2021. Currently, he is working toward the PhD degree with the School of Artificial Intelligence, Nanjing University. His research interest is mainly on machine learning and data mining.
\end{IEEEbiographynophoto}

\begin{IEEEbiographynophoto}{Peng Zhao}
(Member, IEEE) received the BSc degree from Tongji University in 2016 and the PhD degree from Nanjing University in 2021. He is currently an assistant professor with the School of Artificial Intelligence, Nanjing University. His research interests lie in the foundations of machine learning, with a focus on online learning, bandits, and optimization. He has published over 60 papers in total across top-tier journals, such as \emph{Journal of Machine Learning Research} and IEEE/ACM Transactions, and premier conferences, including ICML, NeurIPS, and COLT.
He regularly serves as an area chair for ICML and NeurIPS, and is an action editor for \emph{Machine Learning} (Springer).
\end{IEEEbiographynophoto}

\begin{IEEEbiographynophoto}{Zhi-Hua Zhou}
(Fellow, IEEE) received the BSc, MSc, and PhD degrees (Hons.) in computer science from Nanjing University, Nanjing, China, in 1996, 1998, and 2000, respectively. He joined Nanjing University as an assistant professor, in 2001, where he is currently a professor and the vice president. He is also the founding director of the LAMDA Group. He has authored the books Ensemble Methods: Foundations and Algorithms, Evolutionary Learning: Advances in Theories and Algorithms, Machine Learning, and has published more than 200 papers in top-tier international journals or conference proceedings. He holds more than 30 patents. His research interests are mainly in artificial intelligence, machine learning, and data mining. He is a fellow of ACM, AAAI, AAAS, etc. He received various awards/honors, including the National Natural Science Award of China, the IEEE Computer Society Edward J. McCluskey Technical Achievement Award, and the CCF-ACM Artificial Intelligence Award. He founded Asian Conference on Machine Learning (ACML). He is the president of IJCAI Trustee, a series editor of Lecture Notes in Artificial Intelligence (Springer), an advisory board member of AI Magazine, the editor-in-chief of \emph{Frontiers of Computer Science}, and the associate editor-in-chief of \emph{Science China Information Sciences}.
\end{IEEEbiographynophoto}

\end{document}